\documentclass[11pt]{article}
\usepackage[margin=1in]{geometry}
\usepackage{amssymb,amsthm}
\usepackage{authblk}
\usepackage[hidelinks]{hyperref}
\usepackage[numbers]{natbib}
\usepackage{tablefootnote}

\newtheorem{theorem}{Theorem}
\newtheorem{definition}[theorem]{Definition}

\newtheorem{lemma}[theorem]{Lemma}

\newtheorem{proposition}[theorem]{Proposition}

\newtheorem{assumption}{Assumption}

\usepackage{amsmath, amsfonts, amssymb, mathtools}
\usepackage{algorithm}
\usepackage{algcompatible}
\usepackage{relsize}
\usepackage{subcaption}
\usepackage{algpseudocode}
\usepackage{float}
\newcommand{\argmax}{\operatornamewithlimits{argmax}}
\newcommand{\eps}{\epsilon}

\usepackage[suppress]{color-edits}
\addauthor{srs}{cyan}
\addauthor{cy}{violet}
\addauthor{ts}{blue}
\addauthor{yc}{red}

\def\E{\mathbb{E}}
\def\R{\mathbb{R}}
\def\Z{\mathbb{Z}}
\def\P{\mathbb{P}}
\def\eps{\epsilon}

\def\Bin{\mathrm{Bin}}

\begin{document}

\title{Overcoming the Long Horizon Barrier for Sample-Efficient Reinforcement Learning with Latent Low-Rank Structure}

\author[1]{Tyler Sam}
\author[2]{Yudong Chen}
\author[1]{Christina Lee Yu}
\affil[1]{School of Operations Research and Information Engineering, Cornell University}
\affil[2]{Department of Computer Sciences, University of Wisconsin-Madison}

\maketitle
\begin{abstract}
The practicality of reinforcement learning algorithms has been limited due to poor scaling with respect to the problem size, as the sample complexity of learning an $\epsilon$-optimal policy is $\Tilde{\Omega}\left(|S||A|H^3 / \epsilon^2\right)$ over worst case instances of an MDP with state space $S$, action space $A$, and horizon $H$. We consider a class of MDPs for which the associated optimal $Q^*$ function is low rank, where the latent features are unknown. While one would hope to achieve linear sample complexity in $|S|$ and $|A|$ due to the low rank structure, we show that without imposing further assumptions beyond low rank of $Q^*$, if one is constrained to estimate the $Q$ function using only observations from a subset of entries, there is a worst case instance in which one must incur a sample complexity exponential in the horizon $H$ to learn a near optimal policy. We subsequently show that under stronger low rank structural assumptions, given access to a generative model, Low Rank Monte Carlo Policy Iteration (LR-MCPI) and Low Rank Empirical Value Iteration (LR-EVI) achieve the desired sample complexity of $\Tilde{O}\left((|S|+|A|)\mathrm{poly}(d,H)/\epsilon^2\right)$ for a rank $d$ setting, which is minimax optimal with respect to the scaling of $|S|, |A|$, and $\eps$. In contrast to literature on linear and low-rank MDPs, we do not require a known feature mapping, our algorithm is computationally simple, and our results hold for long time horizons. Our results provide insights on the minimal low-rank structural assumptions required on the MDP with respect to the transition kernel versus the optimal action-value function.
\end{abstract}

\clearpage
\tableofcontents
\clearpage
\section{Introduction}
\label{sec:introduction}

Reinforcement learning (RL) methods have been increasingly popular in sequential decision making tasks due to their empirical success, e.g., Atari Games \citep{Mnih2015HumanlevelCT}, StarCraft II \citep{Vinyals2019GrandmasterLI}, and robotics \citep{li2019reinforcement}. RL algorithms can be applied to any sequential decision making problem which can be modeled by a Markov decision process (MDP) defined over a state space $S$ and an action space $A$. The agent interacts with the environment across a horizon of length $H$. In each step of the horizon, the agent observes the current state of the environment and takes an action. In response the environment returns an instantaneous reward and transitions to the next state. The key Markov property that the dynamics of an MDP must satisfy is that the distribution of the instantaneous reward and the next state is only a function of the current state and action. As a result it is sufficient for the agent to only consider policies that define a distribution over the actions given the current state of the environment. The goal of the agent is to find an optimal policy which maximizes its cumulative expected reward over the horizon. When the dynamics and reward function of the MDP are known in advance, that can be solved directly using dynamic programming. Reinforcement learning considers the setting in which the MDP dynamics are unknown and thus the algorithm must query from the MDP to both learn the model as well as find an optimal policy.

Despite the empirical success and popularity of RL, its usage in practical applications is limited by the high data sampling costs in the training process, resulting from poor scaling of RL algorithms with respect to the size of the state and action spaces. Given a finite-horizon homogeneous MDP with state space $S$, action space $A$, and horizon $H$,  one needs $\Tilde{\Omega}\left(|S||A|H^3/\eps^2\right)$ samples given a generative model to learn an optimal policy \citep{sidford2019nearoptimal}. The required number of samples is often too large as many real-world problems when modeled as a Markov decision process (MDP) have very large state and action spaces. For example, the $n$-city Capacitated Vehicle Routing Problem (CVRP), a  classical combinatorial problem from operations research, involves a state space $\{0,1\}^n$ and an action space being all partial permutations of $n-1$ cities \citep{DBLP:conf/nips/DelarueAT20}. 

A key function that is used in the course of solving for an optimal policy is the $Q^{\pi}$ function, which is also referred to as the action-value function of policy $\pi$. It is defined over steps $h \in [H]$, states $s \in S$, and actions $a \in A$. $Q^{\pi}_h(s,a)$ represents the expected cumulative reward that an agent would collect if it were at state $s$ at step $h$, took action $a$, and subsequently followed the policy $\pi$ for all future steps until the end of the horizon. When the state and action space are finite, the $Q_h^{\pi}$ function can be represented as a $|S|\times|A|$ matrix. The $Q$ function associated to the optimal policy is denoted by $Q^*$. Given $Q_h^*$, the optimal policy at step $h$ is trivial to find as it would follow from simply choosing the action that optimizes $Q_h^*$ for each state. Many RL algorithms rely on estimating the $Q_h^*$ functions across all state action pairs in order to find a near optimal policy, resulting in the $|S||A|$ sample complexity dependence. Furthermore, the tight lower bound also suggests one may need to estimate the full $Q_h^*$ function to find the optimal policy in worst case MDPs.

\medskip \noindent
{\bf MDPs with Low Rank Structures.}
A glaring limitation of general purpose RL algorithms is that they do not exploit application dependent structure that may be known in advance. Many real-world systems in fact have additional structure that if exploited should improve computational and statistical efficiency. The critical question becomes what structure is reasonable to assume, and how to design new algorithms and analyses to efficiently exploit it. In this work, we focus on the subclass of MDPs that exhibit latent low-dimensional structure with respect to the relationship between states and actions, e.g., $Q_h^*$ is low rank when viewed as a $|S|$-by-$|A|$ matrix. A sufficient but not necessary condition that would result in such a property is that  the transition kernel has low Tucker rank when viewed as a $|S|$-by-$|S|$-by-$|A|$ tensor of state transition probabilities, and the expected instantaneous reward function can be represented as a $|S|$-by-$|A|$ low rank matrix.

While low rank structure has been extensively used in the matrix and tensor estimation literature, it has not been widely studied in the RL literature, \tsedit{except for the theoretical results from \cite{serl} and empirical results from} \cite{Yang2020Harnessing, lrVfa, mlrvfa}. However, we will give examples at the end of Section \ref{sec:asm} to illustrate that this property is in fact quite widespread and common in many real world systems.
%
While the sample complexity under the fully general model scales as $|S||A|$, we would expect that the sample complexity under a rank-$d$ model would scale as $d (|S| + |A|)$, as the low rank assumption on a matrix reduces the degrees of freedom of $Q_h^*$ from $|S||A|$ to $d (|S| + |A|)$. Even though this intuition holds true in the classical low rank matrix estimation setting, the additional dynamics of the MDP introduce complex dependencies that may amplify the error for long horizons. The work in~\cite{serl} proposes an algorithm that learns an $\eps$-optimal $Q$ function under a low rank assumption on $Q^*_h$, resulting in a sample complexity of $\Tilde{O}(\mathrm{poly}(d)(|S|+|A|)\mathrm{exp}(H)/\eps^2)$ in the general finite-horizon MDP setting. While they do achieve the reduction from $|S||A|$ to $\mathrm{poly}(d)(|S|+|A|)$, they have an exponential dependence on the horizon that arises from an amplification of the estimation error due to the MDP dynamics and nonlinearity of low rank matrix estimation. A key contribution of this work is to characterize conditions under which we are able to achieve both linear sample complexity on $|S|$ and $|A|$ along with polynomial dependence on $H$.

The term ``low rank'' has been used to describe other types of low dimensional models in the MDP/RL literature, especially in the context of linear function approximation, and we would like to clarify up front that these models are significantly different. In particular, the typical use of ``low rank MDPs'' refers to an assumption that the transition kernel when viewed as a tensor is low rank with respect to the relationship between the originating state action pair $(s,a)$ and the destination state $s'$. This implies that the relationship across time for a given trajectory exhibits a latent low dimensional structure in that the relationship between the future state and the previous state and action pair is mediated through low dimensional dynamics. However, this assumption does not imply that the $Q$ function is low rank when viewed as a matrix, which would imply a low dimensional relationship between the current state and the action taken at that state. Another assumption which is easily confused with ours is the assumption that $Q^*$ is linearly-realizable. This implies that $Q^*$ can be written as a linear combination of $d$ matrices $\{\phi_{\ell}\}_{\ell \in [d]}$ (each of size $|S|$ by $|A|$). While this implies that the set of plausible $Q^*$ lives in a low dimensional space parameterized by $\{\phi_{\ell}\}_{\ell \in [d]}$, this does not imply that $Q^*$ is low rank with respect to the relationship between $S$ and $A$. The guarantees for RL algorithms under low rank MDP and linearly-realizable $Q^*$ structure either require prior knowledge of the feature representation as given by $\{\phi_{\ell}\}_{\ell \in [d]}$, or otherwise do not admit polynomial time algorithms. While assuming a priori knowledge of the feature representation is often restrictive and unlikely in real applications, this assumption enables a reduction to supervised learning such that the sample complexity no longer depends on the size of the state and action space, but only on the dimension of the representation. The low rank structure we assume in this work does not require any knowledge of the latent low dimensional representation, but as a result the optimal sample complexity necessarily must still scale linearly with the size of the state and action space.

\medskip
\noindent
{\bf Our Contributions.} 
We identify sufficient low-rank structural assumptions that allow for computationally and statistically efficient learning, reducing the sample complexity bounds to scale only linearly in $|S|, |A|$ and \emph{polynomially} in $H$ (as opposed to exponential in $H$ in \cite{serl} or $|S||A|$ in the general tabular MDP setting). 
First, we show that there are additional complexities that arise from MDPs with \emph{long horizons}; we provide an example where the optimal action-value function $Q^*$ is low rank, yet the learner must observe an exponential (in $H$) number of samples to learn a near optimal policy when exploiting the low-rank structure of $Q^*$. This lower bound illustrates that exploiting low rank structure in RL is significantly more involved than classical matrix estimation. We propose a new computationally simple model-free algorithm, referred to as Low Rank Monte Carlo Policy Iteration (LR-MCPI). Under the assumption that $Q^*$ is low rank, by additionally assuming a constant suboptimality gap, we prove that LR-MCPI achieves the desired sample complexity, avoiding the exponential error amplification in the horizon. Additionally we prove that LR-MCPI also achieves the desired sample complexity when all $\eps$-optimal policies $\pi$ have low rank $Q^{\pi}$ functions. Under the stronger assumption that the transition kernel and reward function have low rank, we show that the model-free algorithm in \cite{serl}, which we refer to as Low Rank Empirical Value Iteration (LR-EVI), also achieves the desired sample complexity. Table \ref{tbl:results} summarizes our sample complexity bounds in their corresponding settings, and compares them with existing results from literature in the tabular finite-horizon MDP setting; here $d$ refers to the rank parameter.\tablefootnote{The sample complexity bounds of Theorems \ref{thm:gap}, \ref{thm:qnolr}, and \ref{thm:tklr} presented in the table hide terms that are properties of the matrix, which are constant under common regularity assumptions (and will be discussed in later sections) and terms independent of $|S|$ or $|A|$.} 
\begin{table}[h]
    \centering
    \begin{tabular}{|c|c|}
    \hline
        MDP Assumptions & Sample Complexity \\
         \hline
         Low-rank $Q^*_h$ \& suboptimality gap $\Delta_{\min} > 0$ (Theorem~\ref{thm:gap})& $\Tilde{O}\left(\frac{d^3(|S|+|A|)H^4}{\Delta_{\min}^2} \right)$\\
         $\eps$-optimal policies have low-rank $Q^\pi_h$ (Theorem~\ref{thm:qnolr}) & $\Tilde{O}\left(\frac{d^3(|S|+|A|)H^6}{\eps^2}\right)$ \\
         Transition kernels and rewards are low-rank (Theorem~\ref{thm:tklr}) & $\Tilde{O}\left(\frac{d^3(|S|+|A|)H^5}{\eps^2}\right)$\\
         \hline
         Low-rank $Q^*_h$ \& constant horizon  \citep{serl} &$\Tilde{O}\left(\frac{d^5(|S|+|A|)}{\eps^2}\right)$ \\
         Tabular MDP with homogeneous rewards \citep{sidford2019nearoptimal} & $\Tilde{\Theta}\left(\frac{|S||A|H^3}{\eps^2}\right)$\\
         \hline
    \end{tabular}
    \caption{Our sample complexity bounds alongside results from the literature, where $d$ denotes the rank. }
    \label{tbl:results}
\end{table}

We extend our results to approximately low-rank MDPs, for which we show that our algorithm learns action-value functions with error $\eps + O(H^2\xi)$, where $\xi$ is the rank-$d$ approximation error, with an efficient number of samples. Furthermore, we empirically validate the improved efficiency of our low-rank algorithms. In the appendix, we show that our algorithm learns near-optimal action-value functions in a sample-efficient manner in the continuous setting, similar to the results in the table above. Finally, we prove that using existing convex program based matrix estimation methods instead of the one in \cite{serl} also achieves the desired reduction in sample complexity.


\section{Related Work}\label{sec:relatedwork}

\medskip \noindent
{\bf Tabular Reinforcement Learning.}
Sample complexity bounds for reinforcement learning algorithms in the tabular MDP setting have been studied extensively, e.g., \cite{DBLP:conf/colt/AgarwalKY20, Yang2021QlearningWL, NIPS2015_309fee4e, li2021qlearning}. Even with a generative model, $\Omega\left(|S||A| / \eps^{2}(1-\gamma)^{3}\right)$ samples are necessary to estimate an $\eps$-optimal action-value function \cite{azarLowerBound}. The work \cite{sidford2019nearoptimal} presents an algorithm and associated analysis that achieves a matching upper bound on the sample complexity (up to logarithmic factors), proving that the lower bound is tight.
Our work focuses on decreasing the sample complexity's dependence on $|S|$ and $|A|$ from $|S||A|$ to $|S|+|A|$ under models with a low-rank structure. 

\medskip \noindent
{\bf Complexity Measures for RL with General Function Approximation.} 
The search for the most general types of structure that allow for sample-efficient reinforcement learning has resulted in many different complexity measures, including Bellman rank \citep{pmlr-v70-jiang17c, pmlr-v125-dong20a}, witness rank \citep{Sun2019ModelbasedRI}, Bellman Eluder dimension \citep{jin2021bellman}, and Bilinear Class \citep{ pmlr-v139-du21a}.
For these classes of MDPs with rank $d$, finding an $\eps$-optimal policy requires $\Tilde{O}\left(\mathrm{poly}(d, H)/\eps^2 \right)$ samples. Unfortunately, these complexity measures are so broad that the resulting algorithms that achieve sample efficiency are often not polynomial time computable, and they rely on strong optimization oracles in general, e.g., assuming that we can solve a high dimensional non-convex optimization problem. We remark that our settings, including those under our strongest assumptions, cannot be easily incorporated into those frameworks.

\medskip \noindent
{\bf  Linear Function Approximation - Linear Realizability and Low Rank MDPs.} 
To combat the curse of dimensionality, there is an active literature that combines linear function approximation with RL algorithms. As mentioned in the introduction, although these models are referred to as ``low rank'', they are significantly different than the type of low rank structure that we consider in our model. Most notably, the resulting $Q^*$ matrix may not be low rank. As a result we only provide a brief overview of the results in this literature, largely to illustrate the types of properties that one would hope to study for our type of low rank model.
One model class in this literature assumes that $Q^*$ is linearly-realizable with respect to a known low dimensional feature representation, given by a known feature extractor $\phi: S\times A \to \R^d$ for $d \ll |S|, |A|$. \cite{wang2021exponential,weisz2021exponential} show that an exponential number of samples in the minimum of the dimension $d$ or the time horizon $H$ may still be required under linear realizability, implying that additionally assumptions are required. These results highlight an interesting phenomenon that the dynamics of the MDP introduce additional complexities for linear function approximation in RL settings that are not present in supervised learning.

A more restrictive model class, sometimes referred to as \emph{Linear/Low-rank MDPs}, imposes linearity on the dynamics of the MDP itself, i.e. the transition kernels and reward functions are linear with respect to a known low dimensional feature extractor $\phi$ \citep{Jin2020ProvablyER, yang2019sampleoptimal, pmlr-v119-yang20h, wang2021sampleefficient, Hao2021SparseFS}. As this does not impose structure on the relationship between $s$ and $a$, the resulting $Q$ functions may not be low rank. When the feature extractor is known, there are algorithms that achieve sample complexity or regret bounds that are polynomial in $d$ with no dependence on $|S|$ or $|A|$.
There have been attempts to extend these results to a setting where the feature mapping is not known \cite{flambe, moffle, uehara2021representation},
however the resulting algorithms are not polynomial time, as they require access to a strong nonconvex optimization oracle. Furthermore they restrict to a finite class of latent representation functions.

\medskip \noindent
{\bf Low Rank Structure with respect to States and Actions.}
There is a limited set of works which consider a model class similar to ours, in which there is low rank structure with respect to the interaction between the states and actions and hence their interaction decomposes. This structure could be imposed on either the transition kernel, or only on the optimal $Q^*$ function. 
\cite{Yang2020Harnessing, lrVfa, mlrvfa} \tsedit{provide empirical results showing that $Q^*$ and near-optimal $Q$ functions for common stochastic control tasks have low rank. Their numerical experiments demonstrate that the performance of standard RL algorithms, e.g., value iteration and TD learning, can be significantly improved in combination with low-rank matrix/tensor estimation methods.} The theoretical work \cite{serl} considers the weakest assumption that only imposes low rankness on $Q^*$. They develop an algorithm that combines a novel matrix estimation method with value iteration to find an $\eps$-optimal action-value function with $\Tilde{O}\left(d^5(|S|+|A|)/\eps^2 \right)$ samples for infinite-horizon $\gamma$-discounted MDPs assuming that $Q^*$ has rank $d$. While this is a significant improvement over the tabular lower bound $\Tilde{\Omega}\left(|S||A|/((1-\gamma)^3\eps^2)\right)$ \citep{azarLowerBound}, their results require strict assumptions. The primary limitation is that they require the discount factor $\gamma$ to be bounded from above by a small constant, which effectively limits their results to short, \emph{constant} horizons. Lifting this limitation is left as an open question in their paper. In this work, we provide a concrete example that illustrates why long horizons may pose a challenge for using matrix estimation in RL. Subsequently we show that this long horizon barrier can be overcome by imposing additional structural assumptions. The algorithm in \cite{serl} also relies on prior knowledge of special anchor states and actions that span the entire space. We will show that under standard regularity conditions, randomly sampling states and actions will suffice.


\medskip \noindent
{\bf  Matrix Estimation.}
Low-rank matrix estimation methods focus on recovering the missing entries of a partially observed low-rank matrix with noise. The field has been studied extensively with provable recovery guarantees; see the surveys \cite{8399563, 7426724}. However, the majority of recovery guarantees of matrix estimation are in the Frobenius norm instead of an entry-wise/$\ell_\infty$ error bound, whereas a majority of common analyses for reinforcement learning algorithms rely upon constructing entrywise confidence sets for the estimated values. Matrix estimation methods with entry-wise error bounds are given in \cite{chen2019noisy, 9087910, abbe2020entrywise}, but all require strict distributional assumptions on the noise, e.g., independent, mean-zero sub-Gaussian/Gaussian error. 
\tsdelete{For an MDP with bounded rewards, to estimate $Q^*_h$, the noise on the observed entries is sub-Gaussian \citep{wainwright2019high} but typically biased as one often does not know the optimal policy for each step after step $h$. To the best of our knowledge, in this setting, the existing matrix estimation methods with entry-wise error bounds have no guarantees on the mean of their estimators.} The matrix estimation method proposed in \cite{serl} provides entry-wise error guarantees for arbitrary bounded noise settings and is the method we use in our algorithm in order to aid our analysis.  \tsdelete{With additional assumptions, the reference matrix comprised of the estimates of our algorithm is low-rank. Hence, our input to the matrix estimation method has mean-zero and bounded noise with respect to the above reference matrix, therefore enabling us to relax the small discount factor requirement.} \cycomment{these last two sentences need some work to clarify, unsure if we want to just move them later because it's hard to explain what "reference matrix" means, and why is this relevant to "relaxing small discount factor requirement? might m=ake more sense in a later discussion?} \tscomment{Yeah I'll move it to where we introduce the ME method}

\section{Preliminaries}
\label{sec:prelim}

We consider a standard finite-horizon MDP given by $(S, A, P, R, H)$~\citep{Sutton1998}. Here $S$ and $A$ are the finite state and action spaces, respectively. $H \in \Z_+$ is the time horizon. $P = \{P_h\}_{h\in [H]}$ is the transition kernel, where $P_h(s'|s, a)$ is the probability of transitioning to state $s'$ when taking action $a$ in state $s$ at step $h$. $R = \{R_h\}_{h\in [H]}$ is the reward function, where  $R_h: S \times A \rightarrow \Delta([0, 1])$ is the distribution of the reward for taking action $a$ in state $s$ at step $h$. We use $r_h(s,a) := \E_{r \sim R_h(s,a)}[r]$ as the mean reward. A stochastic, time-dependent policy of an agent has the form $\pi = \{\pi_h\}_{h\in[H]}$ with $\pi_h: S \to \Delta(A)$, where the agent selects an action according to the distribution $\pi_h(s)$ at time step $h$ when at state $s$.

For each policy $\pi$, the value function and action-value function of $\pi$ represent the expected total future reward obtained from following policy $\pi$ given a starting state or state-action pair at step $h$,
\begin{align}
V_h^\pi(s) &:= \E\left[\left. \textstyle \sum_{t = h}^H r_t(s_t, a_t) ~\right|~ s_h = s\right], \label{eq:V_fn_def} \\
Q_h^\pi(s, a) &:= \E\left[ \left. \textstyle  \sum_{t = h}^H r_t(s_t, a_t) ~\right|~ s_h = s, a_h = a \right], \label{eq:Q_fn_def}
\end{align}
where $a_t \sim \pi_t(s_t)$ and $s_{t+1} \sim P_{t}( \cdot | s_{t},a_{t})$.
The optimal value and action-value functions are given by $V^*_h(s) := \sup_{\pi} V^\pi_h(s)$ and $Q^*_h(s,a) := \sup_{\pi} Q^\pi_h(s,a)$, respectively, for all $s\in S, h\in[H]$. These functions satisfy the Bellman equations
\begin{align}
V_{h}^*(s) = \max_{a \in A} Q^*_h(s,a), 
\quad
Q_h^*(s,a) = r_h(s,a) + \E_{s' \sim P_h(\cdot|s,a)}[V^*_{h+1}(s')],
\quad 
\forall s,a,h
\label{eq:Bellman} 
\end{align}
with $V_{H+1}^*(s) = 0$.
For an MDP with finite spaces and horizon, there always exists an optimal policy $\pi^*$ that satisfies $V^{\pi^*}_h(s) = V^*_h(s)$ for all $s,h$.

The primary goal in this work is to find a near-optimal policy or action-value function. For $\eps > 0$, $\pi$ is an $\eps$-optimal policy if 
$
|V^*_h(s) - V^\pi_h(s)| \leq \eps, \forall (s, h) \in S \times [H].
$
Similarly, $Q = \{Q_h\}_{h\in [H]}$ is called an $\eps$-optimal action-value function if 
$
|Q^*_h(s, a) - Q_h(s, a)| \leq \eps, \forall (s, a, h) \in S \times A \times [H].
$
\tsdelete{
\tsedit{
In contrast, we define a policy to be $c$-suboptimal if $\max_{s_0 \in S} (V^{*}_1(s_0) - V^{\pi}_1(s_0)) \geq c$. This notation is used when we prove our lower bound on the sample complexity; we show that to avoid learning a $c$-suboptimal policy, for a positive constant $c$, one must use an exponential number of samples in $H$.} \cycomment{I think $c$-suboptimal is incorrectly defined? as stated it is equivalent to $c$-optimal.} \tscomment{Yes, I made a typo and have corrected it.} \yccomment{If this definition is used only locally in the proof of Theorem~\ref{thm:explb}, maybe we can move it there (or unpack it) instead of cluttering the preliminaries.}} We will view $Q_h^*$, $Q_h^\pi$ and $r_h$ as $|S|$-by-$|A|$ matrices and $P_h(\cdot|\cdot,\cdot)$ as an $|S|$-by-$|S|$-by-$|A|$ tensor, for which various low-rank assumptions are considered. For a given function $V:S\to\R$, we sometimes use the shorthand $[P_h V](s,a) := \E_{s' \sim P_h(\cdot|s,a)}[V(s')]$ for the conditional expectation under $P_h$. 

Throughout this paper, we assume access to a simulator (a.k.a.\ the generative model framework, introduced by \cite{NIPS1998_99adff45}), which takes as input a tuple $(s,a,h)\in S\times A\times[H]$ and outputs independent samples $s'\sim P_h(\cdot|s,a)$ and $r \sim R_h(s,a)$ .
This assumption is one of the stronger assumptions in reinforcement learning literature, but common in the line of work that studies sample complexity without directly addressing the issue of exploration, e.g.,  \cite{sidford2019nearoptimal, DBLP:conf/colt/AgarwalKY20}.

\medskip \noindent
{\bf Notation.} Let $a \wedge b := \min(a,b),$ $a \vee b := \max(a, b),$  $\delta_a$ denote the distribution over $A$ that puts probability 1 on action $a$, $ \sigma_i(M)$ denote the $i$-th largest singular value of a matrix $M$, and $M_{i}$ denote the $i$-th row. The $n$-by-$n$ identity matrix is denoted by $I_{n\times n}$, and $[H] := \{1, \ldots, H\}$.  We use several vector and matrix norms: Euclidean/$\ell_2$ norm $\|\cdot \|_2$, spectral norm $\|M\|_{op} = \sigma_1(M)$, nuclear norm $\|M\|_*$, entrywise $\ell_\infty$ norm $\|M\|_\infty$ (largest absolute value of entries), and Frobenius norm $\|M\|_F$. We define the condition number of a rank-$d$ matrix $M$ as $\kappa_{M} \coloneqq \frac{\sigma_1(M)}{\sigma_d(M)}$. \cydelete{\tsedit{We define the $i$-mode product, $\times_i$,  between tensor $X \in \R^{n_1 \times \ldots \times n_m}$ and matrix $A \in \R^{n_i \times d}$ as $X \times_i A = Y$ where $Y_{i_1, i_2, \ldots, j, \ldots i_m} = \sum_{i = 1}^{n_i} X_{i_1, i_2, \ldots, i, \ldots i_m}A_{i, j}$. }}

\section{Information Theoretic Lower Bound}
\label{sec:lowerbound}

While one may hope to learn the optimal action-value functions when only assuming that $Q^*_h$ is low rank, we argue that the problem is more nuanced. Specifically, we present two similar MDPs with rank-one $Q^*$, where the learner has complete knowledge of the MDP except for one state-action pair at each time step. \tsedit{As the learner is restricted from querying that specified state-action pair, in order to distinguish between the two MDPs and learn the optimal policy, the learner must use the low-rank structure to estimate the unknown entry. We then show that doing so requires a exponential number of observations in the horizon $H$.}

Consider MDPs $M^{\theta}=(S,A,P,R^{\theta}, H)$ indexed by a real number $\theta$, where $S=A=\{1,2\}$. At $h=1$, $r_{1}^{\theta}(s_{1},a)=0$, $P_{1}(\cdot|s_{0},a)=\delta_{a}$ for all $a \in A$, and the starting state $s_1$ is deterministic. For $h > 1$,
\[
r_{H}^{\theta}=\left(\begin{array}{c}
\frac{1}{2}\\
1+2\theta
\end{array}\right),\qquad r_{h}^{\theta}=\left(\begin{array}{cc}
-\frac{1}{4} & 0\\
-\frac{1}{2} & 2^{H-h}\theta
\end{array}\right),\qquad\text{and }P_{h}(\cdot|s,a)=\delta_{s},\qquad\forall s,a,\forall h\in \{2, \ldots, H-1\},
\]
where $\delta_s$ denotes the Dirac delta distribution at $s$. 
\tsedit{The rewards are deterministic except for the terminal reward at state $2$, where the reward distribution $R^{\theta}_H(2)$ is such that the reward takes value $2$ with probability $\frac{1}{2}+ \theta$, and takes value 0 otherwise.}

\tsedit{If action $a=1$ (resp., $2$) is
taken at the initial step $h=1$, the MDP will transition to state $1$ (resp.,
$2$) and then stay at this state in all subsequent steps.} 
\cyedit{Thus learning the optimal policy only depends on determining the optimal action in step 1.}
Let $\theta$ take one of two possible values: $\theta_1=-\frac{3}{4\cdot2^{H-1}}$ and $\theta_2 = \frac{3}{4\cdot2^{H-1}}.$ \tsedit{To determine the correct action at the initial step, one must correctly identify $\theta$. We will show that identifying $\theta$ takes an exponential number of samples in the horizon $H$.}

\begin{lemma}\label{lem:Qlb}
The optimal policy for the above MDP (for both values of $\theta$) for steps $h \geq 1$ is $\pi^*_h(1) = \pi^*_h(2) = 2$ for all $h \in \{2, \ldots, H-1\}$. Furthermore,  
\[
Q_{h}^{*,\theta}=\left(\begin{array}{cc}
\frac{1}{4} & \frac{1}{2}\\
\frac{1}{2}+2^{H-h}\theta, & 1+2^{H-h+1}\theta
\end{array}\right),\qquad V_{h}^{*,\theta}=\left(\begin{array}{c}
\frac{1}{2}\\
1+2^{H-h+1}\theta
\end{array}\right),\qquad \forall h \in \{2, \ldots, H-1\}.
\]
\end{lemma}
Lemma~\ref{lem:Qlb}, proved in Appendix~\ref{sec:lbProofs},
shows that $Q^*_h$ is rank one.
We will calculate the optimal Q function and policy at step $h=1$ \cyedit{in the proof of Theorem \ref{thm:explb} after introducing the observation model}. 

\textit{Observation model:} 
The learner has exact knowledge of $r_{H}^{\theta}(1)$, $r_{h}^{\theta}(s,a)$, $P_{h}, r_1^{\theta}, P_1$
for all $(s,a)\in\Omega:=\{(1,1),(1,2),(2,1)\}$ and $h\in[H-1]$.
Note that these known rewards and transitions are independent of $\theta$. In addition, the learner is given $n$ iid samples
from $R_H^{\theta}(2)$.

One interpretation of this observation model is that the learner has
infinitely many samples of the form $(s,a,s'), s'\sim P_h(\cdot|s,a)$ for each $(s,a)$,
so $P_{h}$ can be estimated with zero error. Similarly, the learner
has infinitely many samples from $r_{H}^{\theta}(1)$ and $r_{h}^{\theta}(s,a)$ for
$(s,a)\in\Omega$. However, the learner cannot observe $r_{h}^{\theta}(2,2)$
and hence must estimate $Q_{h}^{\theta}(2,2)$ using the low-rank structure.
Finally, $n$ noisy observations of the terminal reward $r_{H}^{\theta}(2)$ at state $2$ are given. 

\begin{theorem}\label{thm:explb}
Consider the above class of MDPs and observation model. To learn a $1/8$-optimal policy with probability at least 0.9, the learner must observe $n=\Omega(4^H)$ samples from $R_H^{\theta}(2)$.
\end{theorem}
\begin{proof}
From Lemma \ref{lem:Qlb}, we have $ V_{2}^{*,\theta_1}=\left(\begin{array}{c}
1/2\\
1/4
\end{array}\right)$ and $ V_{2}^{*,\theta_2}=\left(\begin{array}{c}
1/2\\
7/4
\end{array}\right)$. Hence, at $h=1$ the optimal action is $\pi_1^*(s_0)=1$ for $\theta=\theta_1$ and $\pi_1^*(s_0)=2$ for $\theta=\theta_2$. If $\theta = \theta_1$ and action 2 is taken instead, $\pi_1(s_0)=2$, then this action \tsedit{incurs a $1/4$ penalty in value relative to the optimal action, i.e., $Q_1^{*,\theta_1}(s_0, 2) \le Q_1^{*,\theta_1}(s_0, 1) - \frac{1}{4}$. If $\theta = \theta_2$ and action $1$ is taken, $\pi_1(s_0) = 1$, then this action incurs a $5/4$ penalty relative to the optimal action.} Therefore, to learn an $\eps$-optimal policy for $\eps < 1/4$, e.g., $\eps = 1/8$ as stated in the theorem, the learner must correctly determine whether $\theta = \theta_1$ or $\theta = \theta_2$. It is well known from existing literature, see e.g., \cite{statlimit, neuralnet, scMAB}, that one needs $\Omega \left(1/(2\theta_1 - 2\theta_2)^2\right)$ samples to distinguish two (scaled) Bernoulli distributions with mean $1/2+\theta,\theta\in\{\theta_1,\theta_2\}$ with probability at least 0.9. Substituting $(2\theta_1 - 2\theta_2)^2 > \frac{4^{H-1}}{9}$ proves the result.
\end{proof}

Consider the following operational interpretation of the above example. The learner can use the rank-one structure to estimate $Q^*_h(2,2)$ given $Q^*_h(s,a),(s,a)\in\Omega$ as follows:
$Q^*_h(2,2) = Q^*_h(1,2)Q^*_h(2,1)/Q^*_h(1,1)$, coinciding with the matrix estimation algorithm in~\cite{serl}.
Lemma~\ref{lem:Qlb} shows that an $\varepsilon=2^{H-h}\theta$ error in $Q_h^*(2,1)$ leads to a $2\cdot\varepsilon=2^{H-h+1}\theta$ error in $Q_h^*(2,2)$ and $Q_{h-1}^*(2,1)$. As such, the error is amplified exponentially when propagating backwards through the horizon, showing that this low-rank based procedure is inherently unstable.

This example illustrates that reinforcement learning with low-rank structure is more nuanced than low-rank estimation without dynamics, and that the constant horizon assumption in~\cite{serl} is not merely an artifact of their analysis. \tsedit{Furthermore, as the entries of $Q^*_h$ are similar in magnitude, the blow up in error is not due to the missing entry containing most of the signal.} This motivates us to consider additional assumptions beyond $Q^*_h$ being low rank. In the above example, the optimal state-action pair $(2,2)$ is not observed, and the reward $r_h$ and transition kernel $P_h$ are not low-rank. To achieve stable and  sample-efficient learning with long horizons, we will consider when additional structures in the MDP dynamics can be exploited to identify and sample from the optimal action.
\section{Assumptions}\label{sec:asm}

In this section, we present three low rank settings that enable sample-efficient reinforcement learning, with each setting increasing in the strength of the low rank structural assumption. 
%

\begin{assumption}[Low-rank  $Q_h^*$]\label{asm:lrqt}
For all $h \in [H]$, the rank of the matrix $Q_h^*$ is $d$. Consequently, $Q_h^*$  can be represented via its singular value decomposition $Q^*_h = U^{(h)} \Sigma^{(h)} (V^{(h)})^\top$, for a $|S| \times d$ orthonormal matrix $U^{(h)}$, a $|A| \times d$ orthonormal matrix $V^{(h)}$, and a $d \times d$ diagonal matrix $\Sigma^{(h)}$.
\end{assumption}


Assumption \ref{asm:lrqt} imposes that the action-value function of the optimal policy is low rank. This assumption can be contrasted with another common structural assumption in the literature, namely linearly-realizable $Q^*$, meaning that  $Q^*_h(s,a)=w_h^\top \phi(s,a)$  for some weight vector $w_h\in \R^d$ and a \emph{known} feature mapping $\phi: S\times A \to \R^d$ \citep{wang2021exponential, weisz2021exponential}. In comparison, Assumption~\ref{asm:lrqt} decomposes $\phi$ into the product of separate feature mappings on the state space $U^{(h)}$ and the action space $V^{(h)}$. Hence, linearly-realizable $Q^*$ does not imply low-rank $Q^*_h$. Furthermore, we assume the latent factors $U^{(h)}$ and $V^{(h)}$ are completely unknown, whereas the linear function approximation literature typically assumes $\phi$ is known or approximately known.

Assumption \ref{asm:lrqt} only imposes low-rankness on $Q^*_h$, allowing for the $Q^{\pi}$ function associated to non-optimal policies $\pi$ to be full rank. Assumption~\ref{asm:lrqt} is likely too weak, as Theorem \ref{thm:explb} illustrates a doubly exponential growth in policy evaluation error under only this assumption. Below we present three additional assumptions. \emph{Each} of these assumptions enable our algorithms to achieve the desired sample complexity when coupled with Assumption \ref{asm:lrqt}.

\begin{assumption}[Suboptimality Gap]\label{asm:sog}
For each $(s, a) \in S \times A$, the suboptimality gap is defined as $\Delta_h(s, a) := V^*_h(s) - Q^*_h(s, a)$. Assume that there exists an $\Delta_{\min} > 0$ such that
\[
\min_{h \in [H], s \in S, a \in A} \{\Delta_h(s, a): \Delta_h(s, a) > 0 \} \geq \Delta_{\min}.
\]
\end{assumption}

Assumption~\ref{asm:sog} stipulates the existence of a suboptimality gap bounded away from zero.
In the finite setting with $|S|, |A|, H < \infty$, there always exists a $\Delta_{\min} > 0$ for any non-trivial MDP in which there is at least one suboptimal action. This is an assumption commonly used in bandit and reinforcement learning literature.

\begin{assumption}[$\eps$-optimal Policies have Low-rank $Q$ Functions]\label{asm:lrqe}
For all $\eps$-optimal policies $\pi$, the associated $Q^{\pi}_h$ matrices are rank-$d$ for all $h \in [H]$, i.e., $Q^{\pi}_h$ can be represented via $Q^\pi_h = U^{(h)} \Sigma^{(h)} (V^{(h)})^\top$ for some $|S| \times d$ matrix $U^{(h)}$, $|A| \times d$ matrix $V^{(h)}$, and $d \times d$ diagonal matrix $\Sigma^{(h)}$. 
\end{assumption}

Assumption~\ref{asm:lrqe} imposes that all $\eps$-optimal policies $\pi$ have low-rank $Q^\pi_h$. 
We have not seen this assumption in existing literature. It is implied by the stronger assumption that \emph{all policies} have low-rank $Q^\pi_h$; see Appendix \ref{app:qnolrex} for an MDP that satisfies Assumption~\ref{asm:lrqe} but fails the stronger assumption. 
The stronger assumption is analogous to the property that $Q^\pi$ is linear in the feature map $\phi$ for \emph{all} policies, which is commonly used in work on linear function approximation and linear MDPs.

\tsedit{To state our strongest low-rank assumption, we first recall the definition of tensor Tucker rank.}

\begin{definition}[Tucker Rank \cite{tucker}]
The Tucker rank of a tensor $X \in \R^{n_1\times n_2 \times n3}$ is the smallest $(d_1, d_2, d_3)$ such that there exists a core tensor $G \in \R^{d_1 \times d_2 \times d_3}$ and orthonormal latent factor matrices $A_i \in \R^{n_i \times d_i}$ for $i \in [3]$ such that for all $(a,b,c) \in [n_1]\times[n_2]\times[n_3]$,
\[X(a,b,c) = \sum_{\ell_1 \in [d_1]} \sum_{\ell_2 \in [d_2]} \sum_{\ell_3 \in [d_3]} G(\ell_1,\ell_2,\ell_3) A_1(a,\ell_1) A_2(b,\ell_2) A_3(c,\ell_3).\]
\end{definition}
\tsdelete{
\begin{definition}[Tucker Rank \cite{tucker}]
The Tucker decomposition of a tensor $X\in\R^{n_1 \times n_2 \times \ldots \times n_m}$ is
\[
X = G \times_1 A_1 \times_2 A_2 \ldots \times_m A_m
\]
where $G \in \R^{d_1 \times d_2 \times \ldots \times d_m}$, $\times_i$ denotes the $i$-mode product (defined in the notation paragraph in section \ref{sec:prelim}), and $A_i \in \R^{n_i \times d_i}$ are matrices with orthonormal columns for all $i \in [m]$. Then, we define the Tucker rank of $X$ as $(d_1, d_2, \ldots, d_m)$.
\end{definition}}
\tsedit{Our strongest low-rank assumption imposes that the expected reward functions are low rank, and the transition kernels have low Tucker rank \emph{along one dimension}. }
\begin{assumption}[Low-rank Transition Kernels and Reward Functions] \label{asm:lrtk}
The expected reward function has rank $d$, and the transition kernel $P_h$ has Tucker rank $(|S|,|S|,d)$ or $(|S|, d, |A|)$, with shared latent factors. For the Tucker rank $(|S|, |S|, d)$ case, this means that for each $h \in [H]$, there exists a  $|S| \times |S| \times d$ tensor $U^{(h)}$, an $|A| \times d$ matrix $V^{(h)}$, and an $|S| \times d$ matrix $W^{(h)}$ such that
\[
P_h(s'|s,a) = \textstyle\sum_{i=1}^d U^{(h)}(s', s, i) V^{(h)}(a, i) \quad\text{ and }\quad r_h(s,a) = \textstyle\sum_{i=1}^d W^{(h)}(s, i)V^{(h)}(a, i).
\]
For the Tucker rank $(|S|, d, |A|)$ case, this means that  for each $h \in [H]$, there exists a  $|S| \times |A| \times d$ tensor $V^{(h)}$, an $|S| \times d$ matrix $U^{(h)}$, and an $|A| \times d$ matrix $W^{(h)}$ such that
\[
P_h(s'|s,a) = \textstyle\sum_{i=1}^d U^{(h)}(s, i)V^{(h)}(s', a, i)\quad\text{ and }\quad r_h(s,a) = \textstyle\sum_{i=1}^d U^{(h)}(s, i)W^{(h)}(a, i).
\]
\end{assumption}


Assumption~\ref{asm:lrtk} is our strongest low-rank structural assumption as it implies that the $Q_h^{\pi}$ functions associated with \emph{any} policy $\pi$ are low rank, which subsequently implies both Assumptions \ref{asm:lrqe} and \ref{asm:lrqt}. In fact, Assumption~\ref{asm:lrtk} implies that for any value function estimate $\hat{V}_h$, the matrix $r_h + [P_h \hat{V}_{h+1}]$ is low rank, as stated in the following proposition.
\begin{proposition} \label{prop:lrEst} 
If the transition kernel has Tucker rank $(|S|, |S|, d)$ or $(|S|,d, |A|)$ and the expected reward function has rank $d$ with shared latent factors, i.e., Assumption \ref{asm:lrtk} holds, then the matrix $r_h + [P_h \hat{V}_{h+1}]$ has rank at most $d$ for any $\hat{V}_{h+1} \in \R^{|S|}$.
\end{proposition}
See Appendix~\ref{sec:proof_lrEst} for the proof. Proposition \ref{prop:lrEst} results from the fact that for any fixed $h$, the matrices corresponding to $r_h$ and $P_h(s'|\cdot, \cdot)$ for all $s'$ share either the same column or row space, which is critically used in the analysis of our Low Rank Empirical Value Iteration algorithm.

Next we present several definitions used to characterize the error guarantees of the matrix estimation algorithm. It is commonly understood in the matrix estimation literature that other properties of the matrix beyond low rank, such as its incoherence or condition number, govern how efficiently a matrix can be estimated. Consider a trivial rank-1 MDP where $H=1$ and the reward is a sparse matrix with only $d$ nonzero entries taking value 1. Since the locations of the nonzero entries are unknown, we will likely observe only zeros upon sampling any small subset of entries. Estimation using a small number of samples would be possible, however, if an expert were to provide knowledge of a special set of rows and columns, which have been referred to as \emph{anchor} states and actions in \cite{serl}. For some sets $S^\#_h \subseteq S$ and $A^\#_h \subseteq A$, we use $Q_h(S^\#_h, A^\#_h)$ to denote the submatrix obtained by restricting $Q_h$ to state-action pairs from $S^\#_h \times A^\#_h$.

\begin{definition}[$(k, \alpha)$-Anchor States and Actions]\label{asm:anchor}
A set of states $S^\#_h \subset S$ and a set of actions $A^\#_h  \subset A$ are $(k, \alpha)$-anchor states and actions for a rank-$d$ matrix $Q_h$ if $|S^\#_h|, |A^\#_h| \leq k$, the submatrix $Q_h(S^\#_h, A^\#_h)$ has rank  $d$, and $\|Q_h\|_\infty/\sigma_d(Q_h(S^\#_h, A^\#_h))\leq \alpha$.
\end{definition}
Any set of valid anchor states and anchor actions must have at least size $d$ in order for the associated anchor submatrix to be rank $d$. As the full matrix $Q_h$ has rank $d$, this also implies that all rows (resp., columns) of $Q_h$ can be written as a linear combination of the rows associated to states $S_h^{\#}$ (resp., columns associated to actions $A_h^{\#}$). The parameter $\alpha$ depends on the quality of the anchor sets; sub-matrices that are close to being singular result in large $\alpha$.  
We remark that assuming knowledge of a minimal set of anchor states and actions is common in literature, i.e., anchor-based topic modelling \cite{Arora2012LearningTM, chen2017survivalsupervised} and linear feature-based RL \cite{wang2021sampleefficient, yang2019sampleoptimal}. Furthermore, Shah et al. \cite{serl} posit that it suffices empirically to choose states and actions that are far from each other as anchor states and actions. However, in the worst case, finding valid anchor states and actions may require significant a priori knowledge about the unknown matrix.

Alternately, anchor states and actions can be randomly constructed for matrices that satisfy standard regularity conditions such as incoherence, commonly used in matrix estimation ~\citep{candes2009exact}.

\begin{definition}[Incoherence]\label{asm:inco}
Let $Q_h \in \R^{|S|\times |A|}$ be a rank-$d$ matrix with singular value decomposition $Q_h = U\Sigma V^\top$ with $U\in\R^{|S|\times d}$ and $V\in\R^{|A|\times d}$. $Q_h$ is $\mu$-incoherent if $ \max_{i \in [|S|]} \|U_{i}\|_2 \leq \sqrt{\mu d/|S|}$ and $ \max_{j \in [|A|]} \|V_{j}\|_2 \leq \sqrt{\mu d/|A|}$, where $U_i$ denotes the $i$-th row of a matrix $U$.
\end{definition}

A small incoherence parameter $\mu$ ensures that the masses of $U$ and $V$ are not too concentrated in a couple of rows or columns. Consequently, a randomly sampled subset of rows (resp., columns) will span the row (resp., column) space, so these subsets of rows and columns contain sufficient information to reconstruct the entire matrix. Both $\mu$ and $\kappa$, the condition number of $Q_h$, will be used in the analysis to show that the entrywise error amplification from the matrix estimation method scales with $\mu, d, \kappa$ instead of the size of the state or action space, $k$, or $\alpha$.



\medskip \noindent
{\bf Discussion of Assumptions.}
\cyedit{While low rank structure with incoherence is widely accepted in the matrix and tensor estimation literature, we provide a few examples to illustrate how these properties could also naturally arise in MDPs.}
Consider a continuous MDP which is converted to a tabular MDP via discretization, which is a common approach for tackling continuous MDPs. As the size of the discretization is artificial, the true complexity of the MDP is governed by the structure of the continuous MDP, which is independent of the discretization size. As long as the reward function and dynamics are sufficiently smooth with respect to the continuous MDP, the resulting tabular MDP will have approximate low-rank structure as $d$ would be at most logarithmic with respect to $|S|, |A|$, due to a universal low rank property of smooth functions \cite{udell}. \cyedit{Additionally, the incoherence condition intuitively states that there cannot be a disproportionately small set of rows or columns that represent a disproportionately large amount of the signal. For MDPs that are derived from uniform discretizations of continuous MDPs with smoothness properties, incoherence also arises naturally as there will be a constant fraction of the rows or columns representing any fixed length interval of the continuous state space.} Even in inherently discrete settings such as a recommendation system with users and movies, when the population is sufficiently large, one could view the discrete population of states/actions as representing a sample from an underlying continuous population with appropriate smoothness conditions. 
Finally, in many physical systems as relevant to most stochastic control tasks, there exist low dimensional feature representations that capture the ``sufficient statistics'' of the state, which fully govern the dynamics of the system. \cydelete{\tsedit{Another example of low dimensional structure that differs from the ones presented in this work is the ``state space collapse'' phenomenon which is observed in many queueing systems. In particular, it has often been observed that under an optimal control policy, it is sufficient to track a low dimensional representation of the state, even though the full state representation could be high dimensional}.}

\cydelete{\tsedit{One example of MDPs that satisfy our low-rank assumptions are MDPs converted from continuous MDPs via discretization, i.e., discretizing the state and action space and normalizing the transition kernel restricted to the discretized states and actions. These tabular MDP will have approximately low-rank structure as long as the reward function and dynamics of the continuous MDP are sufficiently smooth \cite{udell}. Since the reward function and dynamics are sufficiently smooth, the complexity/rank of the discretized MDP depends on the reward function and dynamics instead of the discretization size.}}

\section{Algorithm} \label{sec:algorithm}

Our algorithm follows from a natural synthesis of matrix estimation with empirical value iteration and Monte Carlo policy iteration. We first describe the vanilla approximate dynamic programming algorithms for the general tabular MDP settings. Empirical value iteration  simply replaces the expectation in the Bellman update in Equation~\eqref{eq:Bellman} with empirical samples \citep{evi}. Specifically, to estimate $Q^*_h(s,a)$, one collects $N$ samples of one step transitions, which entails sampling a reward and next state from $R_h(s,a)$ and $P_h(\cdot | s,a)$. Let $\hat{r}_h(s,a)$ denote the empirical average reward of the $N$ samples from $R_h(s,a)$. Let $\hat{P}_h(\cdot | s,a)$ denote the empirical distribution over $N$ next states sampled from $P_h(\cdot | s,a)$. Given an estimate $\hat{V}_{h+1}$ for the optimal value function at step $h+1$,  the empirical Bellman update equation is
\begin{align}
\hat{Q}_{h}(s,a) = \hat{r}_h(s,a) + \E_{s' \sim \hat{P}_h(\cdot|s,a)}[\hat{V}_{h+1}(s')], ~\text{ and }~
\hat{V}_h(s) = \max_{a \in A} \hat{Q}_{h}(s,a). \label{eq:Q_hat_EVI}
\end{align}
Evaluating $\hat{Q}_h$ and $\hat{V}_h$ requires collecting $N$ samples for each of the $|S| |A|$ state action pairs $(s,a)$.

Monte Carlo policy iteration for tabular MDPs approximates $Q^{\pi}_h(s,a)$ for a policy $\pi$ by replacing the expectation in the definition \eqref{eq:Q_fn_def} of $Q^{\pi}$ with empirical trajectory samples, which is similar to first-visit Monte Carlo policy evaluation except we use the generative model to start at a specified state-action pair and time step \citep{Sutton1998}. This involves sampling $N$ independent trajectories starting from state-action pair $(s,a)$ at step $h$ and following a given policy $\pi$ until the end of the horizon $H$. For a fixed policy $\pi$ and state action pair $(s,a)$, let the sequence of rewards along the $i$-th sampled trajectory be denoted $(r_h^i, r_{h+1}^i, \dots r_H^i)$. We will use $\hat{r}^{\text{cum}}_h(s,a)$ to denote the empirical average cumulative reward across the $N$ trajectories, given by
\begin{align}\label{eq:r_cum}
\hat{r}^{\text{cum}}_h(s,a) := \tfrac{1}{N} \textstyle\sum_{i=1}^N \textstyle\sum_{t = h}^H r_t^i.
\end{align}
Given an estimate of the optimal policy for steps greater than $h$, denoted by $(\hat{\pi}_{h+1}, \hat{\pi}_{h+2}, \dots \hat{\pi}_H)$, the Monte Carlo estimate for the optimal action-value function and optimal policy at step $h$ are
\begin{align}
\hat{Q}_h(s,a) = \hat{r}^{\text{cum}}_h(s,a), ~\text{ and }~
\hat{\pi}_h(s) = \delta_{a} ~\text{ for } a = \argmax_{a' \in A} \hat{Q}_h(s,a'), \label{eq:Q_hat_MCPI}
\end{align}
where the trajectories used to compute $\hat{r}^{\text{cum}}_h(s,a)$ are sampled by following the policy $(\hat{\pi}_{h+1}, \hat{\pi}_{h+2}, \dots \hat{\pi}_H)$, and recall $\delta_{a}$ denotes the distribution that puts probability 1 on action $a$. Computing $\hat{Q}_h$ and $\hat{\pi}_h$ involves sampling $|S| |A| N$ trajectories, which are each of length $H - h$, which results in a sample complexity of $|S| |A| N (H-h)$ individual transitions from the MDP.

The dependence on $|S| |A|$ in the sample complexity for both of the classical algorithms described above is due to using empirical samples to evaluate $\hat{Q}_h$ for every state-action pair $(s,a) \in S \times A$. The assumption that $Q^*_h$ is at most rank $d$ imposes constraints on the relationship between $Q^*_h(s,a)$ at different state-action pairs, such that by approximating $Q^*_h$ using empirical samples at only $O(d|S| + d|A|)$ locations, we should intuitively be able to use the low rank constraint to predict the remaining entries. Let $\Omega_h \subset S \times A$ denote the subset of entries $(s,a)$ for which we use empirical samples to approximate $\hat{Q}_h(s,a)$, computed via either \eqref{eq:Q_hat_EVI} or \eqref{eq:Q_hat_MCPI}. Given estimates of $\hat{Q}_h(s,a)$ at $(s,a) \in \Omega_h$, we can then use a low-rank matrix estimation subroutine to estimate the $Q$ function for $(s,a) \not\in \Omega$. This is the core concept of our algorithm, which we then combine with the two classical approaches of empirical value iteration and Monte Carlo policy iteration.

\subsection{Formal Algorithm Statement} \label{sec:algoStatement}
We present two Low Rank RL algorithms, which take as input any \tsedit{matrix estimation algorithm, \texttt{ME}$(\cdot)$, that takes in a subset of entries of the matrix and returns an estimate of the whole matrix, the sets $\{\Omega_h\}_{h\in[H]}$ that indicate the state action pairs for which data should be collected by querying the MDP generative model, and $\{N_{s, a, h}\}_{(s, a, h) \in S\times A\times H}$, which denotes how many samples to query at state-action pair $(s,a)$ at timestep $h$.} We use ``Low Rank Empirical Value Iteration'' (LR-EVI) to refer to the algorithm which uses option (a) for Step 1 below, and we use ``Low Rank Monte Carlo Policy Iteration'' (LR-MCPI) to refer to the algorithm which uses option (b) for Step 1. \tsdelete{For ease of notation, we define
\[
N_{s, a, h} \coloneqq \begin{cases} N_h^{\#} \text{ if } (s, a) \in \Omega_h^{\#} = S^{\#}_h \times A^{\#}_h\\
N_h \text{ otherwise.}
\end{cases}
\]}

\medskip \noindent
{\bf Hyperparameters:} $\{\Omega_h\}_{h\in [H]}, \{N_{s, a, h}\}_{(s, a, h) \in S\times A\times H}$, and \texttt{ME}$(\cdot)$

\medskip \noindent
{\bf Initialize:} Set $\hat{V}_{H+1}(s) = 0$ for all $s$, and let $\hat{\pi}^{H+1}$ be any arbitrary policy.

\medskip\noindent
For each $h \in \{H, H-1, H-2, \dots 1\}$ in descending order,
\begin{itemize}

\item {\bf Step 1:} For each $(s,a) \in \Omega_h$, compute $\hat{Q}_h(s,a)$ using empirical estimates according to either (a) empirical value iteration or (b) Monte Carlo policy evaluation.
\begin{itemize}
\item[(a)] {\bf Empirical Value Iteration:} Collect $N_{s, a, h}$
  samples of a single transition starting from state $s$ and action $a$ at step $h$. 
Use the samples to estimate $\hat{Q}_h(s,a)$ according to
\begin{align*}
\hat{Q}_{h}(s,a) &= \hat{r}_h(s,a) + \E_{s' \sim \hat{P}_h(\cdot|s,a)}[\hat{V}_{h+1}(s')],
\end{align*}
where $\hat{r}_h(s,a)$ denotes the empirical average reward of the $N_{s, a, h}$ samples from $R_h(s,a)$, and $\hat{P}_h(\cdot | s,a)$ denotes the empirical distribution over the $N_{s, a, h}$ states sampled from $P_h(\cdot | s,a)$.
\item[(b)] {\bf Monte Carlo Policy Evaluation:} Collect $N_{s, a, h}$ independent full trajectories starting from state $s$ and action $a$ at step $h$ until the end of the horizon $H$, where actions are chosen according to the estimated policy $(\hat{\pi}_{h+1}, \hat{\pi}_{h+2}, \dots \hat{\pi}_H)$. Let $\hat{Q}_h(s,a) = \hat{r}^{\text{cum}}_h(s,a)$, where $\hat{r}^{\text{cum}}_h(s,a)$ denotes the empirical average cumulative reward across the $N_{\cyedit{s,a,}h}$ trajectories starting from $(s,a)$ at step $h$. If $(r_h^i, r_{h+1}^i, \dots r_H^i)$ denotes the sequence of rewards along the $i$-th sampled trajectory from $(s,a)$ at step $h$, then 
\begin{align*}
\hat{Q}_h(s,a) = \hat{r}^{\text{cum}}_h(s,a) := \tfrac{1}{N_{s, a, h}} \textstyle\sum_{i=1}^{N_{s, a, h}} \textstyle\sum_{t = h}^H r_t^i.
\end{align*} 
\end{itemize}
\item {\bf Step 2:} Predict the action-value function for all $(s,a) \in S \times A$ according to $\texttt{ME}(\cdot)$:
\begin{align*}
\bar{Q}_h = \texttt{ME}\left(\{\hat{Q}_h(s,a)\}_{(s,a) \in \Omega_h}\right).
\end{align*}
\item {\bf Step 3:} Compute the estimates of the value function and the optimal policy according to
\begin{align*}
\hat{V}_h(s) = \max_{a \in A} \bar{Q}_{h}(s,a) ~\text{ and }~
\hat{\pi}_h(s) = \delta_{\argmax \bar{Q}_h(s,a)}.
\end{align*}
\end{itemize}

    The tabular MDP variant of the algorithm proposed in \cite{serl} is equivalent to LR-EVI where anchor states $S_h^\#$ and actions $A_h^\#$ are given and $\Omega_h = (S_h^{\#} \times A) \cup (S \times A_h^{\#})$. Furthermore, LR-EVI is equivalent to a modification of the algorithm in \cite{Yang2020Harnessing} with a different choice for the matrix estimation algorithm used in Step 2 and the corresponding sample set $\Omega_h$ constructed in Step 1. \tsdelete{See Appendix \ref{sec:pseudo} for the pseudo code of LR-MCPI and LR-EVI with the matrix estimation subroutine defined below.}

\subsection{Matrix Estimation Subroutine} \label{sec:MES}

A critical piece to specify for the algorithm above is how to choose the subset $\Omega_h$, and what matrix estimation subroutine \texttt{ME}$(\cdot)$ to use to predict the full $Q_h$ function, where $Q_h$ is $Q^*_h$, $Q^\pi_h$, or $Q'_h = r_h + P_h\hat{V}_{h+1}$ depending on the low-rank setting, given $\hat{Q}_h(s,a)$ for $(s,a) \in \Omega_h$. The performance of any matrix estimation algorithm will depend both on the selected subset $\Omega_h$, as well as the entrywise noise distribution on $\hat{Q}_h(s,a)$ relative to the ``ground truth'' low-rank matrix. As a result, the subset $\Omega_h$ should be determined jointly with the choice of matrix estimation algorithm.

A limitation of a majority of results in the classical matrix estimation literature is that they do not admit entrywise bounds on the estimation error, and the analyses may be sensitive to the distribution of the observation error \tsedit{, i.e., require mean-zero sub-Gaussian noise. When estimating $Q^*_h$, the observations are biased unless one has learned the optimal policy at time steps $h +1$ to $H$. Since $Q_h$ is low-rank under our assumptions, our estimates of the observations for the matrix estimation method are unbiased with bounded noise, therefore enabling us to relax the small discount factor requirement. 
}

Many standard analyses of RL algorithms rely upon the construction of entrywise confidence sets for the estimates of the $Q$ function. Our results and analyses rely on entrywise error bounds for the matrix estimation step that balance the worst case entrywise error amplification with the size of the observation set. As such, similar theoretical guarantees can be obtained for our algorithm under any matrix estimation method that admits suitable entrywise error bounds.

The majority of our theoretical results will be shown for the variant of the algorithm that uses a matrix estimation algorithm from \cite{serl}, which is incidentally equivalent to exploiting a skeleton decomposition of a low rank matrix \cite{pseudo}.
Their algorithm uses a specific sampling pattern, in which $\Omega_h$ is constructed according to $\Omega_h = (S^{\#} \times A) \cup (S \times A^{\#})$, where $S_h^{\#}$ and $A_h^{\#}$ are assumed to be valid anchor states and actions for the matrix $Q_h$ (cf.\ Definition~\ref{asm:anchor}). Given estimates $\hat{Q}_h(s,a)$ for all $(s,a) \in \Omega_h$, their algorithm estimates the $Q$ function at all state action pairs according to
\begin{align} \label{eq:Qbar_estimate}
\bar{Q}_h(s,a) = \hat{Q}_h(s, A^{\#})\left[ \hat{Q}_h(S^{\#},A^{\#})\right]^{\dagger} \hat{Q}_h(S^{\#}, a),
\end{align}
where $M^{\dagger}$ denotes the pseudoinverse of $M$, and $\bar{Q}$ is the output of the matrix estimation algorithm. 
The simple explicit formula for the estimates enables direct entrywise error bounds. Instead of ensuring a uniform error bound over each state-action pair in $(S^{\#} \times A) \cup (S \times A^{\#})$, we show that additional sampling of the anchor submatrix $\Omega_h^{\#} = S^{\#} \times A^{\#}$ yields a smaller error amplification compared to the method propsosed in \cite{serl}.
In addition, we show that if $Q_h$ is $\mu$-incoherent, introduced in Definition~\ref{asm:inco}, $\Tilde{O}(\mu d, \kappa)$-anchor states and actions  can be constructed randomly by including each state in $S^{\#}$ independently with probability $p_1 = \Theta(\mu d \log(|S|)/|S|)$ and including each action in $A^{\#}$ independently with probability $p_2 = \Theta(\mu d \log(|S|)/|S|)$. 
As a result, a priori knowledge of the anchor states and actions is not required under these regularity conditions.

In Section \ref{sec:CME}, we show that our theoretical results also extend to the variation of our algorithm that uses soft nuclear norm minimization for matrix estimation alongside uniform Bernoulli sampling, utilizing entrywise guarantees shown in \cite{chen2019noisy}. \tsedit{One matrix estimation algorithm that solves the soft nuclear norm minimization problem is \texttt{Soft-Impute} \cite{softImpute}.
\texttt{Soft-Impute} proceeds by iteratively filling in the missing values by using a soft-thresholded singular value decomposition on the matrix of observed entries. In contrast to the sampling pattern used in the matrix estimation method given in Equation \ref{eqn:me}, the sampling pattern needed to ensure the entrywise guarantees from \cite{chen2019noisy} assumes that each state-action pair is observed with probability $p_{SI}$. 

We use LR-EVI (resp., LR-MCPI) $+$ \texttt{SI} to refer to the algorithm that uses option (a) (resp., option (b)) for Step 1 and \texttt{Soft-Impute} as the matrix estimation method.
In Section \ref{sec:experiments}, we empirically compare LR-EVI and LR-MCPI for both variations of matrix estimation algorithms.}

\section{Main Results}\label{sec:results}

In this section, we present the sample complexity, i.e., an upper bound on the number of observed samples of the reward and next state, guarantees for LR-MCPI and LR-EVI with the matrix estimation method presented in Section \ref{sec:MES} under different low-rank assumptions, from the weakest to the strongest. \cyedit{For $(s,a) \notin \Omega_h^{\#} = S^{\#}_h \times A^{\#}_h$, we denote $N_{s,a,h} = N_h$, For $(s,a) \in \Omega_h^{\#} = S^{\#}_h \times A^{\#}_h$, we denote $N_{s,a,h} = N_h^{\#} = \alpha^2 k^2 N_h$, such that entries in the anchor submatrix get a factor of $\alpha^2 k^2$ more samples.}\cydelete{\tsedit{In the following theorems, $N_{s, a, h}$ is the number of samples the algorithm observes at each $(s,a) \in \Omega_h$, and recall $N_{s, a, h}$ is defined as,
\[
N_{s, a, h} \coloneqq \begin{cases} N_h^{\#} \text{ if } (s, a) \in \Omega_h^{\#} = S^{\#}_h \times A^{\#}_h\\
N_h \text{ otherwise.}
\end{cases}
\]
}}
\begin{theorem} \label{thm:gap} 
Assume that $Q^*_h$ is rank $d$ and has suboptimality gap $\Delta_{\min}$ (Assumptions \ref{asm:lrqt} and \ref{asm:sog}), and $S^\#_h, A^\#_h$ are $(k, \alpha)$-anchor states and actions for  $Q^*_h$ for all $h \in [H]$. Let $N_{h} = \Tilde{O}\left( (H-h+1)^2 \alpha^2 k^2/\Delta_{\min}^2 \right)$ and $N_{h}^{\#}  = \alpha^2k^2 N_h$. LR-MCPI returns an optimal policy with probability at least $1-\delta $ with a sample complexity of $
\Tilde{O}\left((|S|+|A|)\alpha^2 k^3 H^4/\Delta_{min}^2 + 
\alpha^4k^6 H^4/\Delta_{\min}^2\right)$.
\end{theorem}

\tsedit{The dependence on the rank $d$ is not explicitly shown in the sample complexities stated in these theorems as it is captured by $k$, which we bound with Lemma \ref{lem:anchor_sampling} (presented later in this section). }
In the tabular setting, there always exists a $\Delta_{\min} > 0$. This sample complexity improves upon $|S||A|$ when  $\Delta_{\min}$ is greater than $|S|^{-1/2} \wedge |A|^{-1/2}$. When $\Delta_{\min}$ is small, if stronger low-rank assumptions also hold, then the results in Theorems \ref{thm:qnolr} and \ref{thm:tklr} below may provide stronger bounds. 

Under the assumption that the $Q_h^{\pi}$ function is low rank for all $\eps$-optimal policies, Theorem \ref{thm:qnolr} states that LR-MCPI learns an $\eps$-optimal policy with a sample complexity independent of $\Delta_{\min}$.


\begin{theorem}\label{thm:qnolr} 
Assume that for all $\eps$-optimal policies $\pi$, $Q^{\pi}_h$ is rank $d$ (Assumption \ref{asm:lrqe}), and  $S^\#_h, A^\#_h$ are $(k, \alpha)$-anchor states and actions for $Q^{\hat{\pi}}_h$, where $\hat{\pi}$ is the learned policy from LR-MCPI for all $h \in [H]$. Let $N_{h} = \Tilde{O}\left((H-h+1)^2\alpha^2 k^2H^2/\eps^2\right)$ and $ N^{\#}_h =  \alpha^2k^2 N_h$. Then, LR-MCPI returns an $\eps$-optimal policy and action-value function 
with probability at least $1-\delta$ with a sample complexity of $
\Tilde{O}\left((|S|+|A|)\alpha^2 k^3 H^6 /\eps^2 
+\alpha^4k^6H^6/\eps^2 \right)$.
\end{theorem}


The strongest assumption that the transition kernel has low Tucker rank and the reward function is low rank, implies that $Q^{\pi}_h$ for all policies $\pi$ is low rank. As such, the result in Theorem~\ref{thm:qnolr} also implies an efficient sample complexity guarantee for LR-MCPI under Assumption~\ref{asm:lrtk}. We can further remove a factor of $H$ by using LR-EVI instead. Empirical value iteration (see Step 1(a)) reduces the sample complexity by a factor of $H$ since it does not require sampling a full rollout of the policy to the end of the horizon, as required for the Monte Carlo estimates (see Step 1(b)).

\begin{theorem} \label{thm:tklr} 
Assume that for any $\eps$-optimal value function $V_{h+1}$, the matrix corresponding to $Q'_h = [r_h + [P_h V_{h+1}]]$ is rank $d$ (a consequence of Assumption \ref{asm:lrtk}), and  $S^\#_h, A^\#_h$ are $(k, \alpha)$-anchor states and actions for $\hat{Q}'_h = [r_h + [P_h \hat{V}_{h+1}]]$, where $\hat{V}_{h+1}$ is the learned value function from LR-EVI for all $h \in [H]$. Let $N_{h} = \Tilde{O}\left((H-h+1)^2 \alpha^2 k^2 H^2 / \eps^2 \right)$ and $N_h^{\#} = \alpha^2k^2 N_h$. Then, LR-EVI returns an $\eps$-optimal $Q$ function and policy 
with probability at least $1-\delta$ with a sample complexity of  $
\Tilde{O}\left((|S|+ |A|)\alpha^2k^3 H^5 /\eps^2 + \alpha^4k^6H^5/\eps^2 \right)$.
\end{theorem}

From Proposition \ref{prop:lrEst}, under Assumption \ref{asm:lrtk} (low-rank transition kernel and expected rewards), the matrix corresponding to $[r_h + [P_h\hat{V}_{h+1}]]$ has rank $d$ for any value function estimate $\hat{V}_{h+1}$. This is critical to the analysis of LR-EVI as it guarantees that the expectation of the matrix $\bar{Q}_h$ constructed from Empirical Value Iteration in Step 1(a) is low rank. This property is not satisfied by Assumptions \ref{asm:lrqe} and \ref{asm:lrqt}, and as such the analysis for Theorem \ref{thm:tklr} does not extend to these weaker settings. Additionally, this property eliminates the need for constructing estimates with rollouts, which removes a factor of $H$ in the sample complexity compared to LR-MCPI under Assumption \ref{asm:lrqe}. 

Our sample complexity bounds depend on $k, \alpha$, presuming that the algorithm uses some given set of $(k, \alpha)$-anchor states and actions. When there may not be a domain expert to suggest anchor states and actions, we show in the next lemma that one can construct $(k, \alpha)$-anchor states and actions with high probability by random sampling, where $k$ and $\alpha$ scale with the incoherence and the bounded condition number of the target matrix. 
\begin{lemma} \label{lem:anchor_sampling}
Let $Q_h$ be a rank $d$, $\mu$-incoherent matrix with condition number $\kappa$. Let $S^{\#}$ and $A^{\#}$ be constructed randomly such that each state $s$ is included in $S^{\#}$ with probability $p_1 = \Theta(d\mu/\log(|S|))$, and each action $a$ is included in $A^{\#}$ with probability $p_2 = \Theta(d\mu/\log(|A|))$. With probability $1 - O\left(H(|S|\wedge |A|)^{-10}\right)$, $S^{\#}$ and $A^{\#}$ are $(k,\alpha)$ anchor states and actions for $Q_h$ for $k = \tilde{O}(\mu d)$ and $\alpha = O(\kappa)$. 
\end{lemma}
Lemma \ref{lem:anchor_sampling} asserts that without a priori knowledge, one can find a set of $\Tilde{O}(\mu d, \kappa)$-anchor states and actions using the sampling subroutine defined in Section \ref{sec:MES}, given that the corresponding matrix is $\mu$-incoherent with condition number $\kappa$. 

\medskip \noindent
{\bf Comparison to Impossibility Result in Theorem~\ref{thm:explb}. }
Recall that Theorem~\ref{thm:explb} establishes an exponential $4^H$ lower bound for learning a near-optimal policy in MDPs with low-rank $Q^*$. While the constructed MDP has a constant suboptimality gap, the lower bound does not contradict Theorem~\ref{thm:gap} which achieves a poly$(H)$ sample complexity for LR-MCPI under a \cyedit{stronger} generative model, i.e. after estimating the optimal action at step $h$, LR-MCPI can subsequently sample full trajectories from the estimate of the optimal policy, which would then include entry $(2,2)$, which was prohibited in the setup of Theorem~\ref{thm:explb}. In contrast, LR-EVI does not admit an efficient sample complexity for the MDP constructed in Section \ref{sec:lowerbound}, and one can show that it exhibits exponential blowup in the estimation error due to an amplification of the estimation error in the terminal step when propagating the estimates backwards via value iteration. The MDP does not have a low rank transition kernel violating Assumption \ref{asm:lrtk}, as needed for Theorem \ref{thm:tklr}.

\subsection{Discussion of Optimality} \label{sec:disc}

Theorems \ref{thm:gap}, \ref{thm:qnolr} and \ref{thm:tklr} show that under our various low rank assumptions, LR-MCPI and LR-EVI learn near-optimal polices in a sample efficient manner, decreasing the dependence of sample complexity on $S$ and $A$ from $|S||A|$ to $|S|+|A|$. Furthermore in Lemma \ref{lem:lrlb} we establish a $d(|S|+|A|)H^3/\eps^2$ sample complexity lower bound for MDPs with low rank reward and transition kernel in the sense of Assumption~\ref{asm:lrtk} via minor modifications of existing lower bounds for tabular MDPs. Since Assumption~\ref{asm:lrtk} implies the optimal $Q^*$ function is low rank, the same lower bound holds for the latter setting. Comparing our results to the lower bound, it follows that the dependence on $|S|,|A|$, and $\eps$ in our sample complexity upper bound is minimax optimal.

\begin{lemma}\label{lem:lrlb}
For any algorithm, there exists an MDP $M = (S, A, P, R, H)$ with rank $d$ reward $R_h$ and transition kernel $P_h$ for all $h \in [H]$ such that $\Omega \left( d(|S|+|A|)H^3/\eps^2 \right)$ samples are needed to learn an $\eps$-optimal policy with high probability. 
\end{lemma}

\begin{proof}
Existing lower bounds from \cite{sidford2019nearoptimal} prove the necessity of $\Omega(d|S|H^3/\eps^2)$ samples to learn an $\eps$-optimal policy with high probability for a time-homogeneous MDP with $|S|$ states and $d$ actions. Replicating each action $|A|/d$ times results in an MDP $|A|$ actions and rank $d$ reward functions and  transition kernels, and this MDP is at least as hard as the original MDP. Repeating this construction with an MDP with $d$ states and $|A|$ actions proves an $\Omega(d|A|H^3/\eps^2)$ sample complexity lower bound. Combining these two lower bounds proves the lemma.
\end{proof}

As an aside we also point out that previously shown lower bounds for linearly-realizable MDPs \cite{wang2021exponential, weisz2021exponential} are not directly applicable to our setting, as the constructed instances therein need not have low-rank $Q^*$ or transition kernels, and the size of their state space scales exponentially in $d$.

Our sample complexity bounds depend on $k$ and $\alpha$, the size and quality of the $(k,\alpha)$-anchor sets. As stated in Lemma \ref{lem:anchor_sampling}, we can construct a set of $\Tilde{O}(\mu d, \kappa)$-anchor states and actions for any $\mu$-incoherent matrix with condition number $\kappa$ simply by randomly sampling a subset of state and action. 
The results presented in the table in Section \ref{sec:introduction} are obtained by substituting $k = \Tilde{O}(d\mu)$ and $\alpha = O(\kappa)$ into the sample complexity bounds in Theorems \ref{thm:gap}, \ref{thm:qnolr} and \ref{thm:tklr} and treating $\mu$ and $\kappa$ as constants, as is standard in the matrix estimation literature, e.g., \cite{abbe2020entrywise}. 

In the event that there is a domain expert who provides a set of $(k,\alpha)$-anchor states and actions, then the sample complexity bound may be better by using the given set rather than randomly sampling if $\mu$ and $\kappa$ are large. Note that $k$ must be minimally at least $d$, 
and the quality of a given set of anchor states and actions depends on the smallest singular value associated to the anchor submatrix as reflected in $\alpha$, which for poorly chosen anchor state and actions could scale with $H$.

In Theorems \ref{thm:gap}, \ref{thm:qnolr} and~\ref{thm:tklr}, the cubic dependence on $d$ is likely suboptimal, but this results from the  suboptimal dependence on $d$ in the corresponding entrywise error bounds in the matrix estimation literature \cite{serl, abbe2020entrywise, chen2019noisy}. Without knowledge of good anchor states/actions from a domain expert, the dependence on $\mu$ that arises from randomly sampling anchor states and actions is not surprising, as it also commonly arises in the classical matrix estimation literature under uniform sampling models. Any improvements in the matrix estimation literature on the dependence on $d,\mu$ would directly translate into improved bounds via our results.


\tsedit{Our dependence on the horizon $H$ is fairly standard as it matches the dependence on $H$ for vanilla $Q$-value iteration. There is a gap between the dependence on $H$ in our upper bounds and the $H^3$ lower bound in Lemma~\ref{lem:lrlb}, which is given for homogeneous MDPs. Our upper bound results allow for nonhomogenous rewards and transition kernels, which would likely increase the lower bound to $H^4$. Reducing the upper bounds to $H^4$ would likely require using the total variance technique from \cite{sidford2019nearoptimal}, which requires estimates of the variance of the policy at a given state-action pair. One can show that the variance of the Bellman operator is low rank under the strongest assumption of a low Tucker rank transition kernel, but the corresponding rank of the matrix of variances is $O(d^2)$. Hence, while it may be possible to adapt this variance technique to achieve the optimal dependence on $H$ in our low-rank settings, doing so may incur a significantly worse dependence on $d$, i.e., $d^{6}$. }

\subsection{Proof Sketch}\label{sec:proofSketch}
The analysis of LR-MCPI and LR-EVI are fairly similar, and involves first showing that upon each application of the matrix estimation subroutine stated in \eqref{eq:Qbar_estimate}, the amplification of the entrywise error is bounded, as stated below in Lemma \ref{lem:ME_error_amplification}.
\begin{lemma} \label{lem:ME_error_amplification}
Let $S^{\#}$ and $A^{\#}$ be $(k, \alpha)$-anchor states and actions for matrix $Q_h$. For all $(s,a) \in \Omega_h = (S^{\#} \times A) \cup (S \times A^{\#}) \setminus (S^{\#} \times A^{\#})$, assume that $\hat{Q}_h(s,a)$ satisfies 
$|\hat{Q}_h(s,a) - Q_h(s,a)| \leq \eta$, and for all $(s, a) \in S^{\#} \times A^{\#}$, assume that $\hat{Q}_h(s,a)$ satisfies 
$|\hat{Q}_h(s,a) - Q_h(s,a)| \leq \eta^{\#}$. Then, for all $(s,a) \in S \times A$, the estimates $\bar{Q}_h(s,a)$ computed via \eqref{eq:Qbar_estimate} satisfy 
\[\left|\bar{Q}_h(s,a) - Q_h(s,a)\right| = O(\alpha k \eta + \alpha^2k^2 \eta^{\#}).\]
\end{lemma}
\paragraph{Proof Sketch for Lemma \ref{lem:ME_error_amplification}} \tsedit{As our algorithm constructs $\hat{Q}_h(s,a)$ for $(s,a) \in \Omega_h$ via averaging over samples from the MDP, the condition $|\hat{Q}_h(s,a) - Q_h(s,a)| \leq \eta$ is satisfied with high probability for $\eta = O((H-h)/\sqrt{N_{s, a, h}})$, shown via a simple application of Hoeffding's inequality.} To prove Lemma \ref{lem:ME_error_amplification}, we show that the error is bounded by
\begin{align*}
\left|\bar{Q}_h(s,a) - Q_h(s,a)\right| 
&\lesssim \left \|[\hat{Q}_h(S^\#, A^\#)]^\dagger \right\|_{op} ~~\cdot~~\left \|\hat{Q}_h(S^\#, a)\hat{Q}_h(s, A^\#) - Q_h(S^\#, a)Q_h(s, A^\#)   \right \|_F\\
&+
\left\| \left[ \hat{Q}_h(S^{\#},A^{\#})\right]^{\dagger} - \left[ Q_h(S^{\#},A^{\#})\right]^{\dagger} \right\|_{op} ~~\cdot~~
\left\|Q_h(S^{\#}, a) Q_h(s, A^{\#})\right\|_F \\
&\lesssim \left(\frac{1}{\sigma_d(Q_h(S^{\#},A^{\#})))} \right)~~\cdot~~ k \|Q_h\|_{\infty}(2\eta + \eta^2) \\
&+ \left( \frac{\eta^{\#} k}{(\sigma_d(Q_h(S^{\#},A^{\#})))^2}\right) ~~\cdot~~ \|Q_h\|^2_{\infty} k = O(\alpha k \eta + \alpha^2 k^2 \eta^{\#}).
\end{align*}
\tsedit{The first inequality comes from an application of the triangle inequality and the definition of the operator norm since for any rank $d$ matrix $Q$ with $(k, \alpha)$-anchor states and actions, for all $(s, a) \in S\times A$, $Q(s, a) = Q(s, A^\#) [Q(S^\#, A^\#)]^\dagger Q(S^\#, a) $. The operator norm terms are bounded using Weyl's inequality and a classic result from the perturbation of pseudoinverses. The other two terms are bounded by our assumption on $\hat{Q}_h$ and that the reward functions are bounded by one. }

As $\eta$ is the dominant error term as $\eta^{\#}$ is the error on the small anchor sub-matrix, $\{\bar{Q}_h(s, a)\}_{(s,a) \in S^{\#}\times A^{\#}}$ with size $\tilde{O}(k) \times \tilde{O}(k)$, the critical insight from Lemma \ref{lem:ME_error_amplification} is that the amplification of the error due to matrix estimation is only a factor of $\alpha k$, which is constant for a good choice of anchor states and actions. We set $N_{s, a, h}$ for each $(s,a) \in \Omega_h$ to guarantee $\alpha k \eta$ and $\alpha^2 k^2 \eta^{\#}$ are sufficiently small for a subsequent induction argument that shows the algorithm maintains near optimal estimates of $Q^*$ and $\pi^*$. For each of the Theorems \ref{thm:gap}, \ref{thm:qnolr}, and \ref{thm:tklr}, we will apply Lemma \ref{lem:ME_error_amplification} to different choices of $Q_h$, chosen to guarantee that $\hat{Q}_h(s,a)$ is an unbiased estimate of $Q_h$. For Theorem \ref{thm:gap}, we choose $Q_h = Q^*_h$. For Theorem \ref{thm:qnolr}, we choose $Q_h = Q^{\hat{\pi}}_h$, where $\hat{\pi}$ is an $\eps$-optimal policy. For Theorem \ref{thm:tklr}, we choose $Q_h = r_h + [P_h\hat{V}_{h+1}]$, where $\hat{V}_{h+1}$ is the value function estimate for step $h+1$. 

Choosing $Q_h$ to be potentially distinct from $Q^*_h$  is a simple yet critical distinction between our analysis and \cite{serl}. The analysis in \cite{serl} applies a bound similar to Lemma \ref{lem:ME_error_amplification} with a choice of $Q_h = Q^*_h$. However, as $\hat{Q}_h$ will not be unbiased estimates of $Q^*_h$, the initial error $\eta$ will contain a bias term that is then amplified exponentially in $H$ when combined with an inductive argument for LR-EVI.

\paragraph{Proof Sketch for Lemma \ref{lem:anchor_sampling}} To prove that the random sampling method presented in Section \ref{sec:MES} finds $\tilde{O}(\mu d, \kappa)$-anchor states and actions with high probability, let us denote the singular value decomposition of matrix $Q_h$ with $U \Sigma V^T$. For a randomly sampled set of anchor states and actions $S^{\#}$ and $A^{\#}$, let $\Tilde{U}$ and $\Tilde{V}$ denote the submatrices of $U$ and $V$ limited to $S^{\#}$ and $A^{\#}$, such that the anchor submatrix $Q_h(S^{\#},A^{\#})$ is given by $\Tilde{U}\Sigma\Tilde{V}^T$. By the matrix Bernstein inequality \cite{2011}, when rows and columns are sampled uniformly with probability $p_1 = \Theta(d\mu/\log(|S|)), p_2 = \Theta(d\mu/\log(|A|))$, the columns of $\Tilde{U}$ and $\Tilde{V}$ are nearly orthogonal. In particular, with high probability
\[\|p_1^{-1} \Tilde{U}^T \Tilde{U} - I_{d\times d}\|_{op} \leq \frac12
~~\text{ and }~~ \|p_2^{-1} \Tilde{V}^T \Tilde{V} - I_{d\times d}\|_{op} \leq \frac12,\]
implying that the anchor submatrix is rank $d$. By an application of the singular value version of the Courant-Fischer minimax theorem \cite{horn_johnson_1985}, we can relate $\sigma_d(Q_h(S^{\#},A^{\#}))$ to $\sigma_d(Q_h)$ to show that 
\[\alpha = \max_h \|Q_h\|_\infty/\sigma_d(Q_h(S^\#_h, A^\#_h)) = O(\kappa).\]

\medskip \noindent
{\bf Inductive Argument for Main Theorems.}
The final step is to use the error analysis of each iteration in an inductive argument that argues the estimated policy at each step is near optimal. As the induction argument is similar across all three theorems, we present the inductive argument for Theorem \ref{thm:qnolr}, and refer readers to the Appendix for the full proofs of all the theorems.
For Theorem \ref{thm:qnolr}, the induction step is that if $\hat{\pi}_{H-t+1}$ is $t\eps/H$-optimal, then for time step $H-t$, the policy found with LR-MCPI, $\hat{\pi}_{H-t}$, is $(t+1)\eps/H$-optimal. We then induct backwards across horizon, i.e. $t \in \{1, \dots H\}$.

\cydelete{\tsedit{Recall that we use the following notation for convenience,  \[
N_{s, a, h} \coloneqq \begin{cases} N_h^{\#} \text{ if } (s, a) \in \Omega_h^{\#} = S^{\#}_h \times A^{\#}_h\\
N_h \text{ otherwise.}
\end{cases}
\]}
The base case when $t=1$ is that $\hat{\pi}_H$ is $\epsilon/H$-optimal with high probability. This follows by showing that by Hoeffding's inequality, for 
$N_{H} = \Tilde{O}\left(\alpha^2 k^2H^2/\eps^2\right), N_H^{\#} = \alpha^2 k ^2 N_H$, 
the error of the initial estimates $\hat{Q}_H(s,a)$ for $(s,a) \in \Omega_H$ are small enough, such that an application of Lemma \ref{lem:ME_error_amplification} gives $|\bar{Q}_H(s,a) - Q^*_H(s,a)| \leq \epsilon/2H$. As a result the policy $\hat{\pi}_H$ that results from choosing greedily with respect to the estimates in $\bar{Q}_h$ will be an $\epsilon/H$-optimal policy.}

\tsedit{To show the induction step, first we argue that by Hoeffding's inequality, for $(s,a) \in \Omega_{H-t}$, with high probability $|\hat{Q}_{H-t} - Q^{\hat{\pi}}_{H-t}|= O(\epsilon/\alpha^2 k^2 H)$ for $N_{H-t} =  \Tilde{O}\left((t+1)^2\alpha^2 k^2H^2/\eps^2\right) , N_{H-t}^{\#} = \alpha^2k^2 N_{H-t}$.} It is critical that $\hat{Q}_{H-t}$ are indeed unbiased estimates of $Q^{\hat{\pi}}_{H-t}$ as the estimate is constructed via Monte Carlo rollouts. By Assumption \ref{asm:lrqe} and the inductive hypothesis, $Q^{\hat{\pi}}_{H-t}$ is low rank, such that by an application of Lemma \ref{lem:ME_error_amplification}, it follows that for all $(s,a) \in S \times A$, $|\bar{Q}_{H-t}(s,a) - Q^{\hat{\pi}}_{H-t}(s,a)| \leq \epsilon/2H$ for the appropriate choice of $N_{H-t}$. Finally we argue that assuming the inductive hypothesis, choosing greedily according to $\bar{Q}_{H-t}$ results in a $(t+1)\epsilon/H$-optimal policy. For some state $s$, we denote $a^* = \pi^*_{H-t}(s)$ and $\hat{a} = \hat{\pi}_{H-t}(s) = \max_{a} \bar{Q}_{H+t}(s,a)$. The final induction step is shown via \tsreplace{
\begin{align*}
V^*_{H-t}(s) - V^{\hat{\pi}}_{H-t}(s)
&= Q^*_{H-t}(s,a^*)- \bar{Q}_{H-t}(s,\hat{a})
+ \bar{Q}_{H-t}(s,\hat{a}) - Q^{\hat{\pi}}_{H-t}(s,\hat{a}) \\
&\leq |Q^*_{H-t}(s,a^*)- \bar{Q}_{H-t}(s,a^*)|
+ |\bar{Q}_{H-t}(s,\hat{a}) - Q^{\hat{\pi}}_{H-t}(s,\hat{a})| \\
&\leq |Q^*_{H-t}(s,a^*)-Q^{\hat{\pi}}_{H-t}(s,a^*)| 
+ |Q^{\hat{\pi}}_{H-t}(s,a^*) - \bar{Q}_{H-t}(s,a^*)|
+ \frac{\epsilon}{2H} \\
&\leq \max_{s'} (V^*_{H-t+1}(s')-V^{\hat{\pi}}_{H-t+1}(s'))
+ \frac{\epsilon}{2H}
+ \frac{\epsilon}{2H} \\
&\leq 
\frac{t\epsilon}{H}
+ \frac{\epsilon}{H} = \frac{(t+1)\epsilon}{H}.
\end{align*}}
{
\begin{align*}
V^*_{H-t}(s) - V^{\hat{\pi}}_{H-t}(s)
&= Q^*_{H-t}(s,a^*)- \bar{Q}_{H-t}(s,\hat{a})
+ \bar{Q}_{H-t}(s,\hat{a}) - Q^{\hat{\pi}}_{H-t}(s,\hat{a}) \\
&\leq |Q^*_{H-t}(s,a^*)-Q^{\hat{\pi}}_{H-t}(s,a^*)| 
+ |Q^{\hat{\pi}}_{H-t}(s,a^*) - \bar{Q}_{H-t}(s,a^*)|
+ |\bar{Q}_{H-t}(s,\hat{a}) - Q^{\hat{\pi}}_{H-t}(s,\hat{a})| \\
&\leq \max_{s'} (V^*_{H-t+1}(s')-V^{\hat{\pi}}_{H-t+1}(s'))
+ \frac{\epsilon}{2H}
+ \frac{\epsilon}{2H},
\end{align*}
where $\max_{s'} (V^*_{H-t+1}(s')-V^{\hat{\pi}}_{H-t+1}(s')) \leq t\eps/H$ from the induction hypothesis.}

The proof of Theorem \ref{thm:gap} involves a similar inductive argument except that given the stronger suboptimality gap assumption, we guarantee that $\hat{\pi}_h$ is an exactly optimal policy with high probability. This removes the linear growth in the error across the horizon that arises in Theorem \ref{thm:qnolr}, enabling us to reduce $N_h$ by $H^2$. The proof of Theorem \ref{thm:tklr} also involves a similar inductive argument, but under Assumption \ref{asm:lrtk}, we additionally show that at each time step $Q'_{H-t} = r_{H-t} + [P_{H-t}\hat{V}_{H-t+1}]$, the expected value of $\hat{Q}_{H-t}$ for LR-EVI, is close to not only $Q^*_{H-t}$ but also $Q^{\hat{\pi}}_{H-t}$, which ensures that LR-EVI not only recovers an $\eps$-optimal $Q$ function, but also an $\eps$-optimal policy.

\medskip \noindent
{\bf Sample Complexity Calculation.} The sample complexity of LR-MCPI is given by $\sum_h (H-h) (N_h |\Omega_h| + N_h^{\#}k^2)$, and the sample complexity of LR-EVI is given by $\sum_h (N_h |\Omega_h| + N_h^{\#}k^2)$. The set $|\Omega_h|$ scales as $O(k|S| + k|A|)$, where $k = \tilde{O}(\mu d)$ when the anchor states and actions are sampled randomly. The final sample complexity bounds result from substituting the choices of $N_h$ and $N_h^{\#}$ as specified in the statements of Theorems \ref{thm:gap}, \ref{thm:qnolr}, and \ref{thm:tklr} into the summation.
\subsection{Extension to Approximately Low-Rank MDPs}\label{sec:approx}
Our sample complexity results rely on either $Q^*_h$, $Q^\pi_h$, or $[r_h + P_h\hat{V}_{h+1}]$ having rank $d$, which may only be approximately satisfied. Furthermore, our algorithms require knowledge of the rank of those matrices, which may not be feasible to assume in practice. Hence, we extend our results under the low-rank reward and low Tucker rank transition kernel setting (Assumption \ref{asm:lrtk}) to a $(d,\xi_R, \xi_P)$-approximate low-rank MDP.
\begin{assumption}[$(\cyedit{d},\xi_R, \xi_P)$-Approximate Low-rank MDP]\label{asm:approx}
An MDP specified by $(S, A, P, R, H)$ is a $(d,\xi_R, \xi_P)$-approximate low-rank MDP if for all $h \in [H]$, there exists a rank $d$ matrix $r_{h, d}$ and a low Tucker rank transition kernel $P_{h, d}$ with Tucker rank either $(|S|, d, |A|)$ or $(|S|, |S|, d)$, such that $\forall h$,
\begin{align}
\max_{(s,a) \in S \times A} \left| r_h(s,a) - r_{h,d}(s,a)\right| \leq \xi_R
~~~~\text{ and }~~~~
\sup_{(s,a) \times A} 2d_{\mathrm{TV}}(P_h(\cdot|s, a), P_{h, d}(\cdot| s,a)) \leq \xi_P, \label{eq:approxP}
\end{align}
where $d_{TV}$ is the total variation distance.
\end{assumption}

Assumption \ref{eq:approxP} extends the exact low-rank Assumption \ref{asm:lrtk}, where $\xi_{R}$ is the entrywise low-rank approximation error of the reward function, and $\xi_{P}$ is the low-rank approximation error of the transition kernel in total variation distance. For small values of $\xi_R$ and $\xi_P$, the MDP can be approximated well by a rank $d$ MDP, and subsequently, it follows that for any estimate of the future value function, $r_h + [P_h \hat{V}_{h+1}]$ is close to a corresponding rank $d$ approximation. 
\begin{proposition}\label{prop:lrEstApprox}
Consider a $(d,\xi_R, \xi_P)$-approximate low-rank MDP with $\xi_{R}$, $\xi_{P}, r_{h,d},$ and $P_{h, d}$ as defined in Assumption \ref{eq:approxP}, with respect to the low rank approximation. For all $h \in [H]$ and any $\hat{V}_{h+1}$,
\[
\left|[r_{h,d} +  P_{h, d}\hat{V}_{h+1}] - [r_h +  P_{h}\hat{V}_{h+1}]\right|_\infty \leq \xi_{R} + (H-h) \xi_{P}.
\]
\end{proposition}

Theorem \ref{thm:tklrApprox} shows that LR-EVI with the matrix estimation routine defined in Section \ref{sec:MES} is robust with respect to the low rank approximation error.

\begin{theorem} \label{thm:tklrApprox}
Assume we have a $(d,\xi_R, \xi_P)$-approximate low-rank MDP (Assumption \ref{eq:approxP}) where $r_{h,d}$ and $P_{h,d}$ refer to the corresponding low rank approximations for the reward function and transition kernel. For all $h \in [H]$, let $S^\#_h, A^\#_h$ be $(k, \alpha)$-anchor states and actions for $Q'_{h,d} = [r_{h, d} +  P_{h, d}\hat{V}_{h+1}]$, where $\hat{V}_{h+1}$ is the learned value function from Low Rank Empirical Value iteration. Let $N_{H-t} = \Tilde{O}\left( k^2\alpha^2 H^4/\eps^2\right), N_{H-t}^{\#} = \alpha^2 k^2 N_{H-t}$ for all $t\in \{0,\ldots, H-1\}$. Then LR-EVI returns an $\left(\eps + \Tilde{O}\left( k^2 \alpha^2 \left(\xi_{R}H + \xi_{P} H^2 \right)\right)\right)$-optimal policy with probability at least $1-\delta$  with a sample complexity of $\Tilde{O}\left(k^3\alpha^2(|S|+|A|)H^5 / \eps^2 + k^6\alpha^4H^5/\eps^2 \right)$.
\end{theorem}
The proof of this theorem (see Appendix \ref{app:approx}) follows the same steps as the proof of Theorem \ref{thm:tklr} but additionally accounts for the low rank approximation error using applications of Proposition \ref{prop:lrEstApprox}. Proposition \ref{prop:lrEstApprox} is first used to bound the error between $\hat{Q}_h(s,a)$ and $\cyreplace{[r_{h,d} +  P_{h, d}\hat{V}_{h+1}]}{Q'_{h,d}}(s,a)$ for $(s,a) \in \Omega_h$. Second, the proposition is used to bound the second term in the below inequality which controls the error of our estimate relative to $Q^*_h$ and $Q^{\hat{\pi}}_h$:
\[
|\hat{Q}_h(s,a) - Q^*_h(s,a)| \leq |\hat{Q}_h(s,a) - Q'_{h, d}(s,a)|+ |Q'_{h, d}(s,a) - Q'_{h}(s,a)| + |Q'_{h}(s,a) - Q^*_h(s,a)|
\]
for all $(s,a) \in S \times A$ where $Q'_h = r_h + P_h\hat{V}_{h+1}$ and $Q'_{h, d} = r_{h, d} + P_{h, d}\hat{V}_{h+1}$.

Theorem \ref{thm:tklrApprox} shows that in the approximate rank setting, the error of the policy  our algorithm finds is additive with respect to the approximation error while remaining sample efficient. If $Q'_{h,d}$ is $\mu$-incoherent with condition number $\kappa$, one can use the result of Lemma \ref{lem:anchor_sampling} to find $\Tilde{O}(\mu d, \kappa)$-anchor states and actions without a priori/domain knowledge. 
\section{Experiments} \label{sec:experiments}
\tsedit{We empirically compare the performance of combining low-rank matrix estimation with empirical value iteration and Monte Carlo policy iteration on a tabular version of the Oil Discovery problem \cite{adaptDisc}. 
Our results also join other empirical works that show the benefit of using low-rank variants of RL algorithms on stochastic control problems \cite{Yang2020Harnessing, lrVfa, mlrvfa}.}

\medskip \noindent
\textbf{Experimental Setup:}  We formulate this problem as a finite-horizon MDP, where the state and action spaces are both $\{0,1,\ldots, D\}$ for $D = 399$, and the horizon length $H = 10$. The learner's goal is to locate the oil deposits over a 1 dimensional space, where the target location $l_h = \texttt{round}(400(1- \frac{1}{h}))$ changes with $h$ to make the learning task more difficult. At step $h$ the learner receives a reward $f_h(s)$ that depends on how close the learner is to the oil deposit at $l_h$, perturbed by a zero-mean Gausian noise with variance $\sigma_h^2(s, a)=(0.5 + a/400)^2/10$. The action $a$ chosen indicates what state the learner attempts to move to next, and the learner additionally pays a transportation cost proportional to the distance between $s$ and $a$, denoted by $c(s,a)$. As a result the reward function is
\[
r_h(s, a) = f_h(s) - c(s, a) + \mathcal{N}(0, \sigma_h^2(s,a)),
\]
where we choose $f_h(s)$ and $c(s,a)$ according to
\[
f_h(s) = 1 - \frac{1}{4}\left\lceil \frac{4}{D}\max\left(0, \left |s - l_h \right| - 20 \right) \right\rceil ~~\text{ and }~~ c(s, a) = \cyreplace{\alpha}{0.01} \times \texttt{round}\left(\frac{|s - a|}{100} \right),
\]
where $\texttt{round}(s)$ rounds $s$ to the nearest integer. 
$c(s,a)$ is discretized to take on only 5 distinct values, but the level sets of $c(s,a)$ are diagonal bands, such that $c(s,a)$ is in fact full rank. \tsedit{However, the stable rank of $c(s, a)$, as defined by $\|c(s, a)\|_F^2/ \|c(s, a)\|_*^2$ is only 1.46, which implies that $c(s,a)$ is close to a low-rank matrix \cite{stable}. See Appendix \ref{app:cost} for further discussion about $c(s, a)$. }


The learner's intended movements are perturbed, resulting in the following transition kernel:
\[
\P_h(s'|s, a) = \max\{0, \min \{D, \delta_a + \mathrm{Unif}(-C_h, C_h)\}\}
\]
where $\mathrm{Unif}(-C_h, C_h)$ denotes the discrete uniform distribution over $\{-C_h, -C_h+1,\ldots,C_h\}$ and $C_h = 4(H-h+1)$ \tsedit{\cydelete{is a positive constant that depends only on $h$ and not $s$ or $a$ that} determines the amount of noise in the transitions. Since $\E[ \P_h(s'|s, a) ]$ only depends on the time step $h$ and action $a$, it follows that the rows of $\E[ \P_h(s'|:, :) ]$ are the same and the rank of  $\E[ \P_h(s'|:, :) ]$ is one. Hence, the transition kernel has Tucker rank $(|S|, 1, |A|)$.
}  


\cycomment{Update this discussion if $r$ is not low rank.}
\tsedit{Because the reward function is approximately low-rank and the transition kernel has low Tucker rank, it follows that this MDP is approximately low rank (satisfying Assumption \ref{asm:approx}). See Appendix \ref{app:qStar} for a visualization of $Q^*_1$ and more discussion on the rank of $Q^*_h$.}

\medskip \noindent
\cyedit{\textbf{Algorithms:}} We compare LR-EVI, LR-EVI $+$ \texttt{SI}, LR-MCPI, and LR-MCPI $+$ \texttt{SI} with empirical value iteration (EVI) and Monte Carlo policy iteration (MCPI). Recall that  LR-EVI $+$ \texttt{SI} is essentially the same as LR-EVI 
but uses \texttt{Soft-Impute} from the \texttt{fancyimpute} package \cite{fancyimpute} for the matrix estimation method, whereas LR-EVI uses the matrix estimation algorithm presented in section \ref{sec:MES}.
\cyedit{The observation set, i.e. $\Omega_h$, that is used for \texttt{Soft-Impute} is a Bernoulli sampled subset of entries where the probability of including each entry is denoted $p_{SI}$.} Equivalently LR-MCPI $+$ \texttt{SI} is the same as LR-MCPI except that it uses \texttt{Soft-Impute} with Bernoulli sampled $|\Omega_h|$. The vanilla EVI (resp., MCPI) refers to our algorithm using option (a) (resp., option (b)) for Step 1 without the matrix estimation step, setting $\Omega_h = S \times A$ for all $h \in [H]$ and change Step 2 to be $\bar{Q}_h = \hat{Q}_h$.

To empirically validate the performance of the algorithms, for a fixed sample budget $\overline{N}$, we compare the max entrywise error of $\bar{Q}_1$ of all the algorithms. We test five different allocation schemes on how to distribute the $\overline{N}$ samples across the time steps to determine $N_{s, a, h}$ and use the best one for each algorithm. We ensure that an equal number of samples are allocated to each state-action pair. As $\overline{N}$ may not be divisible by $D^2$, the true samples used is within $D^2$ of $\overline{N}$ due to rounding. We show that LR-EVI and LR-MCPI are robust to $p = p_S = p_A$ as both algorithms perform well for a range of values of $p$, and it suffices to choose $p$ to be small. We perform a grid search to determine $p_{SI}$ for each value of $\overline{N}$, choosing the best performing parameter for each. See Appendix \ref{app:hyp} for the details on how we chose and set the hyperparameters of the algorithms.

\medskip \noindent
\cyedit{\textbf{Results:}} \tsedit{
For each value of $\overline{N} \in [10^6, 10^7, 10^8, 10^9]$, we run each of the above algorithms 10 times.
Figure~\ref{fig:barPlots} shows the average $\ell_\infty$ error of $\bar{Q}_1$ across the 10 simulations, along with error bars whose height indicates one standard deviation above and below the mean. Note that for vanilla EVI to produce an estimate, it requires one sample per $(s,a,h)$, which already requires $1.6 \times 10^6$ samples. For vanilla MCPI to produce an estimate, it requires one trajectory per $(s,a,h)$ of length $H-h+1$, which requires $8.8\times 10^6$ one-step samples. As a result, there is no bar depicted for either EVI or MCPI for $\overline{N} = 10^6$, as both algorithms require more than $10^6$ samples to even produce any estimate.

\begin{figure}[H]
    \centering
    \subfloat{{\includegraphics[width=.45\textwidth]{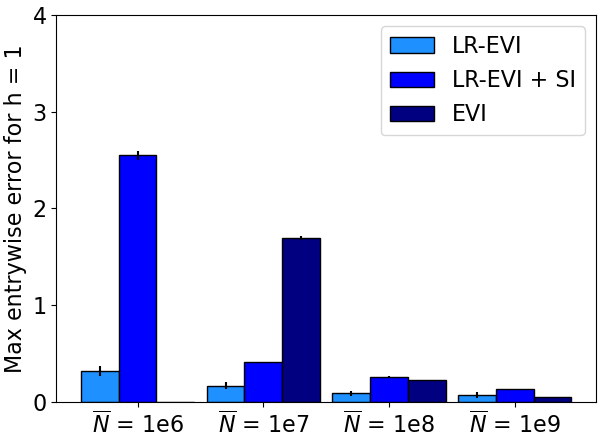} }}%
    \qquad
    \subfloat{{\includegraphics[width=.45\textwidth]{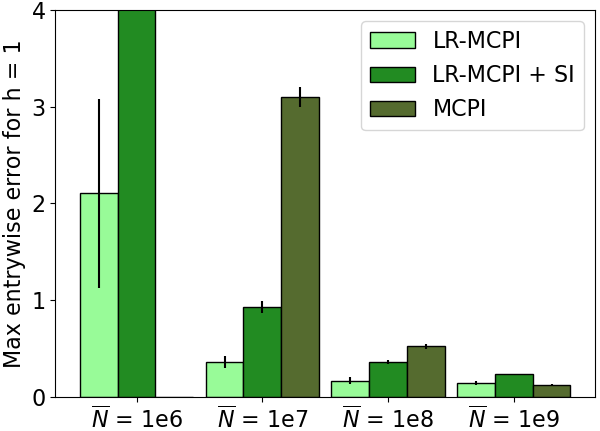} }}%
    \caption{Max entrywise error of $\bar{Q}_1$ vs. sample budget for LR-EVI, LR-EVI $+$ \texttt{Soft Impute}, empirical value iteration and LR-MCPI, LR-MCPI $+$ \texttt{Soft-Impute}, Monte Carlo policy iteration at $h = 1$. Note that the optimal $Q_1^*$ function ranges in value from roughly 8.3 to 9.6, such that 0.8 error would be roughly 10\% error.}
    \label{fig:barPlots}
\end{figure}
\noindent
For $\overline{N} = 10^6$, the error bar for LR-MCPI $+$ \texttt{SI} has a height of $11.5$ but is trimmed to align the $y$-axis in both graphs. 
}
\cyedit{Figure \ref{fig:barPlots} shows that when the sample budget is small ($\overline{N} = 10^6$), the low rank RL algorithms can still produce reasonable estimates even when there are not sufficient samples to even run the vanilla RL algorithms, i.e., less than $8.8\times 10^6$ one-step samples. Our chosen MDP is also not strictly low rank, but only approximately low rank, thus our results validate that our algorithms are not sensitive to the exact rank, as they perform very well on this approximately low rank MDP as well. The Monte Carlo Policy Iteration variants seem to require more samples to achieve the same performance relative to Empirical Value Iteration variants. This is expected as the sample complexity of MCPI is multiplied by $H$ due to sampling entire trajectories rather than one step samples. The MDP in this illustration is well-behaved for LR-EVI as it has a low rank transition kernel, but the practical benefit of LR-MCPI is that it is more robust to MDPs that may not have low rank structure in the transition kernel, as exhibited by the MDP constructed in Section \ref{sec:lowerbound}.}

\tsedit{We also compare the performance of EVI and MCPI and their low-rank variants on the Double Integrator, a stochastic control problem, see Appendix \ref{app:DI} for full details. The results from the Double Integrator simulations also show the benefit of the low-rank methods when the sample budget is small; LR-MPCI produces a reasonable estimate of $Q^*_1$ even when there are not sufficient samples to run MCPI. However, LR-EVI and LR-MPCI are sensitive to the choice of matrix estimation method, so in practice, one should carefully tune the matrix estimation methods' hyperparameters  given computational limits on storage and runtime. When the sample budget is large, the low-rank methods lose their advantage and may even perform worse than tabular variants. }

\tsdelete{Table~\ref{tbl:scompODV} contains the mean, median, and standard deviation of the entrywise error of $\bar{Q}_1$ and sample complexity of LR-EVI, LR-EVI $+$ \texttt{SI}, and EVI for five
trials using the first convergence criterion. Table~\ref{tbl:scompODP} contains the same statistics for LR-MCPI, LR-MCPI $+$ \texttt{SI}, and MCPI. To give context to the error values, ten percent of the smallest entry of $Q^*_1$ is $ 0.786$. 

\begin{table}[H]
    \centering
    \begin{tabular}{c|cccccc}
        & Mean  Error & Median Error & Error SD& Mean S.C. & Median S.C.& S.C. SD \\
         \hline
    LR-EVI &$0.893 $& $0.895$& $0.0636$ &$7.24  \times 10^6$ & $7.06 \times 10^6$ & $6.31 \times 10^5$\\
    LR-EVI $+$ \texttt{SI}& $0.840 $&$0.839$& $ 0.0194$ &$2.87  \times 10^7$&$2.87 \times 10^7$& $4.11 \times 10^6$\\
    EVI & $0.863$ & $0.864$  & $ 0.00857$ & $2.58 \times 10^8$ &$2.59\times 10^8$&$1.38 \times 10^7$
    \end{tabular}
    \caption{The mean, median, and standard deviation of the $\|\bar{Q}_1\|_\infty$ error and sample complexity of LR-EVI, LR-EVI + \texttt{SI}, and EVI.}
    \label{tbl:scompODV}
\end{table}
\begin{table}[H]
    \centering
    \begin{tabular}{c|cccccc}
        & Mean  Error & Median Error& Error SD & Mean S.C. & Median S.C. & S.C. SD \\
         \hline
    LR-MCPI &$0.747 $ & $0.721$ & $0.0860$&$4.00 \times 10^7$ & $4.22 \times 10^7$& $3.21\times 10^7$\\
    LR-MCPI $+$ \texttt{SI}& $0.697 $&$0.716$& $0.0556$& $2.51 \times 10^8$&$2.50 \times 10^8$ & $1.67 \times 10^7$\\
    MCPI & $0.771$ & $0.773$  & $0.0272$& $1.45 \times 10^9$ &$1.44\times 10^9$& $1.46 \times 10^8$
    \end{tabular}
    \caption{The mean, median, and standard deviation of the $\|\bar{Q}_1\|_\infty$ error and sample complexity of  LR-MCPI, LR-MCPI + \texttt{SI}, and MCPI.}
    \label{tbl:scompODP}
\end{table}
}

\tsdelete{
The mean, median, and standard deviation of the $\ell_\infty$-error of $\bar{Q}_1$ and sample complexity for 10 trials for all our algorithms with $p_S = p_A = 0.025$ for LR-EVI and LR-MCPI and $p = 0.2$ for LR-EVI $+$ \texttt{Soft Impute} are presented in Table~\ref{tbl:scompOD} see Appendix \ref{app:chooseP} for discussion on how $p_S, p_A, $ and $p$ were chosen. \cycomment{Yikes, this sentence is hard to read. Also we never introduced $p$.} 
Figure \ref{fig:errvsSc} shows the error of LR-EVI, LR-MCPI, LR-EVI $+$ \texttt{Soft Impute}, and $Q$-value iteration vs. their respective sample complexity at time step $h=1$.
To give context to our results, $.1 \times \min_{(s, a) \in S \times A} Q^*_1(s, a) = 0.786$. It is clear from Table~\ref{tbl:scompOD} and Figure~\ref{fig:errvsSc} that the low-rank algorithms require significantly fewer samples to achieve the same error as $Q$-value iteration.
\begin{table}[H]
    \centering
    \begin{tabular}{c|cccc}
        & Mean  Error & Median Error & Mean S.C. & Median S.C. \\
         \hline
    LR-MCPI &$0.747 \pm 0.0860$ & $0.721$ & $(4.00 \pm 0.321) \times 10^7$ & $4.22 \times 10^7$\\
    LR-EVI &$0.895 \pm 0.0636$ & $0.939$& $(6.38 \pm 1.35) \times 10^6$ & $6.69 \times 10^6$\\
    LR-EVI $+$ \texttt{SI}& $0.840 \pm 0.0194$&$0.839$& $(2.87 \pm 0.411) \times 10^7$&$2.87 \times 10^7$\\
    $Q$-value Iteration & $0.863 \pm 0.00857$ & $0.864$  & $(2.58 \pm 0.138) \times 10^8$ &$2.59\times 10^8$
    \end{tabular}
    \caption{The mean $\pm$ standard deviation and median $\|\bar{Q}_1\|_\infty$ error and sample complexity of  LR-MCPI, LR-EVI, LR-EVI + \texttt{Soft Impute}, and $Q$-value iteration.}
    \label{tbl:scompOD}
\end{table}

\begin{figure}[H]
    \centering
    \includegraphics[width = .65\textwidth]{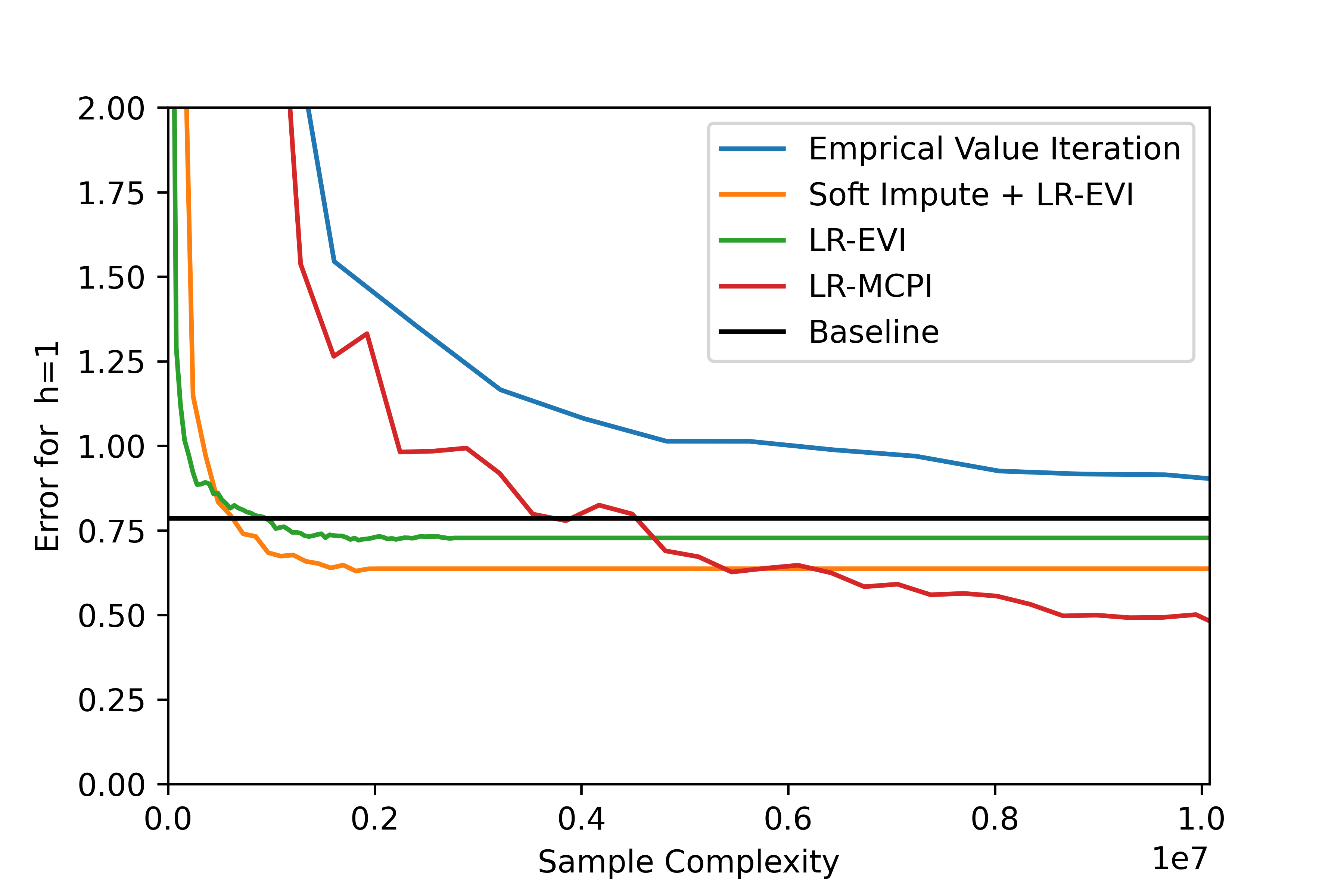}
    \caption{Max entrywise error of $\bar{Q}_1$ vs. number of samples for LR-EVI, LR-MCPI, LR-EVI $+$ \texttt{Soft Impute}, and $Q$-value iteration at $h = 1$. To give context to the error values, Baseline $= .1 \times \min_{(s, a) \in S \times A} Q^*_1(s, a) = 0.786$}
    \label{fig:errvsSc}
\end{figure}}
\section{Conclusion}
In this work, we prove novel sample complexity bounds using matrix estimation methods for MDPs with long time horizons without knowledge of special anchor states and actions, showing that incorporating matrix estimation methods into reinforcement learning algorithms can significantly improve the sample complexity of learning a near-optimal action-value function. Furthermore, we empirically verify the improved efficiency of incorporating the matrix estimation methods. We also provide a lower bound that highlights exploiting low rank structure in RL is significantly more challenging than the static matrix estimation counterpart without dynamics. While we show a gain from $|S||A|$ to $|S|+|A|$, the sample complexity may not be optimal with respect to $d$ and $H$, which may be an interesting topic for future study. For example one could consider how to incorporate advanced techniques in existing tabular reinforcement learning literature that decrease the dependence on the time horizon into our low rank framework. While our results show the value of exploiting low-rank structure in reinforcement learning, the algorithms heavily rely on a generative model assumption, which may not always be realistic. Extensions to online reinforcement learning is an interesting and potentially impactful future direction.





\bibliographystyle{plain}
\bibliography{sigmetrics/references}
\clearpage
\appendix

\clearpage
\section{Extensions}\label{sec:ext}
We extend our results to the continuous MDP setting and infinite-horizon discounted MDP setting. We also discuss the use of alternative matrix estimation subroutines.
Each of these extensions are fairly minor technically, but we include them to illustrate the wider implications of the low rank framework.

\subsection{Continuous State and Action Spaces}\label{sec:cont}

Our results in Theorems \ref{thm:qnolr} and \ref{thm:tklr} can be extended to the continuous MDP setting where $S$ and $A$ are both continuous spaces. In particular, the action-value function obtained when running LR-EVI and LR-MCPI on a discretized version of the continuous MDP can be used to construct an $\eps$-optimal action-value function for the continuous MDP, similar to the reduction used in \cite{serl}. We assume the same regularity conditions on the continuous MDP as used in \cite{serl}.
\begin{assumption}[MDP Regularity for Continuous MDPs \citep{serl}]\label{asm:regC}
Assume the MDP satisfies
\begin{itemize}
    \item (Compact Domain): $S = [0, 1]^{n}$, $A = [0, 1]^{n}$,
    \item (Lipschitz): $Q^*_h$ is $L$-Lipschitz with respect to the one-product metric:
    \[
    |Q^*_h(s, a) - Q^*_h(s', a')| \leq L(\|s-s'\|_2 + \|a-a'\|_2) ~~~~\forall~~ h \in [h].
    \]
\end{itemize}
\end{assumption}

We follow the same steps as in \cite{serl} to discretize the state and action spaces into $\beta$-nets ($S^\beta$ and $A^\beta$, respectively), i.e. $S^\beta$ is a set such that for all $s \in S$, there exists an $s' \in S^\beta$ where $|s' - s|_2\leq \beta$. We next define the discretized MDP to be $M^\beta = (S^\beta, A^\beta, P^\beta, r, H)$ where $P^\beta_h$ is defined as follows:
\[
P^\beta_h(s' | s,a) = \int_{\{s''\in S: |s'' - s'|_2 \leq \beta\}} P_h(s''| s,a) \mathrm{d}s''.
\]
After discretizing the state and action spaces, LR-MCPI or LR-EVI is run on the discretized MDP. Our approach differs from the one from \cite{serl} because we only discretize the continuous sets once and then run the tabular algorithm while their algorithm changes the discretization error $\beta$ at each iteration. To run LR-MCPI or LR-EVI on the discretized MDP, one needs to be able to sample transitions/rollouts from $P^\beta_h$ instead of $P_h$. See Appendix \ref{app:cont} for details on how we exploit the generative model to obtain transitions/rollouts on $M^\beta$. The following lemma shows how 
the optimal $Q$ function on $M^\beta$ can be used to approximate $Q^*$ of the original MDP with small enough $\beta$. 
\begin{lemma}\label{lem:bnet}
Let MDP $M^\beta = (S^\beta, A^\beta, P^\beta, R, H)$ be the discretized approximation to MDP $M = (S, A, P, R, H)$ where $S^\beta$ and $A^\beta$ are $\beta$-nets of $S$ and $A$, respectively. Let $Q^*$ and $Q^\beta$ be the optimal $Q$ functions of $M$ and $M^\beta$, respectively.  For any $s \in S, a \in A$ and $s' \in S^\beta, a' \in A^\beta$ such that $\|s - s'\|_2 \leq \beta, \|a - a'\|_2 \leq \beta$  and for all $h \in [H]$, 
\[
|Q^*_h(s,a) - Q^\beta_h(s', a')| \leq 2L(H-h+1)\beta, \quad |V^*_h(s,a) - V^\beta_h(s', a')| \leq 2L(H-h+1)\beta. 
\]
\end{lemma}


If the transition kernels and reward functions of $M^{\beta}$ are low rank, satisfying Assumption \ref{asm:lrtk}, then LR-EVI finds an $\eps$-optimal $Q_h$ function with an efficient number of samples. If $M^{\beta}$ only satisfies Assumption \ref{asm:lrqe}, then LR-MCPI finds an $\eps$-optimal $Q_h$ function with an efficient number of samples. For sake of brevity, we present only the sample complexity bound of LR-EVI under Assumption \ref{asm:lrtk}. See Appendix \ref{app:cont} for the analogous result with LR-MCPI.

\begin{theorem}\label{thm:conTKLR}
Let $Q_h^\beta = [r_h + [P_h \hat{V}_{h+1}]]_{(s,a) \in S^\beta \times A^\beta}$ where $\hat{V}_{h+1}$ is the value function learned when running LR-EVI on the discretized MDP $M^\beta$. Let Assumption \ref{asm:lrtk} hold on $M^\beta$, and $S^\#_h, A^\#_h$ be $(k, \alpha)$-anchor states and actions for $Q_h^\beta$ for all $h \in [H]$. Then, the learned $\bar{Q}_h$ from LR-EVI can be used to construct an $\eps$-optimal $Q$ function with probability at least $1-\delta$ when  $\beta = \eps/4LH$ and $N_{H-t} = \Tilde{O}\left((t+1)^2k^2\alpha^2 H^2/\eps^2\right), N_{H-t}^{\#} = \Tilde{O}\left((t+1)^2k^4\alpha^4 H^2/\eps^2\right)$
 for all $t\in \{0, \ldots, H-1\}$ with a sample complexity of $\Tilde{O}\left(k^3\alpha^2H^{n+5} / \eps^{n+2}Vol(B)\right )$, where $B$ is the unit norm ball in $\R^n$.
\end{theorem}


Theorem \ref{thm:conTKLR} shows that if the low-rank and matrix estimation assumptions hold on the discretized MDP, then one can use the learned $Q_h$ estimate from LR-EVI to construct an $\eps$-optimal estimate of $Q$ function. Both algorithms are sample efficient (with respect to the dimension of the state and action spaces) as the bounds have a $1/\eps^{n+2}$ dependence instead of $1/\eps^{2n+2}$, which is minimax optimal without the low-rank assumption. Furthermore, if $Q^\beta_{h}$ is $\mu$-incoherent with condition number $\kappa$, one can use the result of Lemma \ref{lem:anchor_sampling} to find $\Tilde{O}(\mu d, \kappa)$-anchor states and actions without a priori/domain knowledge. Using the finite-horizon version of Corollary 2 from \cite{policyErrorAmp}, we can construct an $O(\eps H)$-optimal policy by defining a policy greedily with respect to $\bar{Q}$. 

The proof of Theorem \ref{thm:conTKLR} follows from combining Theorem \ref{thm:tklr} with a covering number lemma to upper bound the size of the $\beta$-nets. $\beta$ is chosen carefully to account for the error amplification with respect to $H$ from Lemma \ref{lem:bnet} while ensuring that the algorithms use an efficient number of samples.

\subsection{Infinite-Horizon Discounted MDPs}\label{sec:IH}
We consider the standard setup for infinite-horizon tabular MDPs, $(S,A,P,R,\gamma)$, where $S$ and $A$ denote the finite state and action spaces. $R: S \times A \to \Delta([0,1])$ denotes the reward distribution, and use $r_h(s,a) = \E_{r\sim R(s,a)}[r]$ to denote the expected reward. $P$ denotes the transition kernel, and $0 < \gamma < 1 $ denotes the discount factor. The value and action-value function of following the policy $\pi$ are defined as:
\[
V^\pi(s) \coloneqq \E\left[\left. \sum_{t = 0}^\infty \gamma^t R_t \right|  s_0 = s \right], \quad 
Q^\pi(s, a) \coloneqq \E\left[\left. \sum_{t = 0}^\infty \gamma^t R_t \right |  s_0 = s, a_0 = a \right],
\]
for $R_t \sim R(s_t, a_t), a_t \sim \pi(s_t),$ and $s_t \sim P(\cdot|s_{t-1}, a_{t-1})$.
We define the optimal value function as $V^*(s) = \sup_{\pi} V^\pi(s)$ for all $s\in S$ and the optimal action-value function as $Q^*(s,a) = r(s,a) + \gamma \E_{s' \sim P(\cdot|s,a)}[V^*(s')]$. Since the reward function is bounded, for any policy $\pi$, $ Q^\pi(s,a), V^\pi(s) \leq \frac{1}{1-\gamma}$ for all $(s,a) \in S\times A$. To use matrix estimation methods, we require the transition kernel to have low Tucker rank and the reward function to have shared latent factors, which is our strongest low-rank assumption (Assumption \ref{asm:lrtk}). 
\begin{assumption}[Low-rank Transition Kernels and Reward Functions (Infinite-horizon)] \label{asm:lrtkI}
The expected reward function has rank $d$, and the transition kernel $P$ has Tucker rank $(|S|,|S|,d)$ or $(|S|, d, |A|)$, with shared latent factors. For the Tucker rank $(|S|, |S|, d)$ case, this means that there exists a  $|S| \times |S| \times d$ tensor $U$, an $|A| \times d$ matrix $V$, and an $|S| \times d$ matrix $W$ such that
\[
P(s'|s,a) = \textstyle\sum_{i=1}^d U(s', s, i) V(a, i) \quad\text{ and }\quad r(s,a) = \textstyle\sum_{i=1}^d W(s, i)V(a, i).
\]
For the Tucker rank $(|S|, d, |A|)$ case, this means that there exists a  $|S| \times |A| \times d$ tensor $V$, an $|S| \times d$ matrix $U$, and an $|A| \times d$ matrix $W$ such that
\[
P(s'|s,a) = \textstyle\sum_{i=1}^d U(s, i)V(s', a, i)\quad\text{ and }\quad r(s,a) = \textstyle\sum_{i=1}^d U(s, i)W(a, i).
\]
\end{assumption}
\noindent Similar to the finite-horizon setting, this assumption implies that $r + \gamma[P\bar{V}]$ has low rank for any value function estimate.
 
\begin{proposition} \label{prop:lrEstI} For any MDP that satisfies Assumption \ref{asm:lrtkI}, for any estimate of the value function, the rank of $r + \gamma [P\bar{V}]$ is upperbounded by $d$.
\end{proposition}

\noindent The algorithm we consider that admits an efficient sample complexity is LR-EVI with the same matrix estimation method adapted for the infinite-horizon discounted setting, i.e., including the discount factor in the estimates and running Step 1, Step 2, and Step 3 for $T$ iterations instead of recursing backwards through the horizon. We overload notation and let $\hat{Q}_i$ refer to the $Q$ function estimated in the $i$-th iteration of the algorithm. The correctness result and sample complexity bound in this setting is as follows.

\begin{theorem}[Correctness and Sample Complexity of LR-EVI under Assumption \ref{asm:lrtkI}] \label{thm:tklrI}
Assume that for any $\eps$-optimal value function $\bar{V}$, the matrix corresponding to $Q'_t = [r + [P \hat{V}_t]]$ has rank $d$ (a consequence of Assumption \ref{asm:lrtkI}) for all $t \in [T]$, and  $S^\#_t, A^\#_t$ are $(k, \alpha)$-anchor states and actions for $\hat{Q}'_t = [r + [P \hat{V}_t]]$, where $\hat{V}_t$ is the learned value function from LR-EVI at iteration $t$ for all $t \in [T]$. Let $N_{t} = \Tilde{O}\left( \alpha^2 k^2 / \eps^2 (1-\gamma)^4 \right)$ and $N_t^{\#} = O(\alpha^2k^2 N_t)$. Then, LR-EVI returns an $\eps$-optimal $Q$ function with probability at least $1-\delta$ with a sample complexity of  $
\Tilde{O}\left(\frac{(|S|+ |A|)\alpha^2k^3}{\eps^2(1-\gamma)^4} + \frac{\alpha^4k^6}{\eps^2(1-\gamma)^4} \right)$.
\end{theorem}
Theorem \ref{thm:tklr} states that if the transition kernel has low Tucker rank, one can learn an $\eps$-optimal $Q$ function with sample complexity that scales with the sum of the sizes of the state and action space instead of the product. Furthermore, if $Q'_t$ is $\mu$-incoherent, then one can use Lemma \ref{lem:anchor_sampling} to find $\Tilde{O}(\mu d, \kappa)$-anchor states and actions without domain knowledge, where $\kappa$ is the condition number of $Q'$. To prove the correctness result in Theorem \ref{thm:tklr}, we show that at each iteration the error of the $Q$ function decreases with the following lemma.
\begin{lemma}\label{lem:tklrI}
Let $ S^\#_t, A^\#_t \text{ and } N_t$ be as defined as in Theorem \ref{thm:tklrI}, and let the estimate of the value function at step $t$ satisfy $|\bar{V}_t- V^*|_\infty \leq B_t$. Ater one iteration of the algorithm, the resulting estimates of the value function and action-value function satisfy
\[
|\bar{Q}_{t+1}-Q^*|_\infty\leq \frac{(1+\gamma)B_t}{2} , \quad |\bar{V}_{t+1}-V^*|_\infty \leq\frac{(1+\gamma)B_t}{2}
\]
with probability at least $1- \frac{\delta}{T}$ for each $t \in [T]$.
\end{lemma}
Running the algorithm for a logarithmic number of times returns an $\eps$-optimal $Q$ function, which gives the sample complexity shown in Theorem \ref{thm:tklr}.

\subsection{Matrix Completion via Nuclear Norm Regularization}\label{sec:CME}

While all of the above results are stated for the variants of LR-MCPI and LR-EVI that use the matrix estimation algorithm as stated in Section \ref{sec:MES}, our results are not limited only to this specific choice of the matrix estimation algorithm. As briefly mentioned in Section \ref{sec:MES}, the analysis relies on using entrywise error bounds for the outputs of the matrix estimation algorithm. While the algorithm stated in Section \ref{sec:MES} lends itself to explicit entrywise error bounds given its explicit form, it requires the non-standard sampling pattern associated to a set of anchor states and actions. 

We show next that similar results can be derived for a different variation of LR-MCPI or LR-EVI that performs matrix estimation by solving the convex relaxation of the low-rank matrix completion problem.
We utilize Theorem 1 from \cite{chen2019noisy} to obtain entry-wise bounds on the matrix estimator. Chen et al. \cite{chen2019noisy} state that it is straightforward to extend their results to the rectangular matrix setting, but for ease of notation, they only consider square matrices. To directly apply the theorem, we let $|S| = |A| = n$ but note that it is easy to extend our results to $|S| \neq |A|$. 
 Their analysis assumes data is gathered via a Bernoulli sampling model; i.e. each state-action pair is added to $\Omega_h$ with probability $p$, i.e., $\Omega_h = \{(s,a) | X_{(s,a)} = 1\}$ where $X_{(s,a)}\sim \text{Bernoulli}(p)$ for $(s,a) \in S \times A$.

The matrix estimator is the minimizer of a least-squares loss function with a nuclear norm regularizer, which is the convex relaxation of the low rank constraint. For the observed matrix $M_{ij} = M^*_{ij} + E_{ij}$, with $M^*$ being the matrix we wish to recover and error matrix $E$, the formulation is 
\begin{equation}\label{eq:obj}
  \min_{Z \in \R^{n\times n}} g(Z) \triangleq \tfrac{1}{2}\textstyle\sum_{(i, j) \in \Omega_h}( Z_{ij} - M_{ij})^2 + \lambda \|Z\|_*,
\end{equation}
with $\Omega_h$ constructed via Bernoulli sampling as mentioned above \cite{chen2019noisy}. 

We next present the primary result that is needed from \cite{chen2019noisy} for the readers' convenience. Assume that  $\Omega$ is constructed with the Bernoulli sampling model and the error matrix $E = [E_{i,j}]$ is composed of i.i.d. zero-mean sub-Gaussian random variables with norm at most $\eta$. 
\begin{theorem}[Theorem 1 in \cite{chen2019noisy}]\label{thm:conRel}
Let $M^*$ have rank$-d$ and be $\mu$-incoherent with condition number $\kappa$, where $d, \kappa \in O(1)$. Let $\lambda = C_\lambda n\sigma \sqrt{np}$ in Equation \ref{eq:obj} for a large enough positive constant $C_\lambda$. Assume that $n^2p \geq C\mu^2n \log^3 n$ and $\sigma \leq c \sqrt{\frac{np}{\mu^3 \log n}}\|M^*\|_\infty$ for some sufficiently large constant $C > 0 $ and small constant $c > 0$. Then with probability $1 - O(n^{-3})$, any minimizer $Z_{cvx}$ of Equation \ref{eq:obj} satisfies
\[
\|Z_{cvx} - M^*\|_\infty \leq C_{cvx}\frac{\sigma}{\sigma_d(M^*)} \sqrt{\frac{\mu n \log n }{ p}} \|M^*\|_\infty
\]
for some constant $C_{cvx} > 0$.
\end{theorem}

Applying Theorem \ref{thm:conRel} into our analyses for LR-MCPI and LR-EVI gives us the necessary error bounds to prove the desired linear $|S|+|A|$ sample complexities for LR-MCPI and LR-EVI with $\Omega_h$ generated according to the Bernoulli sampling model and 
\begin{align*}
 \texttt{ME}\left(\{\hat{Q}_h(s,a)\}_{(s,a) \in \Omega_h}\right) \xleftarrow[]{} \text{CvxSolver} \bigg (\min_{Q \in \R^{|S|\times |A|} } g(Q) &\triangleq \tfrac{1}{2}\textstyle\sum_{(s,a) \in \Omega_h}(Q(s,a) - (\hat{Q}_h(s,a)) )^2 + \lambda \|Q\|_* \bigg).
\end{align*}

We state only the result for LR-EVI under Assumption \ref{asm:lrtk} (low-rank reward function and low Tucker rank transition kernel). The modifications to the theorems and proofs to show the analogous result for LR-MCPI under Assumption \ref{asm:lrqe} are essentially the same. 

\begin{theorem}\label{thm:MECR}
Let $p_h = \mu^3 d^2 \kappa^2 H^4 C_{cvx}^2\log(n) / \eps^2n$. Assume that for any $\eps$-optimal value function $\hat{V}_{h+1}$,  $Q'_h = [r_h + [P_h\hat{V}_{h+1}]]$ has rank $d$ (Assumption \ref{asm:lrtk}), is $\mu$-incoherent, and has condition number $\kappa$ for all $h \in [H]$. Then, the learned policy from the algorithm specified above is $\eps$-optimal with probability at least $1- O\left(Hn^{-3}+\exp \left(-\mu^3d^2\kappa^2H^4 n\log(n) /\eps^2 \right) \right)$. Furthermore, the number of samples used is upper bounded by $\Tilde{O}\left(\mu^3H^5n / \eps^2\right)$ with the same probability. 
\end{theorem}
The proof of Theorem \ref{thm:MECR} follows the same argument as the proof of Theorem \ref{thm:tklr} but uses Theorem \ref{thm:conRel} to control the error amplification from the matrix estimation method. Similar to our main results, using this matrix estimation method as a subroutine reduces the sample complexity's dependence on $|S|$ and $|A|$ from $|S||A|$ to $|S|+|A|$. This theorem provides a potential explanation for the successful experimental results in \cite{Yang2020Harnessing}, and answers an open question posed in \cite{serl}; it guarantees that using existing matrix estimation methods based on convex problems as a subroutine in traditional value iteration has significantly better sample complexity compared to vanilla value iteration when finding $\eps$-optimal action-value functions.

\section{Example Illustrating Assumption \ref{asm:lrqe} (Low Rank Q functions for Near Optimal Policies)}\label{app:qnolrex}

Assumption \ref{asm:lrqe} states that the $\eps$-optimal policies $\pi$ have associated Q functions that are low rank. At first glance, it might be unclear if this Assumption can be satisfied without requiring the stronger conditions in Assumption \ref{asm:lrtk} of low rank rewards and low Tucker rank transition kernels.
In this section, we present an MDP $(S, A, P, R, H)$ in the reward maximization setting with all $\eps$-optimal policies $\pi$ ($\eps$-optimal $\pi$ for all $s \in S$ and $h \in [H]$) having low-rank $Q^\pi$ without the transition kernel having low Tucker rank. Specifically, we upperbound the rank of $Q^\pi$ with a function of $\eps$ and the size of the state/action space.
where $\Pi_\eps$ is the policy class containing all $\eps$-optimal deterministic policies. With the following example, we show that there exists an MDP with a non-trivial relationship between $d_\eps$ and $\eps, |S|, $ and $|A|$.
We now present the $H$-step MDP that exhibits this property. Let $S = A =  0 \cup [m],$ and the reward function be $r_h(s, a) = 0$ for all $(s, a, h) \in S \times A \times [H-1]$ and $r_H(s, a) = 1 - \left(\frac{sa}{(m+1)^2} \right)^{1/2}$ for all $(s, a) \in S \times A$.

For all $h \in [H-1]$, the transition kernel is 
\[
P_h(0|s, a) = \begin{cases}
0, & \text{if } s = a, \\
1, & \mathrm{otherwise,} 
\end{cases}
\qquad
P_h(s'|s, a) = \begin{cases}
1, & \text{if }s'= s = a, \\
0, & \mathrm{otherwise,} 
\end{cases}
\]
for $s' \in \{1, \ldots, m\}$. We note that 
\[
P_h(0|\cdot, \cdot) =  \begin{bmatrix}
\mathbf{0} & 1 & 1 & \ldots & 1 \\
 1 &\mathbf{0} &  1 & \ldots & 1 \\
\vdots & \vdots & \ddots & \vdots & \vdots \\
 1  &  1 & \ldots &\mathbf{0} & 1 \\
 1 &  1 & \ldots & 1 &\mathbf{0}  \\
\end{bmatrix},
\]
and for $s' \in [m]$, $P_h(s'|s,a) = E^{s'}$ is an all-zero matrix with the $s'$-th diagonal entry equal to one, so the transition kernels do not have low Tucker rank. Next, we prove the main result of this section, which upper bounds the rank of the $Q_h$ functions of $\eps$-optimal policies.

We remark that at time step $2$, selecting action $0$ is always optimal, regardless of the state.

\begin{lemma}\label{lem:rank}
Let $\pi$ be an $\eps$-optimal policy, that is $V^*_h(s) - V^\pi_h(s) \leq \eps$ for all $(s, h) \in S \times [H]$. Then, $rank(Q^\pi_H) =  2,$ and $rank(Q^\pi_h) \leq 1 + \lfloor \eps^2(m+1)^2 \rfloor$ for all $h \in [H-1]$.
\end{lemma}
\begin{proof}[Proof of Lemma \ref{lem:rank}]
Let $\pi$ be an $\eps$-optimal policy. We first show that $Q^\pi_H = r_2$ is a matrix with rank $2$. By construction, the first two rows of $Q_H^\pi$ are:
\begin{align*}
(Q^\pi_2)_1 = \left[1, 1, 1,  \ldots, 1 \right], \quad
(Q^\pi_2)_2 = \left[1, 1 - \left(\frac{1}{(m+1)^2}\right)^{1/2}, 1 - \left(\frac{2}{(m+1)^2}\right)^{1/2}, \ldots, 1- \left(\frac{m+1}{(m+1)^2}\right)^{1/2} \right],
\end{align*}
and for any $i \in \{3, \ldots m+1\}$, the $i$-th row of $Q^\pi_2$ is 
\[
(Q^\pi_2)_i = \left[1, 1 - \left(\frac{i}{(m+1)^2}\right)^{1/2}, 1 - \left(\frac{2i}{(m+1)^2}\right)^{1/2}, \ldots, 1- \left(\frac{(m+1)i}{(m+1)^2}\right)^{1/2} \right].
\]
Hence, $(Q^\pi_2)_i = (1 - i^{1/2})(Q^\pi_2)_1 + i^{1/2}(Q^\pi_2)_2$, and $rank(Q^\pi_2) = 2$.

Let $h \in [H]$, to bound the rank of $Q_h^\pi$, we first note that for all $s\in S$, $|V_H^\pi(s) - 1| = \left(\frac{s\pi(s)}{(m+1)^2} \right)^{1/2} \leq \eps$ since $\pi$ is $\eps$-optimal. It follows that $\pi(s) < \frac{\eps^{2}(m+1)^2}{s}$, and if $\frac{\eps^{2}(m+1)^2}{s} < 1$, $\pi(s)$ must equal $0$, which is the optimal action. Hence, there are at most $s \leq \eps^{2}(m+1)^2 $ number of states in which $\pi$ can deviate from the optimal policy. The value function of an $\eps$-optimal policy is 
\[
V^\pi_H(s) = \begin{cases}
1, & \text{if }s = 0, \\
1, & \text{if }s > \eps^2(m+1)^2, \\
1 - \left( \frac{s\pi(s)}{(m+1)^2}\right)^{1/2}, & \mathrm{otherwise.}
\end{cases}
\]
Since $\pi$ is $\eps$-optimal, we have $1 - V^\pi_{h+1}(s) \leq \eps$ for all $s \in S$. Due to the construction of the dynamics, if one starts at state $s$ at time step $h$, one will be at either state $s$ (choosing action $\pi(s) = s$ at each time step) or state $0$ (taking any other sequence of actions).
Thus, $V^\pi_{h+1}(s) = V^{\pi}_H(s)$ or $V^\pi_{h+1}(s) = 1$ depending on the sequence of action. It follows that $V^\pi_{h+1}(s) \geq V^{\pi}_H(s)$. It follows that 
if $s = 0$ or $s > \eps^2(m+1)^2$, $V^\pi_{h+1}(s) = 1$. Otherwise, $V^\pi_{h+1}(s) \geq  \left(\frac{s\pi(s)}{(m+1)^2} \right)^{1/2} $ for $s \leq \lfloor \eps^2(m+1)^2 \rfloor$. 

We next compute the $Q$ function at the time step $h$ to show that  we can upperbound the rank of $Q_h$ by the number of states that $\pi_h$ deviates from the optimal policy. Specifically, for each $s \leq \lfloor \eps^2(m+1)^2 \rfloor$, let $\pi_h(s) = s$ for $h \in [H]$. It follows that
\begin{align*}
    Q_h^\pi(s,a) &= r_h(s,a) + \sum_{s' = 0}^m P_h(s'|s,a)V^\pi_{h+1}(s') \\
    &= 0 + P_h(0|s,a)V^\pi_{h+1}(0) + \sum_{s' = \lfloor \eps^2(m+1)^2 \rfloor + 1}^m P_h(s'|s,a)V^\pi_{h+1}(s')+ \sum_{s' =1}^{\lfloor \eps^2(m+1)^2 \rfloor} P_h(s'|s,a)V^\pi_{h+1}(s')\\
    &= P_h(0|s,a)  + \sum_{s' = \lfloor \eps^2(m+1)^2 \rfloor + 1}^m P_h(s'|s,a) + \sum_{s' =1}^{\lfloor \eps^2(m+1)^2 \rfloor} P_h(s'|s,a)\left( 1 - \left( \frac{s\pi(s)}{(m+1)^2}\right)^{1/2}\right).
\end{align*}
In matrix form, it follows that 
\begin{align*}
    Q_h^\pi &\geq \begin{bmatrix}
\mathbf{0} & 1 & 1 & \ldots & 1 \\
 1 &\mathbf{0} &  1 & \ldots & 1 \\
\vdots & \vdots & \ddots & \vdots & \vdots \\
 1  &  1 & \ldots &\mathbf{0} & 1 \\
 1 &  1 & \ldots & 1 &\mathbf{0}  \\
\end{bmatrix} + \sum_{s' = \lfloor \eps^2(m+1)^2 \rfloor + 1}^m E^{s'} + \sum_{s' =1}^{\lfloor \eps^2(m+1)^2 \rfloor} E^{s'}\left( 1 - \left( \frac{s\pi(s)}{(m+1)^2}\right)^{1/2}\right)\\
    &= J_{m\times m } -  \sum_{s' =1}^{\lfloor \eps^2(m+1)^2 \rfloor} E^{s'}\left( \frac{s\pi(s)}{(m+1)^2}\right)^{1/2}.
\end{align*}
Thus, at most $\lfloor \eps^2(m+1)^2 \rfloor$-rows of $Q_h^\pi$ are different from the all-ones row. It follows that $rank(Q_h^\pi) \leq 1 + \lfloor \eps^2(m+1)^2 \rfloor$, and each state $s \leq \lfloor \eps^2(m+1)^2 \rfloor$ that $\pi_h$ performs optimally at tightens the above upperbound on the rank by one. Since the bound holds for arbitrary $h \in [H-1]$, it follows that it holds for all $h \in [H-1]$.
\end{proof}

\section{Experimental Details for Oil Discovery Problem}\label{app:expDetails}
In this section we discuss the rank of the cost function, the rank of the $Q^*$ function, and details of how we tuned the hyperparameters of our algorithms for the experiments presented in Section \ref{sec:experiments}.

\subsection{Rank of $c(s, a)$}\label{app:cost}
Recall that $c(s, a) = 0.01 \times \text{round} \left( \frac{|s-a|}{100} \right)$. Figure \ref{fig:cost} displays a heat map and the singular values of the cost function.
\begin{figure}[H]
    \centering
    \subfloat{\includegraphics[width=.45\textwidth]{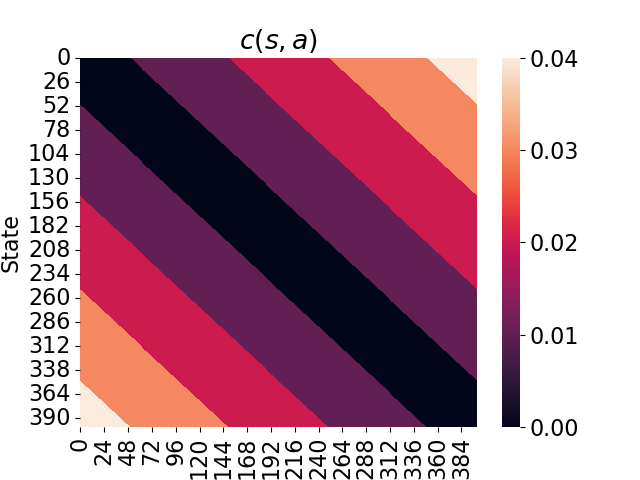}}%
    \qquad
    \subfloat{\includegraphics[width=.45\textwidth]{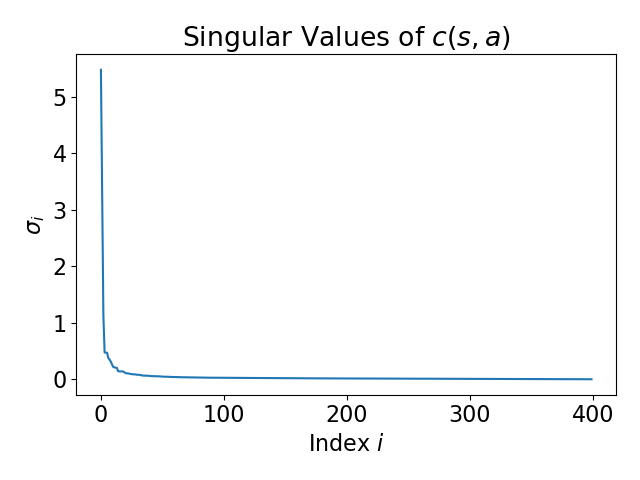} }%
    \caption{Heat map (left) and singular values (right) of $c(s, a)$.}
    \label{fig:cost}
\end{figure}
Even though each entry of $c(s, a)$ can only be one of five values, the rank of $c(s, a)$ is 400 as all of the singular values are greater than zero. However, Figure \ref{fig:cost} shows that the magnitude of the singular values decrease quickly. Furthermore, the stable rank of $c(s, a)$ is small $\|c(s, a)\|_F^2/\|c(s, a)\|^2_* = 1.46$. It follows that $c(s, a)$ is ``approximately'' low-rank. 

\subsection{Rank of $Q^*$} \label{app:qStar}
Figure \ref{fig:qStar} displays a heat map of $Q^*_1$ and a plot of $Q^*_1$ singular values from largest to smallest. 
\begin{figure}[H]
    \centering
    \subfloat{\includegraphics[width=.45\textwidth]{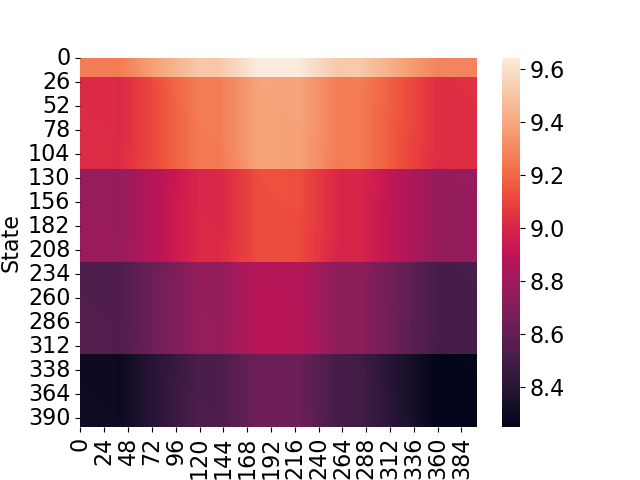}}%
    \qquad
    \subfloat{\includegraphics[width=.45\textwidth]{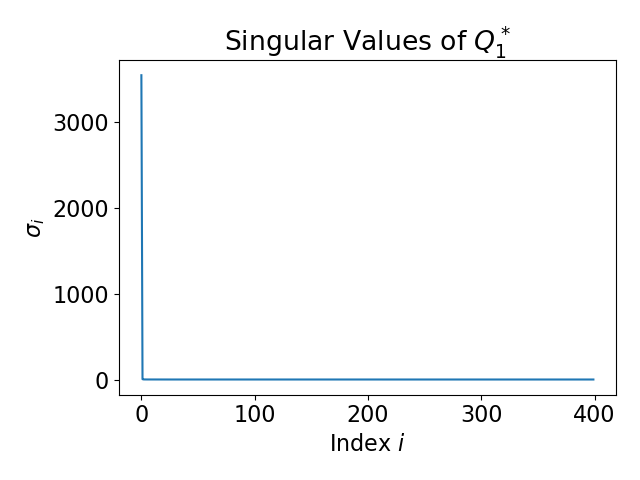} }%
    \caption{Heat map (left) and singular values (right) of $Q^*_1$.}
    \label{fig:qStar}
\end{figure}
While all of the singular values of $Q^*_1$ are greater than zero ($rank(Q^*_1) = 400$), the magnitude of the first singular value is significantly larger than all the other singular values; $\sigma_1/ \sum_{i = 1}^{400}\sigma_i = 0.995$. Table \ref{tbl:stableRank} displays the rank and stable rank of $Q^*_h$ for $h \in [H]$.
\begin{table}[H]
    \centering
    \begin{tabular}{c|cccccccccc}
         & $h = 1$ & $h = 2$& $h = 3$& $h = 4$& $h = 5$& $h = 6$& $h = 7$& $h = 8$& $h = 9$& $h = 10$ \\
         \hline
    Rank & $400$& $400$& $400$& $400$& $400$& $400$& $400$& $400$& $400$& $400$   \\
    Stable Rank& $1.000$ & $1.000$& $1.000$& $1.000$& $1.000$& $1.000$& $1.000$& $1.000$& $1.004$& $1.000$
    \end{tabular}
    \caption{Rank and stable rank of $Q^*_h$ for $h \in [H]$.   }
    \label{tbl:stableRank}
\end{table}
From the results in Table \ref{tbl:stableRank}, it's clear that despite $rank(Q^*_h) = 400$ for all $h \in [H]$, $Q^*_h$ is approximately low rank for all $h \in [H]$ because all the stable ranks are close to one.

\subsection{Hyperparameter Tuning}\label{app:hyp}
In this section, we discuss the hyperparameters of our algorithms and our methodology on how to choose their value. 

\medskip \noindent
\textbf{Allocation Scheme $N_{s, a, h}:$} In the proof of our theorems for LR-EVI and LR-MCPI, the sample allocations $N_{s,a,h}$ are chosen to ensure that at each time step $h$, the algorithm takes enough samples so that that $\|\bar{Q}_h - Q^*_h\|_\infty \leq \eps(H-h+1)/H$.  In practice, the algorithm does not have access to the optimal $Q$ function and cannot use this condition as a criteria to choose $N_{s,a,h}$, so we instead empirically test a few different allocation schemes and choose the best. We choose $N_{s,a,h}$ to be uniform for all $s,a$, not distinguishing between $(s,a)$ in the anchor submatrix or not.

To determine how to allocate samples across the ten time steps for each algorithm, we run a set of experiments benchmarking the algorithm's performance on a set of different allocation schedules. For an allocation scheme $\tau = \{\tau^i\}_{h \in H}$, where $\tau^i \geq 0, \sum_{h \in [H]} \tau^i = 1$, and total sample budget of $\overline{N}$, we will allocate roughly $\tau^i \overline{N}$ samples to the estimation of $\bar{Q}_h$. Essentially $\tau$ specifies the proportion of samples that are allocated to the estimates at each step, where there is some rounding involved as the number of samples must be integral. For some sequence of nonnegative numbers $\{a^i_{1}, a^i_{2}, \dots a^i_{H}\}$, the corresponding allocation scheme $\tau^i$ follows by simply normalizing according to $\tau^i_{h} = a^i_{h}/\sum_{h' \in [H]}a^i_{h'}$.
The five different allocation schemes we consider are, $\tau^i = \{\tau^i_{h}\}_{h \in [H]}$, corresponding to
\begin{equation*}
  \begin{split}
    a^1_{h} &= H - h + 1,\\
    a^2_{h} &= (H - h + 1)^2, \\
    a^3_{h} &= (H - h + 1)^3
  \end{split} \quad
  \begin{split}
    a^4_{h} &= \lfloor(h+1)/2\rfloor, \\
    a^5_{h} &= 1. \\
    &
  \end{split}
\end{equation*}
$\tau^5$ is simply a constant allocation schedule, which evenly allocates samples across the steps. When there is an insufficient sample budget to implement other allocation schemes, we let the Empirical Value Iteration algorithms default to $\tau^5$, evenly allocating one-step samples across steps $h$.

$\tau^1$ is a linearly decreasing allocation schedule. Note that as Monte Carlo Policy Iteration requires samples of length $(H-h+1)$ trajectories to estimate $Q_h$, the allocation that would uniformly allocate trajectories across $h$ for MCPI corresponds to $\tau^1$. When there is an insufficient sample budget to implement other allocation schemes, we let the Monte Carlo Policy Iteration algorithms default to $\tau^1$, evenly allocating trajectories across steps $h$.

$\tau^4$ is also a linearly decreasing allocation schedule, but simply at a slower rate. $\tau^2$ is a quadratically decreasing allocation schedule, which matches the allocation schedule chosen in our Theorem \ref{thm:tklr} for LR-EVI, as indicated by the number of one-step samples $N_h$ scaling as $(H-h+1)^2$ in its dependence on $h$. $\tau^3$ is a cubically decreasing allocation schedule, which matches the allocation schedule chosen in our Theorems \ref{thm:gap} and \ref{thm:qnolr} for LR-MCPI. In particular, the number of trajectories $N_h$ scales as $(H-h+1)^2$ in its dependence on $h$, but the samples used need to be multiplied by the trajectory length $(H-h+1)$, resulting in a cubic relationship.

Table~\ref{tbl:alloc} displays the average entrywise error of $\bar{Q}_1$ of all the algorithms over ten trials for each of these allocation schemes, where we set the sample budget $\overline{N} = 10^8$, $p_S, p_A = 0.025$, and $p_{SI} = 0.2$. ``---'' refers to the algorithm not having enough samples to take even one sample for each state-action pair at one time step according to the specified allocation scheme.

\begin{table}[H]
    \centering
    \begin{tabular}{c|ccccc}
         & $\tau^1$ & $\tau^2$ & $\tau^3$ & $\tau^4$ & $\tau^5$ \\
         \hline
         LR-EVI &$0.0803$
&$0.0736$
&$0.0948$
&$0.0828$
&$0.0825$
\\
         LR-MCPI 
&$0.1742$
&$0.1419$
&$0.1532$
&$0.1795$
&$0.2031$ \\
         LR-EVI $+$ \texttt{SI} &$0.5638$
&$0.4512$
&$1.0065$
&$0.5541$
&$0.5358$\\
         LR-MCPI $+$ \texttt{SI} 
&$0.6126$
&$0.6535$
&$0.6868$
&$0.6285$
&$0.6406$\\
         EVI &$0.2927$
&$1.5558$
& ---
&$0.255$
&$0.2264$
\\
         MCPI &$0.52$
&$0.5437$
& ---
&$0.5498$
&$0.6127$
    \end{tabular}
    \caption{Mean $\ell_\infty$ error of $\bar{Q}_1$ of LR-EVI, LR-MCPI, LR-EVI $+$ \texttt{SI}, LR-MCPI $+$ \texttt{SI}, EVI, and MCPI. }
    \label{tbl:alloc}
\end{table}
\noindent
To calibrate these results, recall that entries in $Q^*_1$ take values from roughly 8.3 to 9.6. 
While the performances are generally pretty similar, from the results in Table~\ref{tbl:alloc} the best allocation schemes for the algorithms are $\tau^2$ for LR-EVI, $\tau^2$ for LR-MCPI, $\tau^2$ for LR-EVI $+$ \texttt{SI}, $\tau^1$ for LR-MCPI $+$ \texttt{SI}, $\tau^5$ for EVI, and $\tau^1$ for MCPI. We use these allocation schemes for the experiments in Section \ref{sec:experiments}.

Note that for our algorithms to run, they require minimally one sample for the value iteration algorithms or one trajectory for the policy iteration algorithms for each $(s,a) \in \Omega_h$. Hence, for smaller values of $\overline{N}$, e.g., $\overline{N} = 10^6$, there may be state-action pairs in $\Omega_h$ that do not receive even one sample/trajectory following the best allocation scheme chosen from the data in Table~\ref{tbl:alloc}. Hence, if that problem occurs, we default to allocation scheme $\tau^5$ for the value iteration algorithms and we default to $\tau^1$ for the policy iteration algorithm. These allocations uniformly spread the samples/trajectories to ensure that the algorithm still produces an estimate when $\overline{N}$ may be small.

Finally we describe the details of the rounding that we implement to ensure that $N_{s,a,h}$ are integral, yet are distributed as close as possible to the desired allocation schedule $\tau = \{\tau^h\}$ with the sample budget of $\overline{N}$. For Empirical Value Iteration algorithms, $N_{s,a,h}$ denotes one-step samples at $(s,a)$ used to construct the estimate for $Q_h$. Thus we compute initial values by $N_{s, a, h} =  \lfloor \tau^h \overline{N}/|\Omega_h| \rfloor$, where the floor function is applied as the number of samples must be integral. Subsequently, we compute the number of excess samples, given by $N_{\Delta} = \overline{N} - \sum_{h \in [H]} \sum_{(s,a) \in \Omega_h} N_{s, a, h}$. Then, recursing forwards through the horizon, we add one sample to each state action pair in $\Omega_h$, i.e., $N_{s, a, h} = N_{s, a, h} + 1$, provided that there is sufficient samples $N_{\Delta} \geq |\Omega_h|$. Then we recompute the number of extra samples, i.e., $N_{\Delta} = N_{\Delta} - |\Omega_h|$ and repeat continuing at $h+1$. With this rounding scheme, the final number of samples used by our algorithm will be within $[\overline{N} - D^2,  \overline{N}]$, where $D^2 = 1.6\times 10^5$. 

For Monte Carlo Policy Iteration algorithms, $N_{s,a,h}$ denotes number of trajectories sampled starting from $(s,a)$, that are used to construct the estimate for $Q_h$. Thus we compute initial values by $N_{s, a, h} =  \lfloor \tau^h \overline{N}/|\Omega_h|(H-h+1) \rfloor$, where the floor function is applied as the number of trajectories must be integral. Subsequently, we compute the number of excess samples, given by $N_{\Delta} = \overline{N} - \sum_{h \in [H]} \sum_{(s,a) \in \Omega_h} N_{s, a, h} (H-h+1)$. Then, recursing forwards through the horizon, we add one trajectory to each state action pair in $\Omega_h$, i.e., $N_{s, a, h} = N_{s, a, h} + 1$, provided that there is sufficient samples $N_{\Delta} \geq |\Omega_h| (H-h+1)$. Then we recompute the number of extra samples, i.e., $N_{\Delta} = N_{\Delta} - |\Omega_h|(H-h+1)$ and repeat continuing at $h+1$. With this rounding scheme, the final number of samples used by our algorithm will be again within $[\overline{N} - D^2,  \overline{N}]$.

\medskip \noindent
{\bf Choosing $p = p_S = p_A$ for LR-EVI and LR-MCPI:} While our theorems use knowledge of the rank and incoherence to choose $p_S$, $p_A$, and $N_h$, one cannot assume knowledge of these quantities in practice. However, many other matrix estimation methods face similar challenges. For example, in \cite{Yang2020Harnessing}, to compute the $Q$ function of the optimal policy, they need to choose $p_{SI}$, which depends on the rank of $Q^*$,  for their algorithm, which combines value iteration and \texttt{Soft-Impute}. They show their algorithm is robust to the choice of $p_{SI}$ by showing their algorithm performed similarly for multiple $p_{SI}$ values. Similarly, we show LR-EVI and LR-MCPI are robust to the choice of $p = p_S = p_A$ (the only parameters in LR-EVI and LR-MPCI that depend on the rank and incoherence for a fixed allocation scheme and $\overline{N}$) in a similar manner. To show that LR-EVI and LR-MCPI are robust to $p = p_S = p_A$, we ran both LR-EVI and LR-MPCI with allocation scheme $\tau^2$ and $\overline{N} = 10^8$ for each $p \in [0.025, 0.05, 0.075, 0.1]$, repeating each experiment 10 times. Since $Q^*_h$ effectively has a rank of one, $p$ should be minimally greater than or equal to $1/400 = 0.0025$; ideally even larger to ensure that with high probability there are a sufficient number of rows and columns sample. We set the smallest value of $p$ to be $0.025$, which results in an expected number of sampled rows/columns of 10 our of 400, which is already a fairly small number. Table \ref{tbl:chooseP} shows the average $\ell_\infty$ error of $\bar{Q}_1$ at time step $h = 1$ for different values of $p$. 
\begin{table}[H] 
    \centering
    \begin{tabular}{c|ccccc}
        & $p = 0.025$ & $p = 0.05$ &$ p = 0.075 $& $p =  0.1$  \\
         \hline
    LR-EVI &$0.077$
&$0.084$
&$0.09$
&$0.129$ \\
    LR-MCPI 
&$0.152$
&$0.179$
&$0.196$
&$0.216$\\
    \end{tabular}
    \caption{The mean $\|\bar{Q}_1\|_\infty$ error of  LR-MCPI and LR-EVI for different values of $p$.}
    \label{tbl:chooseP}
\end{table}
\noindent

To calibrate these results, recall that entries in $Q^*_1$ take values from roughly 8.3 to 9.6. The results show that for the different values of $p$, LR-EVI performs well and the errors are on the same order. Furthermore, the errors are less for smaller values of $p$. The same results hold for LR-MCPI for the different values of $p$. As a result, for the experiments in Section \ref{sec:experiments}, we set $p = p_S = p_A = 0.025$. As our table suggests, the algorithm has decent performance for different values of $p$, so empirically one could choose $p$ based on given computational and memory constraints. The tradeoff is that small values of $p$ could reduce computation and memory usage, but it does assume the MDP satisfies the desired low rank conditions. By choosing $p$ to be as larger, one may increase some robustness to the low rank conditions, as the guarantees would be able to tolerate MDPs with larger ranks. 

\medskip \noindent
\textbf{Choosing $p_{SI}$:}  In contrast to LR-EVI and LR-MPCI, for larger values of $\overline{N}$, $p_{SI}$ needs to be increased as the gain from decreasing the noise is not as beneficial as observing more samples. For LR-EVI $+$ \texttt{SI} and LR-MCPI $+$ \texttt{SI}, we determine the best value of $p_{SI}$ for the four different values of $\overline{N} \in [10^6, 10^7, 10^8, 10^9]$ used in our experiments in Section \ref{sec:experiments}. We test eight values of $p_{SI} \in [0.2, 0.3, \ldots, 0.9]$ for the different $\overline{N}$.

For LR-EVI  $+$ \texttt{SI} to run, it minimally requires one sample for each $(s, a) \in \Omega_h$, which would mean at least $p_{SI}*1.6*10^6$ total samples in expectation. Thus for the lowest sample budget of $\overline{N}=10^6$, we set $p_{SI} = 0.2$ to ensure that LR-EVI  $+$ \texttt{SI}  has sufficient samples to run successfully for all ten trials.

Similarly, LR-MCPI  $+$ \texttt{SI} requires at least one trajectory for each $(s,a) \in \Omega_h$ to run, which would mean a total of $p_{SI}*8.8*10^6$ one-step samples in expectation. Thus for the lowest sample budget of $\overline{N}=10^6$, we set $p_{SI} = 0.075$ to ensure that LR-MCPI  $+$ \texttt{SI}  has sufficient samples to run successfully for all ten trials.

For the larger values of $\overline{N} \in [10^7, 10^8,  10^9]$, we test eight values of $p_{SI} \in [0.2, 0.3, \ldots, 0.9]$. Figure \ref{fig:pSI} shows the average $\ell_\infty$ error of $\bar{Q}_1$ for LR-EVI $+$ \texttt{SI} and LR-MCPI $+$ \texttt{SI} for the different values of $p_{SI}$ and $\overline{N}$ for LR-EVI  $+$ \texttt{SI} using allocation scheme $\tau^2$ and LR-MCPI  $+$ \texttt{SI} using allocation scheme $\tau^1$, where each experiment is repeated ten times.
\begin{figure}[H]
    \centering
    \subfloat[LR-EVI $+$ \texttt{SI}]{\includegraphics[width=.45\textwidth]{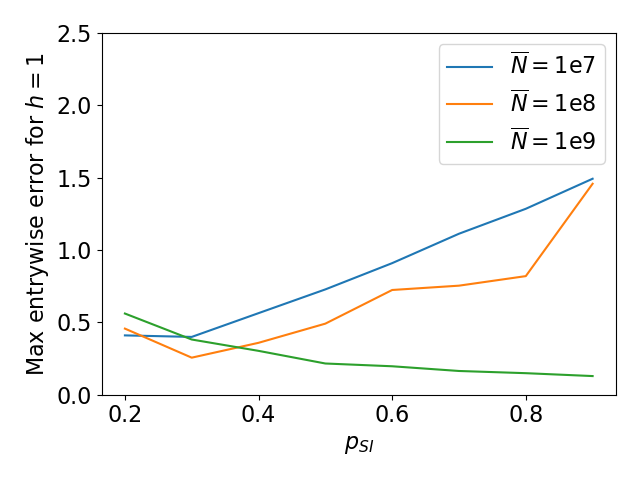} }%
    \qquad
    \subfloat[LR-MCPI $+$ \texttt{SI}]{\includegraphics[width=.45\textwidth]{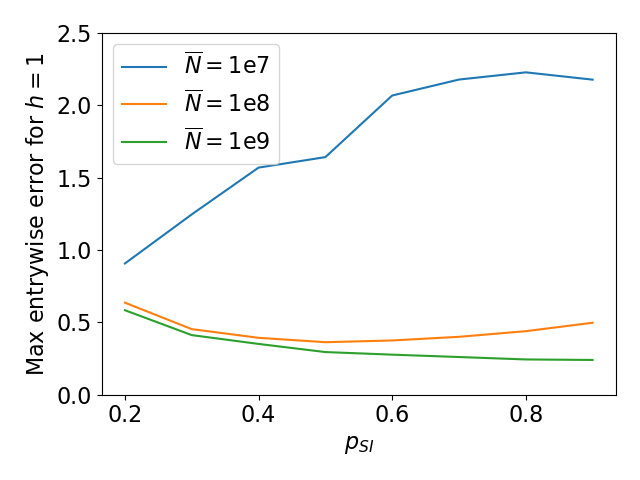} }%
    \caption{Max entrywise error of $\bar{Q}_1$ vs. $p_{SI}$ for four different values of $\overline{N}$ for LR-EVI $+$ \texttt{SI} and LR-MCPI $+$ \texttt{SI}.}
    \label{fig:pSI}
\end{figure}

Figure \ref{fig:pSI} shows that when the sample budget is smaller, i.e. $\overline{N} = 10^7$, smaller values of $p$ perform better; this is expected as there is insufficient sample budget so that increasing $p$ means the number of samples or trajectories allocated to each $(s,a) \in \Omega_h$ will be small, resulting in large noise. For large sample budget, i.e. $\overline{N} = 10^9$, the performance is not very sensitive to the choice of $p_{SI}$, though the larger values of $p_{SI}$ do perform better. This is also expected as there is sufficient samples to both sample more entries while still having $N_{s,a,h}$ large enough that the noise is well controlled. For $\overline{N} = 10^8$ the performance with respect to $p_{SI}$ is quite different in these two plots, and it may be due in part to the different allocation schemes. Allocation scheme $\tau^2$ significantly skews the proportion of samples to the earlier time steps compared to $\tau^1$. Hence for $\overline{N} = 10^8$, with scheme $\tau^2$, increasing $p_{SI}$ results in the error LR-EVI $+$ \texttt{SI} growing perhaps due to high noise in the later time steps. In contrast, with scheme $\tau^1$, for $\overline{N} = 10^8$, the error of LR-MCPI $+$ \texttt{SI} does not increase in $p_{SI}$.

Table \ref{tbl:pSI} displays the value of $p_{SI}$ we use for our experiments in Section \ref{sec:experiments}. As discussed before, the value of $p_{SI}$ is chosen for $\overline{N} = 10^6$ simply to ensure that the observation set is small enough such that the algorithms can produce some estimate for the given sample budget. For $\overline{N} \in [10^7, 10^8, 10^9]$, $p_{SI}$ is chosen according to the value that minimized the error in the results displayed in Figure \ref{fig:pSI}. 
\begin{table}[H]
    \centering
    \begin{tabular}{c|cccc}
         &$\overline{N} = 10^6$  & $\overline{N} = 10^7$  &$\overline{N} = 10^8$ &$\overline{N} = 10^9$ \\
         \hline
     LR-EVI $+$ \texttt{SI}  &0.2  & 0.3 & 0.3 & 0.9\\
     LR-MCPI $+$ \texttt{SI}  & 0.075 & 0.2 & 0.5 & 0.9
    \end{tabular}
    \caption{Values of $p_{SI}$ for each $\overline{N}$ in the experiments in Section \ref{sec:experiments}.}
    \label{tbl:pSI}
\end{table}

\tsdelete{
\subsection{Convergence Schedule Results}\label{app:convSchedule}
Running LR-EVI and LR-MCPI with the first convergence schedule $\tau^h = 0.2 - 0.02(h-1)$ resulted in  $0.885$ and $0.775$ error, respectively. To achieve similar levels of error, we run LR-EVI, LR-EVI $+$ \texttt{SI}, EVI, LR-MCPI $+$ \texttt{Soft-Impute}, and MCPI with convergence schedules $\tau^h = 0.05$ ($c_2$) and $\tau^h = 0.05 - .01( \lfloor(h - 1)/2\rfloor )$ ($c_3$). Table \ref{tbl:convSched} shows the error and sample complexity of the experiments. 
\begin{table}[H]
    \centering
    \begin{tabular}{c|cccccc}
         & $c_1$ Error & $c_1$ S.C. & $c_2$ Error & $c_2$ S.C. & $c_3$ Error & $c_3$ S.C.  \\
        \hline
         LR-EVI + \texttt{SI}& $1.24$ &$ 2.10 \times 10^7$ &  $0.860$&$ 2.83 \times 10^7$ & $0.654$&$ 5.93 \times 10^7$ \\
         LR-MCPI + \texttt{SI}& $0.901$&$ 1.57\times 10^8$ &  $0.775$&$ 3.65 \times 10^8$& $0.804$&$ 3.96\times 10^8$\\
         EVI & $1.53$&$ 1.17 \times 10^8$& $1.07$&$ 1.53 \times 10^8$ & $0.811$&$ 3.28 \times 10^8$ \\
         MCPI & $1.36$&$ 6.57 \times 10^8$& $0.980$&$ 9.68 \times 10^8$ & $0.767 $&$ 1.68 \times 10^9$
    \end{tabular}
    \caption{Error and sample complexity of running LR-EVI $+$ \texttt{SI}, EVI, LR-MCPI $+$ \texttt{SI}, and MCPI with convergence schedules $c_1, c_2,$ and $c_3$.}
    \label{tbl:convSched}
\end{table}
The results in Table \ref{tbl:convSched} show that to achieve similar levels of error, we run LR-EVI with $c_1$, LR-EVI $+$ \texttt{SI} with $c_2$, EVI with $c_3$, LR-MCPI with $c_1$, LR-MCPI $+$ \texttt{Soft-Impute} with $c_2$, and MCPI with $c_3$. }

\section{Additional Experiments for Double Integrator Problem}\label{app:DI}
We empirically evaluate the benefit of including a low-rank subroutine in tabular RL algorithms on the discretized finite-horizon version of the Double Integrator problem, a stochastic control problem seen in \cite{Yang2020Harnessing, serl}. 
\\\\
\noindent \textbf{Experimental Setup:} We formulate the Double Integrator problem as finite-horizon tabular MDP with state space $S = \{(x, \dot{x})\}$ for $x \in \{-2, -1.9,   \ldots, 1.9\}, \dot{x} \in \{-1, -0.9, \ldots, 0.9\}$,  action space $A = \{ -0.5, -0.499,  \ldots, 0.5\}$, and $H = 5$.  With this setup, the size of the state space is $|S| = 40\times 20 = 800$, and the size of the action space is $|A| = 1000$. The learner's goal is to control a unit brick on a frictionless surface and guide it to the origin, state $(0, 0)$. $x$ refers to the brick's position, and $\dot{x}$ denotes the brick's velocity. At each step, the learner is given a noisy reward that penalizes them for the brick's current position, 
\[
r_h((x, \dot{x}), a) = -\frac{x^2 + \dot{x}^2}{2} + \mathcal{N}(0, 1),
\]
for all $h \in [H]$, and $\mathcal{N}(0, 1)$ is a standard normal random variable. The learner chooses an action $a$ to change the velocity of the brick. The dynamics of the system for a given state-action pair $((x, \dot{x}), a)$ for all $h \in [H]$ are 
\[
x' \coloneqq \min(\max (x + \dot{x}, -2), 1.9), \quad \dot{x}' \coloneqq \min( \max(\lfloor \dot{x} + a \rfloor, -1), 0.9),
\]
where $\lfloor \dot{x} \rfloor$ rounds $\dot{x}$ down to the nearest tenth. Since the reward function does not depend on the action, the rank of the reward function is one. Due to the deterministic dynamics, for a given next state $(x', \dot{x}')$, the current state $(x, \dot{x})$ must minimally satisfy $x' = x + \dot{x}$. Thus, there are at most twenty $(x, \dot{x})$ pairs that satisfy $x' = x + \dot{x}$ (the velocity can only take on twenty different values), so there are at most twenty non-zero entries in $P((x', \dot{x}')| \cdot, \cdot)$. Therefore, the Tucker rank of $P((x', \dot{x}')|(x, \dot{x}), a)$ is upperbounded  by $(|S|, 20, |A|)$. Hence, this MDP satisfies Assumption \ref{asm:lrtk}.  Figure \ref{fig:qStarDI} displays a heat map of $Q^*_1$ and a plot of $Q^*_1$ singular values from largest to smallest. 
\begin{figure}[H]
    \centering
    \subfloat{\includegraphics[width=.45\textwidth]{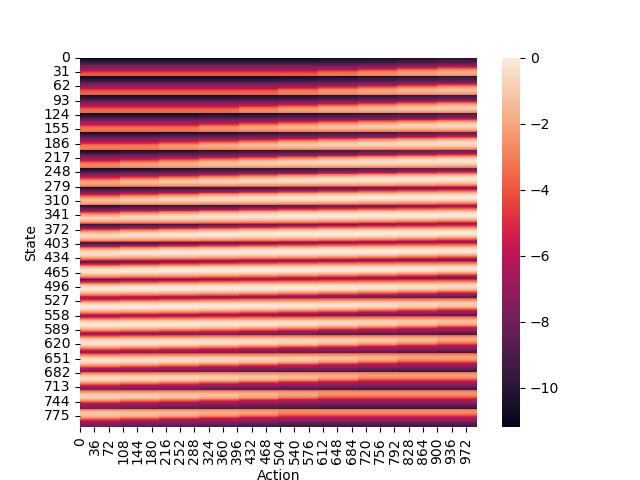}}%
    \qquad
    \subfloat{\includegraphics[width=.45\textwidth]{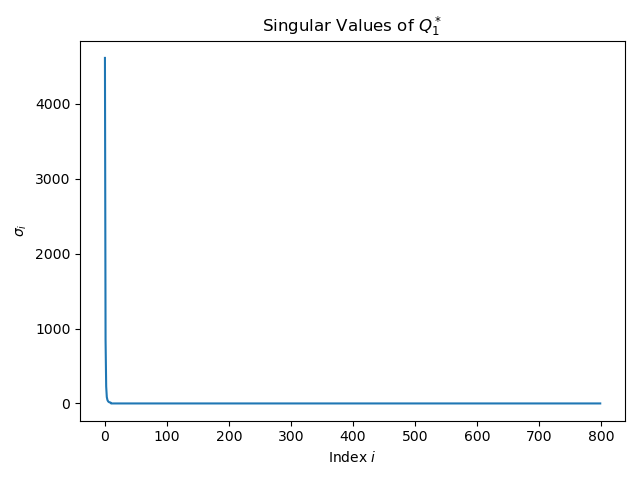} }%
    \caption{Heat map (left) and singular values (right) of $Q^*_1$.}
    \label{fig:qStarDI}
\end{figure}
While the transition kernel has Tucker rank upper bounded by $(|S|, 20, |A|)$, the rank of $Q^*_h$ is ten for $h \in [4]$ while the rank of $Q^*_H = r_H$ is one. Table \ref{tbl:stableRankDI} displays the rank and stable rank of $Q^*_h$ for $h \in [H]$.
\begin{table}[H]
    \centering
    \begin{tabular}{c|ccccc}
         & $h = 1$ & $h = 2$& $h = 3$& $h = 4$& $h = 5$  \\
         \hline
    Rank & $10$& $10$& $10$& $10$& $1$  \\
    Stable Rank& $1.04$ & $1.02$& $1.01$& $1.00$& $1.00$
    \end{tabular}
    \caption{Rank and stable rank of $Q^*_h$ for $h \in [H]$.   }
    \label{tbl:stableRankDI}
\end{table}
From the results in Table \ref{tbl:stableRankDI}, it's clear that the rank of $Q^*_h$ is much smaller than $|S|$ or $|A|$. 
\\\\
\noindent \textbf{Algorithms:} We compare the same algorithms used in the oil discovery experiments. Since $|S| \neq |A|$, we allow for $p_S$ and $p_A$ to be different. Instead of using \texttt{Soft-Impute} from the \texttt{fancyimpute} package, we use the implementation from \cite{tonyduan_matrix_completion_github} because it yielded better results and shorter runtimes. 

\subsection{Hyperparameter Tuning:}
In this section, we discuss the results of tuning the allocation schemes, $p_S, p_A,$ and $ p_{SI}$ for our different algorithms. 
\\\\
\noindent \textbf{Allocation Schemes:} To determine how to divide samples across the five time steps, we test our algorithms on the five different allocation schemes introduced in Appendix \ref{app:hyp}. Recall that the allocation scheme $\tau^i$ is $\tau^i_{h} = a^i_{h}/\sum_{h' \in [H]}a^i_{h'}$ for $[a_1, \ldots, a_H]$. The five different allocation schemes we consider are, $\tau^i = \{\tau^i_{h}\}_{h \in [H]}$, corresponding to
\begin{equation*}
  \begin{split}
    a^1_{h} &= H - h + 1,\\
    a^2_{h} &= (H - h + 1)^2, \\
    a^3_{h} &= (H - h + 1)^3
  \end{split} \quad
  \begin{split}
    a^4_{h} &= \lfloor(h+1)/2\rfloor, \\
    a^5_{h} &= 1. \\
    &
  \end{split}
\end{equation*}
With our implementation, roughly $\tau^i_h \bar{N}$ samples are allocated to estimating $Q^*_h$ ($\bar{N}$ is the sample budget). Table~\ref{tbl:allocDI} displays the mean entrywise error of $\bar{Q}_1$ of all the algorithms over five trials for each allocation schemes. We set the sample budget $\overline{N} = 10^8$, $p_S = 0.1, p_A = 0.08$, and $p_{SI} = 0.4$. 
\begin{table}[H]
    \centering
    \begin{tabular}{c|ccccc}
         & $\tau^1$ & $\tau^2$ & $\tau^3$ & $\tau^4$ & $\tau^5$ \\
         \hline
          LR-EVI  &$0.761 $&$1.20  $&$4.30$&$ 2.15$&$ 0.507$
\\
         LR-MCPI  &$0.550$&$0.416 $&$0.633$&$ 2.07$&$1.16$ \\
                  LR-EVI$+$SI &$0.459$&$ 0.469$&$ 1.343$&$ 0.512 $&$0.456$
\\
         LR-MCPI$+$SI &$0.394 $&$ 0.405 $&$ 0.388$&$  0.436$&$0.441$ \\
         EVI &$1.044$&$ 2.921$&$ 1.343$&$ 3.074 $&$1.469$
\\
         MCPI 
&$0.642 $&$ 0.615 $&$ 1.09$&$  2.225$&$2.526$ 
    \end{tabular}
    \caption{Mean $\ell_\infty$ error of $\bar{Q}_1$ of LR-EVI, LR-MCPI,  LR-EVI$+$SI, LR-MCPI$+$SI, EVI, and MCPI. }
    \label{tbl:allocDI}
\end{table}
While the errors are roughly similar for many of the allocation schemes for each algorithm, we choose the allocation scheme that corresponds to the lowest error. Hence, we use allocation scheme $\tau^5$ for LR-EVI,  $\tau^2$ for LR-MCPI,  $\tau^5$ for LR-EVI$+$SI, $\tau^3$ for LR-MCPI$+$SI, $\tau^1$ for EVI, and $\tau^2$ for MCPI.
\\\\
\noindent\textbf{Choosing $p_S$ and $p_A$:}  Similar to \texttt{Soft-Impute}, varying $p_s$ and $p_a$ as a function of the total number of samples improves the performance of LR-EVI and LR-MPCI. When the sample budget is small ($\bar{N} = 10^7$), one should set $p_s$ and $p_a$ to be smaller, which increases the bias from the matrix estimation method but decreases the noise on the empirical estimates. However, when the sample budget is increased ($\bar{N} = 10^9$), one should increase $p_s$ and $p_a$ to reduce the bias of the matrix estimation method as the estimation error on $\hat{Q}$ is already very small. Thus,  for $\bar{N} \in [10^7, 10^8, 10^9]$, we try the following $(p_s, p_a) \in [(0.025, 0.02), (0.05, 0.04), (0.1, 0.08), (0.2, 0.16)]$ over five trials. Table~\ref{tbl:lrevip} displays the entrywise error of $\bar{Q}_1$ obtained from running LR-EVI with allocation scheme $\tau^5$. Table~\ref{tbl:lrmcpip} displays the entrywise error of $\bar{Q}_1$ obtained from running LR-MCPI with allocation scheme $\tau^2$.
\begin{table}[H]
    \centering
    \begin{tabular}{c|ccc}
         & $\bar{N} = 10^7$ & $\bar{N} = 10^8 $ & $\bar{N} = 10^9 $ \\
         \hline
       $(p_s, p_a) = (0.025, 0.02)$  & $1.72 $&$ 1.09 $&$ 0.650$  \\
       $(p_s, p_a) = (0.05, 0.04)$  & $2.74$&$   0.236$&$  0.190$ \\
       $(p_s, p_a) = (0.1, 0.08)$  & $5.43$& $ 0.488 $& $ 0.0175$ \\
       $(p_s, p_a) = (0.2, 0.16)$  &  $10.9 $& $  1.05 $&$  0.0365$ 
    \end{tabular}
    \caption{Mean $\ell_\infty$ error of $\bar{Q}_1$ of LR-EVI for varying values of $(p_S, p_A)$. }
    \label{tbl:lrevip}
\end{table}
\begin{table}[H]
    \centering
    \begin{tabular}{c|ccc}
         & $\bar{N} = 10^7$ & $\bar{N} = 10^8 $ & $\bar{N} = 10^9 $ \\
         \hline
       $(p_s, p_a) = (0.025, 0.02)$  & $2.26 $&$1.50$&$ 1.34$ \\
       $(p_s, p_a) = (0.05, 0.04)$  & $4.70 $&$0.293 $&$0.143$ \\
       $(p_s, p_a) = (0.1, 0.08)$  & $4.54 $& $ 0.288$& $0.0452$ \\
       $(p_s, p_a) = (0.2, 0.16)$  & $10.84 $&$0.747 $&$0.108$ 
    \end{tabular}
        \caption{Mean $\ell_\infty$ error of $\bar{Q}_1$ of LR-MCPI for varying values of $(p_S, p_A)$. }
    \label{tbl:lrmcpip}
\end{table}
Hence, for LR-EVI, we use $(p_s, p_a) = (0.025, 0.02)$  for  $\bar{N} = 10^7 $, $(p_s, p_a) = (0.05, 0.04)$ for $\bar{N} = 10^8 $, and $(p_s, p_a) = (0.1, 0.08)$ for $\bar{N} = 10^9 $. For LR-MCPI, we use  $(p_s, p_a) = (0.025, 0.02)$ for  $\bar{N} = 10^7 $,  $(p_s, p_a) = (0.1, 0.08)$ for $\bar{N} = 10^8 $, and  $(p_s, p_a) = (0.1, 0.08)$ for $\bar{N} = 10^9 $.
\\\\
\noindent\textbf{Choosing $p_{SI}$:} For LR-EVI $+$ \texttt{SI} and LR-MCPI $+$ \texttt{SI}, we test different values of $p_{SI}$ for  $\overline{N} \in [ 10^7, 10^8, 10^9]$ to determine what to set $p_{SI}$ to in our final experiments in Section \ref{sec:experiments}. We test five values of $p_{SI} \in [0.1, 0.5, \ldots, 0.9]$ for the different $\overline{N}$. Figure \ref{fig:pSIDI} shows the average $\ell_\infty$ error of $\bar{Q}_1$ for LR-EVI $+$ \texttt{SI} and LR-MCPI $+$ \texttt{SI} for the different values of $p_{SI}$ and $\overline{N}$ for LR-EVI  $+$ \texttt{SI} using allocation scheme $\tau^5$ and LR-MCPI  $+$ \texttt{SI} using allocation scheme $\tau^3$, where each experiment is repeated five times.
\begin{figure}[H]
    \centering
    \subfloat{\includegraphics[width=.45\textwidth]{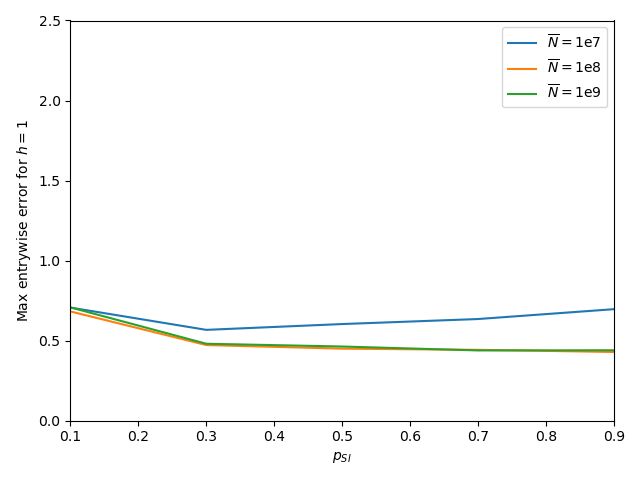}}%
    \qquad
    \subfloat{\includegraphics[width=.45\textwidth]{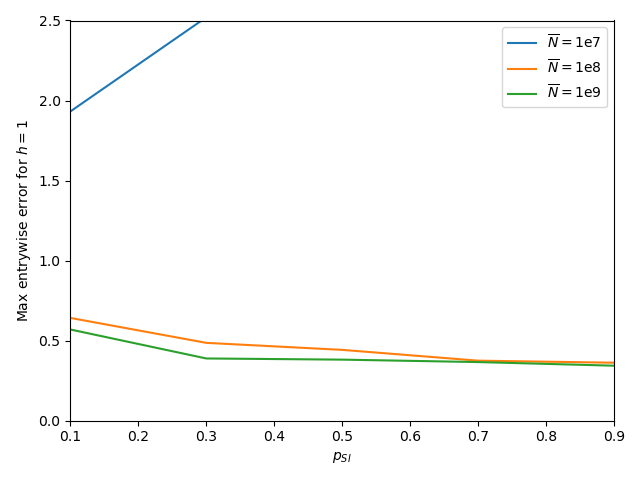} }%
    \caption{Mean $\ell_\infty$ error of $\bar{Q}_1$ of LR-EVI$+$SI(Left) and LR-MCPI$+$SI(Right) for various values of $p_{SI}$ and $\bar{N}$}
    \label{fig:pSIDI}
\end{figure}
From these results, we choose the $p_{SI}$ value that corresponds to the lowest error for our final experiments. Note that for $\bar{N}= 10^7$, the error is strictly increasing as $p_{SI}$ increases for LR-MCPI$+$SI. Hence, for LR-EVI$+$SI, we use $p_{SI} = 0.3$ for $\bar{N} = 10^7 $,  $p_{SI} = 0.9$ for $\bar{N} = 10^8$, and $p_{SI} = 0.9$ for $\bar{N} = 10^9$. For LR-MPCI$+$SI, we use $p_{SI} = 0.1$ for $\bar{N} = 10^7 $,  $p_{SI} = 0.9$ for $\bar{N} = 10^8$, and $p_{SI} = 0.9$ for $\bar{N} = 10^9$.
\\\\
\noindent \textbf{Results:} 
For each value of the sample budget $\bar{N} \in [10^7, 10^8, 10^9]$, we run each of the six algorithms ten times, with the hyperparameters specified above, and compute the average $\ell_\infty$ error of $\bar{Q}_1$.   Figure~\ref{fig:barPlotsDI} displays the mean entrywise error of $\bar{Q}_1$ over ten simulations with the error bars corresponding to the standard deviation. For vanilla MCPI to produce an estimate of $Q^*_1$, it requires at least $\sum_{h=1}^H SAh = 1.2\times 10^7$ one-step samples. Hence, there is no error bar for MCPI with $\bar{N} = 10^7$.  
\begin{figure}[H]
    \centering
    \subfloat{\includegraphics[width=.45\textwidth]{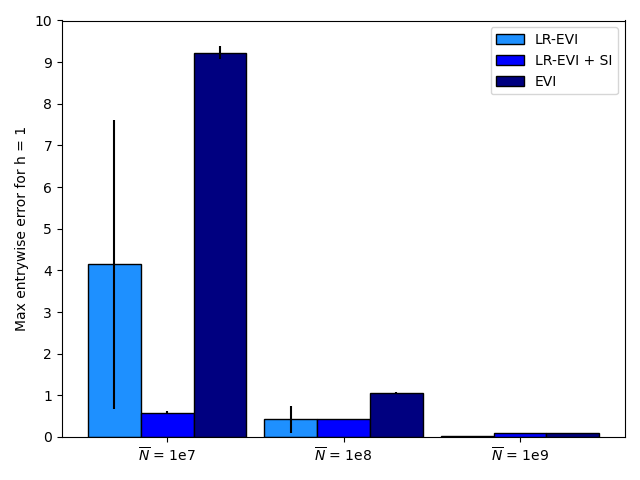}}%
    \qquad
    \subfloat{\includegraphics[width=.45\textwidth]{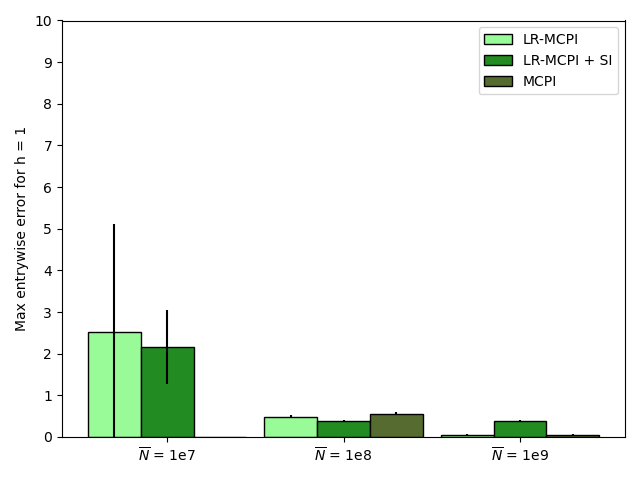} }%
    \caption{Max entrywise error of $\bar{Q}_1$ vs. sample budget for LR-EVI, LR-EVI $+$ \texttt{Soft Impute}, empirical value iteration and LR-MCPI, LR-MCPI $+$ \texttt{Soft-Impute}, Monte Carlo policy iteration at $h = 1$. Note that the optimal $Q_1^*$ function ranges in value from roughly $-11.2$ to $0$, such that $1.12$ error would be roughly 10\% error.}
    \label{fig:barPlotsDI}
\end{figure}
Similarly to the results from the oil discovery problem, Figure~\ref{fig:barPlotsDI} shows that even when there are not enough samples to MCPI, LR-MCPI produces a reasonable estimate. Furthermore, the low-rank algorithms perform better than the tabular versions, EVI and MCPI, when the sample budget is small, i.e., $\bar{N} = 10^7$ or $\bar{N} = 10^8$. When the sample budget is large, i.e., $\bar{N} = 10^9$, the low-rank methods perform similarly to EVI and MCPI, except for LR-MCPI$+$SI. The relatively large error from LR-MCPI$+$SI for $\bar{N} = 10^9$ suggests that the matrix estimation methods are sensitive to the choice of hyperparameters, so in practice, one should carefully tune these given their computational limits, e.g., storage and runtime constraints. In contrast to the oil discovery simulations, the policy iteration algorithms achieve a similar error to the value iteration algorithms with the same sample budget. 
\section{Proof of Lemma \ref{lem:Qlb}}\label{sec:lbProofs}

\begin{proof}[Proof of Lemma \ref{lem:Qlb}]
Consider the MDP defined in Section \ref{sec:lowerbound}. Let $\pi_h(1) = \pi_h(2) = 2$ for all $h \in \{2, \ldots, H-1\}$. We prove that 
\[
Q_{h}^{\pi, \theta}=\left(\begin{array}{cc}
\frac{1}{4} & \frac{1}{2}\\
\frac{1}{2}+2^{H-h}\theta, & 1+2^{H-h+1}\theta
\end{array}\right),\qquad V_{h}^{\pi, \theta}=\left(\begin{array}{c}
\frac{1}{2}\\
1+2^{H-h+1}\theta
\end{array}\right)
\]
with backwards induction on $h$.  Since $V_H^{\pi, \theta} = \left(\begin{array}{c}
\frac{1}{2}\\
1 + 2 \theta
\end{array}\right)$, the base case occurs at step $H-1$. Applying the exact Bellman operator, it follows that 
\[
Q^{\pi, \theta}_{H-1} = \left(\begin{array}{cc}
\frac{1}{4} & \frac{1}{2}\\
\frac{1}{2}+2\theta, & 1+2^{2}\theta
\end{array}\right)
\]
and 
\[
V^{\pi, \theta}_{H-1} = \left(\begin{array}{c}
\frac{1}{2}\\
1+2^{2}\theta
\end{array}\right)
\]
because for both values of $\theta$, $1+4\theta > 0$. Next, assume that the induction hypothesis holds, that is for some $t\in \{2, \ldots, H-1\}$,
\[
Q_{t}^{\pi, \theta}=\left(\begin{array}{cc}
\frac{1}{4} & \frac{1}{2}\\
\frac{1}{2}+2^{H-t}\theta, & 1+2^{H-t+1}\theta
\end{array}\right),\qquad V_{t}^{\pi, \theta}=\left(\begin{array}{c}
\frac{1}{2}\\
1+2^{H-t+1}\theta
\end{array}\right).
\]
Applying the exact Bellman operator, it follows that 
\begin{align*}
Q_{t-1}^{\pi, \theta} &= r_{t-1} + P_{t-1} V_{t}^{\pi, \theta}\\
&= \left(\begin{array}{cc}
-\frac{1}{4} & 0\\
-\frac{1}{2} & 2^{H-t+1}\theta
\end{array}\right) + \left(\begin{array}{cc}
V_{t}^{\pi, \theta}(1) & V_{t}^{\pi, \theta}(1)\\
V_{t}^{\pi, \theta}(2) & (2)
\end{array}\right)\\
&= \left(\begin{array}{cc}
\frac{1}{4} & \frac{1}{2}\\
\frac{1}{2}+2^{H-t+1}\theta, & 1+2^{H-t+2}\theta
\end{array}\right).
\end{align*}
Because $2^H|\theta| = 3/4$, $1 + 2^{H}\theta > 0$, which implies that $Q_{t-1}^{\pi, \theta}(2, 2) > Q_{t-1}^{\pi, \theta}(2, 1)$. Therefore, 
\[
V_{t-1}^{\pi, \theta}=\left(\begin{array}{c}
\frac{1}{2}\\
1+2^{H-t+2}\theta
\end{array}\right)
\]
and the induction hypothesis holds. 
Finally, since one stays in the same state at all steps after $h=1$ by construction, $\pi_h(1) = \pi_h(2) = 2$ for all $h \in \{2, \ldots, H-1\}$ is the unique optimal policy because $2^H|\theta| = 3/4$, which implies that $Q^{\pi, \theta}_{h}(2, 2) = 2Q^{\pi, \theta}_{h}(2, 1) > 0$.
\end{proof}

\section{Proof of Proposition \ref{prop:lrEst}}
\label{sec:proof_lrEst}

Proposition \ref{prop:lrEst} states that if the reward function and transition kernel are low rank, then for any value function estimate $\hat{V}_{h+1}$, $r_h + [P_h\hat{V}_{h+1}]$ has rank upper bounded by $d$.

\begin{proof}[Proof of Proposition \ref{prop:lrEst}]
Let MDP $M = (S,A,P,r, H)$ satisfy Assumption \ref{asm:lrtk} (specifically, $P_h$ has Tucker rank $(|S|, |S|, d)$. Hence, for any value function estimate $\hat{V}_{h+1}$,
\begin{align*}
    r_h(s,a) + P_h \hat{V}_{h+1} &= \sum_{i=1}^d W^{(h)}(s, i)V^{(h)}(a, i) + \sum_{s' \in S}\hat{V}_{h+1}(s')P_h(s'|s,a)\\
    &=\sum_{i=1}^d W^{(h)}(s, i)V^{(h)}(a, i) + \sum_{s' \in S}\hat{V}_{h+1}(s')\sum_{i=1}^d U^{(h)}(s', s, i)V^{(h)}(a, i) \\
    &= \sum_{i=1}^d V^{(h)}(a, i) \left (W^{(h)}(s, i)  + \sum_{s' \in S}\hat{V}_{h+1}(s') U^{(h)}(s', s, i) \right).
\end{align*}
Since $W^{(h)}(:, :)  + \sum_{s' \in S}\hat{V}_{h+1}(s') U^{(h)}(s', :, :)$ is an $|S|\times d$ matrix, $r_h(s,a) + P_h \hat{V}_{h+1}$ has rank at most $d$. The same result holds when $P_h$ has Tucker rank $(|S|, d, |A|)$ from a similar argument.
\end{proof}


\section{Proof of Lemma \ref{lem:anchor_sampling} (Random Sampling of Anchor States and Actions)}

As stated in Lemma \ref{lem:anchor_sampling}, our sampling method is as follows: we sample states and actions using the Bernoulli model. Let 
$\Tilde{U} \in \R^{|S|\times d}, \Tilde{V} \in \R^{|A|\times d}$ such that 
\[
\Tilde{U}_{i} = \begin{cases}
U_{i} \text{ with probability } p_1, \\
0 \text{ otherwise}
\end{cases}, \qquad
\Tilde{V}_{i} = \begin{cases}
V_{i} \text{ with probability } p_2, \\
0 \text{ otherwise}
\end{cases}
\]
Let $\Tilde{Q}_h := \Tilde{U}\Sigma \Tilde{V}^\top \in \R^{|S| \times |A|}$.  The sampled anchor states and actions are the states corresponding to the non-zero rows and columns, respectively. We remark that the Bernoulli model is chosen for convenience and similar results hold if we sample with replacement. To prove Lemma \ref{lem:anchor_sampling}, we present two intermediate lemmas, the first shows $p_1^{-1/2}\Tilde{U}$ and $p_2^{-1/2}\Tilde{V}$ have near orthonormal columns, which implies that  $\Tilde{U}$ and $\Tilde{V}$ have full column rank, with high probability.

\begin{lemma}\label{lem:anchor}
Let $Q_h, U, \Tilde{U}, \Sigma, V,$ and $\Tilde{V}$ be defined as above. Let $Q_h$ be $\mu$-incoherent. Then, with probability at least $1 - 4(|S| \wedge |A|)^{-10}$, we have 
\begin{align*}
    \|p_1^{-1}\Tilde{U}^\top\Tilde{U} - I_{d\times d}\|_{op}
 &\leq \sqrt{ \frac{40\mu d\log(|S|)}{p_1 |S|}} +  \frac{40\mu d \log(|S|)}{p_1 |S|}\\
 \|p_2^{-1}\Tilde{V}^\top \Tilde{V} - I_{d\times d}\|_{op},
 &\leq \sqrt{ \frac{40\mu d\log(|A|)}{p_2 |A|}} + \frac{40\mu d \log(|A|)}{p_2 |A|}.
 \end{align*}
\end{lemma}
\begin{proof}[Proof of Lemma \ref{lem:anchor}]
For each $i \in [|S|]$, let $Z^{(i)}\in \R^{|S| \times d}$ be the matrix obtained from $U$ by zeroing out all but the $i$-th row. Let $\delta_1, \ldots, \delta_{|S|}$ be i.i.d. Bernoulli$(p_1)$ random variables. We can express 
\[
U = \sum_{i\in [|S|]}Z^{(i)} \text{  and  } \Tilde{U} = \sum_{i\in [|S|]}\delta_i Z^{(i)} 
\]
Note that 
\begin{align}
\Tilde{U}^\top \Tilde{U} &= \sum_{i \in [|S|]}\sum_{j \in [|S|]} \delta_i \delta_j Z^{(i)\top} Z^{(j)}\\
&=\sum_{i \in [|S|]} \delta_i^2  Z^{(i)\top} Z^{(i)} \label{eqn:bound}
\end{align}
by construction of $Z^{(i)}$ and $Z^{(j)}$. Hence,
\begin{align}
    \E[\Tilde{U}^\top \Tilde{U}] &= p_1\sum_{i \in [|S|]} Z^{(i)^\top} Z^{(i)} \nonumber\\
    &= p_1\sum_{i \in [|S|]}\sum_{j \in [|S|]} Z^{(i)^\top} Z^{(j)} \nonumber\\
    &= p_1 U^\top U \nonumber\\
    &= p_1I_{d\times d} \label{eqn:boundEx}
\end{align}
where the last equality is due to $U$ having orthonormal columns. For each $i \in [|S|]$, we define the following the mean-zero matrices
\[
X^{(i)} := (\delta_i^2 - \E[\delta_i^2]) Z^{(i)\top} Z^{(i)} = (\delta_i - p_1)Z^{(i)\top} Z^{(i)}.
\]
Since $Q^*_h$ is $\mu-$incoherent, 
\[
\|X^{(i)} \|_{op} \leq |\delta_i - p_1| \|Z^{(i)^\top} Z^{(i)}\|_{op} \leq \|Z^{(i)^\top} Z^{(i)}\|_{op} = \|U_{i-}\|_2^2 \leq \frac{\mu rd}{|S|} \quad \text{ surely}.
\]
Furthermore, 
\begin{align*}
    \sum_{i\in [|S|]}\E[X^{(i)^\top} X^{(i)}] =  \sum_{i\in [|S|]}\E[X^{(i)} X^{(i)^\top}]  &= \sum_{i\in [|S|]}\E[(\delta_i - p)^2]Z^{(i)^\top} Z^{(i)}Z^{(i)^\top} Z^{(i)} \\
    &= p_1(1-p_1) \sum_{i\in [|S|]}\|U_{i-}\|_2^2Z^{(i)^\top} Z^{(i)} \\
    &\preceq p_1 \cdot \frac{d \mu}{|S|} \sum_{i\in [|S|]}Z^{(i)^\top} Z^{(i)} \\
    &=  \frac{d \mu p_1}{|S|} U^TU \\
    &= \frac{d \mu p_1}{|S|}  I_{d\times d}.
\end{align*}
Thus, 
\[
\| \sum_{i\in [|S|]}\E[X^{(i)^\top} X^{(i)}] \|_{op}= \|\sum_{i\in [|S|]}\E[X^{(i)} X^{(i)^\top}] \|_{op}\leq \frac{d \mu p_1}{|S|}
\]
From the matrix Bernstein inequality (Theorem \ref{thm:mb}), we have 
\begin{align*}
\P\left( \| \Tilde{U}^\top \Tilde{U} - p_1 I_{d \times d}\|_{op} \geq t        \right) &= \P\left( \left  \| \sum_{i\in [|S|]}\left( (\delta_i^2  - p_1) Z^{(i)^\top} Z^{(i)} \right) \right \|_{op} \geq t        \right)\\
&= \P \left( \left \| \sum_{i\in [|S|]} X^{(i)} \right \|_{op} \geq t \right )\\
&\leq 2|S| \exp\left (-\frac{t^2/2}{\frac{p_1\mu d}{|S|} + \frac{\mu d}{3|S|}t}\right)\\
&\leq 2|S|\exp\left (-\frac{t^2}{\frac{2p_1\mu d}{|S|} + \frac{2\mu d}{|S|}t}\right)
\end{align*}
where the first equality follows from equations \ref{eqn:bound} and \ref{eqn:boundEx}. For $t = \sqrt{  \frac{40p_1 \mu d \log(|S|)}{|S|}} +   \frac{40\mu d \log(|S|)}{|S|}$, we have 
\[
\left \|\Tilde{U}^\top \Tilde{U}  - p_1 I_{d \times d} \right \|_{op} \leq \sqrt{  \frac{40 p_1 \mu d \log(|S|)}{|S|}} + \frac{40\mu d \log(|S|)}{|S|}
\]
with probability at least $1 - 2|S|^{-10}$. Dividing both sides by $p_1$ yields the first inequality in the lemma. The corresponding bound for $\Tilde{V}$ holds from a similar argument. Taking a union bound over the two events proves the lemma. 
\end{proof}

Now, we present our second lemma that shows that the uniformly sample submatrix ($\Tilde{O}(d)$ by $\Tilde{O}(d)$ in expectation) has rank-$d$ with its smallest non-zero singular value bounded away from zero. 
\begin{lemma}\label{lem:anchorSV}
Let $p_1 = \frac{\mu d \log(|S|)}{320|S|}$ and $p_2 = \frac{\mu d \log(|A|)}{320|A|}$. Under the event in Lemma \ref{lem:anchor}, we have 
\[
\sigma_{d}((p_1 \vee p_2)^{-1} \Tilde{Q}) \geq \frac{1}{2}\sigma_d(Q_h).
\]
\end{lemma}
\begin{proof}[Proof Of Lemma \ref{lem:anchorSV}]
Under the assumption that $p_1 = \frac{\mu d \log(|S|)}{320|S|}$ and $p_2 = \frac{\mu d \log(|A|)}{320|A|}$ and the event in Lemma \ref{lem:anchor}, we have $\| p^{-1}_1\Tilde{U}^\top \Tilde{U}  -  I_{d \times d} \|_{op} \leq \frac{1}{2}$. From Weyl's inequality, we have $\sigma_d(p^{-1}_1\Tilde{U}^\top \Tilde{U}) \geq \frac{1}{2}$, which implies $\sigma_d(p^{-1/2}_1\Tilde{U})\geq \frac{1}{\sqrt{2}}$. From a similar argument,   $\sigma_d(p^{-1/2}_1\Tilde{V})\geq \frac{1}{\sqrt{2}}$. Let $p = p_1 \vee p_2$, from the singular value version of the Courant-Fischer minimax theorem (Theorem 7.3.8 \citep{horn_johnson_1985}), we have 
\begin{align*}
\sigma_d(p^{-1} \Tilde{Q})
     &= \max_{S:dim(S) = d} \min_{x \in S, x \neq 0} \frac{\|p^{-1}\Tilde{U}\Sigma \Tilde{V}^\top x\|_2}{\|x\|_2}\\
    &= \max_{S:dim(S) = d} \min_{x \in S, x \neq 0} \frac{\|(p^{-1/2}\Tilde{U})\Sigma (p^{-1/2}\Tilde{V}^\top) x \|_2}{\|\Sigma(p^{-1/2} \Tilde{V}^\top) x\|_2} \frac{\|\Sigma(p^{-1/2} \Tilde{V}^\top) x\|_2}{\|p^{-1/2} \Tilde{V}^\top x\|_2}\frac{\|p^{-1/2} \Tilde{V}^\top x\|_2}{\|x\|_2}\\
    &\geq \max_{S:dim(S) = d} \min_{x \in S, x \neq 0} \frac{\|(p^{-1/2}\Tilde{U})\Sigma (p^{-1/2}\Tilde{V}^\top) x\|_2}{\|(p^{-1/2} \Tilde{U})^\dagger \|_{op}\|(p^{-1/2}\Tilde{U})\Sigma (p^{-1/2}\Tilde{V}^\top) x\|_2} \\
    &\hspace{90pt} \cdot
    \frac{\|\Sigma(p^{-1/2} \Tilde{V}^\top) x\|_2}{\| \Sigma^{-1}\|_{op}\|\Sigma(p^{-1/2} \Tilde{V}^\top) x\|_2}\frac{\|p^{-1/2} \Tilde{V}^\top x\|_2}{\|x\|_2}\\
    &= \sigma_d(p^{-1/2}\Tilde{U})\cdot \sigma_d(\Sigma)\max_{S:dim(S) = d} \min_{x \in S, x \neq 0}\frac{\|p^{-1/2} \Tilde{V}^\top x\|_2}{\|x\|_2} \\
    &=  \sigma_d(p^{-1/2}\Tilde{U})\cdot \sigma_d(\Sigma) \sigma_d(p^{-1/2} \Tilde{V}^\top )\\
    &\geq \sigma_d(p_1^{-1/2}\Tilde{U})\cdot \sigma_d(\Sigma) \sigma_d(p_2^{-1/2} \Tilde{V}^\top )\\
    &\geq \frac{1}{\sqrt 2} \sigma_d(Q_h) \frac{1}{\sqrt 2}\\
    &= \frac{1}{2}\sigma_d(Q_h)
\end{align*}
where the first inequality comes from properties of the operator norm and inverses/pseudo-inverses and the second inequality comes from replacing $p = p_1 \vee p_2$ with either $p_1$ or $p_2$. 
\end{proof}

Using the two above lemmas, we next prove Lemma \ref{lem:anchor_sampling}.
\begin{proof}[Proof of Lemma \ref{lem:anchor_sampling}]
Let $p_1, p_2$ be defined as in Lemma \ref{lem:anchorSV}. From the previous two lemmas, it follows that with probability at least  
$1 - 4(|S| \wedge |A|)^{-10}$, we have $\sigma_{d}((p_1 \vee p_2)^{-1} \Tilde{Q}) \geq \frac{1}{2}\sigma_d(Q_h)$. Next, we upper bound $\alpha = \frac{\|Q_h\|_\infty}{\sigma_d(Q_h(S^\#, A^\#))} $ assuming that $Q_h$ is $\mu$-incoherent with condition number $\kappa$. Let the singular value decomposition of the rank $d$ matrix $Q_h$ be $Q_h = U \Sigma V^\top$.  For $(s,a) \in S \times A$, 
\begin{align*}
    |Q_h(s,a)| &= |U_s \Sigma V_a|\\
    &\leq \|\Sigma\|_{op}|U_s V_a|\\
    &\leq \sigma_1(Q_h)\|U_s\|_2 |V_a\|_2\\
    &\leq \sigma_1(Q_h) \sqrt{\frac{\mu d}{|S|}} \sqrt{\frac{\mu d}{|A|}}\\
    &= \frac{d\sigma_1(Q_h) \mu }{\sqrt{|S||A|}}
\end{align*}
where the third inequality comes from $Q_h$ being $\mu$ incoherent. Hence, 
\begin{align*}
    \frac{\|Q_h\|_\infty}{\sigma_d(Q_h(S^\#, A^\#))} &\leq  \frac{d\sigma_1(Q_h) \mu }{\sigma_d(Q_h(S^\#, A^\#))\sqrt{|S||A|}}  \\
    &\leq \frac{d\sigma_1(Q_h) \mu }{\sigma_d(Q_h(S^\#, A^\#))(|S|\wedge|A|)} \\
    &= \frac{320\sigma_1(Q_h)  }{\sigma_d((p_1 \vee p_2)^{-1}Q_h(S^\#, A^\#))\log(|S|\wedge |A|)}\\
    &= \frac{640 \sigma_1(Q_h)}{\sigma_d(Q_h)\log(|S|\wedge |A|)}\\
    &= \frac{640 \kappa}{\log(|S|\wedge |A|)}
\end{align*}
where the third line comes from the definition of $p_1$ and $p_2$ and the fourth line comes from Lemma \ref{lem:anchor}. Hence, $\alpha \in O(\kappa)$. Next, we upperbound the size of the anchor sets with high probability. 

From the one-sided Bernstein's inequality, Proposition \ref{prop:oBern}, for $C'' = \frac{25600}{3\mu d}$, 
\begin{align*}
    \P \left ( |S^\#| - \E[|S^\#|] \geq C''p_1|S|       \right) &\leq \exp\left(    -\frac{p_1^2 (C'')^2|S|}{2( p_1 + \frac{ p_1 C''}{3})}  \right)\\
    &\leq \exp\left(    -\frac{\mu d C''}{640( 1 + \frac{ 1}{3})} \log(|S|) \right)\\
    &= |S|^{-10}.
\end{align*}
With a similar argument, 
\[
\P \left ( |A^\#| - \E[|A^\#|] \geq C''p_2|A|       \right) \leq |A|^{-10}.
\]
From our definition of $p_1, p_2$, it follows that $\E \left [|S^\#| \right] = O\left(d\mu\log(|S|)\right)$ and  $\E \left [|A^\#| \right] = O\left(d\mu\log(|A|)\right)$. A union bound on the above two events and the one in Lemma \ref{lem:anchor} asserts that 
\[
|S^\#| \in O\left(d\mu\log(|S|)\right), \quad |A^\#| \leq O\left(d\mu\log(|A|)\right), \quad \text{and } \alpha \in O(\kappa)
\]
with probability at least $1 - 6(|S| \wedge |A|)^{-10}$.

\end{proof}

\section{Proof of Lemma \ref{lem:ME_error_amplification} (Entrywise Bounds for Matrix Estimation)}

Lemma \ref{lem:ME_error_amplification} provides bounds for the entrywise error amplification of the matrix estimation method as a function of on $k$ and $\alpha$, assuming that $S^\#$ and $A^\#$ are ($k, \alpha$)-anchor states and actions for matrix $Q_h$.

\begin{proof}[Proof of Lemma \ref{lem:ME_error_amplification}]
Let $S^\#$ and $A^\#$ be ($k, \alpha$)-anchor states and actions for matrix $Q_h$. For all $(s, a) \in \Omega^{\#} = S^\# \times A^\#$, assume that $\hat{Q}_h(s,a)$ satisfies $|\hat{Q}_h(s,a) - Q_h(s,a)| \leq \eta^{\#},$ and for all $(s, a) \in \Omega \setminus \Omega^{\#}$, assume that $\hat{Q}_h(s, a)$ satisfies $|\hat{Q}_h(s,a) - Q_h(s,a)| \leq \eta$. We follow the same argument as the proof of Proposition 13 in \cite{serl}
except we upperbound equations (22) and (23) with $\|Q_h\|_\infty$ instead of $V_{\max}$. Following the steps in \cite{serl}, i.e., using the triangle inequality and from the definition of the operator norm, for all $(s,a) \in S\times A$, since $S^\#$ and $A^\#$ are ($k, \alpha$)-anchor states and actions, 
\begin{align*}
 |\bar{Q}_{h}(s,a) - Q_{h}(s,a)| \leq \sqrt{2}& \left \|[\hat{Q}_h(S^\#, A^\#)]^\dagger \right\|_{op} \left \|\hat{Q}_h(S^\#, a)\hat{Q}_h(s, A^\#) - Q_h(S^\#, a)Q_h(s, A^\#)   \right \|_F\\
 +& \left\|[\hat{Q}_h(S^\#, A^\#)]^\dagger- [Q_h(S^\#, A^\#)]^\dagger \right\|_{op}\left \| Q_h(S^\#, a)Q_h(s, A^\#)   \right \|_F.
\end{align*}
Following the steps in the proof of Proposition 13,  we upperbound the first operator norm term with Weyl's inequality and our assumption on $\eps$ and the second operator norm term with a classic result from perturbing pseudoinverses,  
\begin{align*}
    \left \|[\hat{Q}_h(S^\#, A^\#)]^\dagger \right\|_{op} & \leq \frac{2}{\sigma_d(Q_h(S^\#, A^\#))}\\
    \left\|[\hat{Q}_h(S^\#, A^\#)]^\dagger- [Q_h(S^\#, A^\#)]^\dagger \right\|_{op}&\leq 2(1+\sqrt{5}) \frac{\eta^{\#} k}{\sigma_d(Q_h(S^\#, A^\#))^2}.
\end{align*}
Since for all $s, s' \in S$ and $a, a' A$, 
\begin{align*}
    \left| \hat{Q}_h(s', a)\hat{Q}_h(s, a') - Q_h(s', a)Q_h(s, a') \right| &\leq   \left|(Q_h(s', a)+\eta)(Q_h(s, a')+\eta) - Q_h(s', a)Q_h(s, a') \right|\\
    &\leq  \eta| Q_h(s', a)|+\eta|Q_h(s, a')| +     \eta^2  \\
    &\leq 2\eta \|Q_h\|_\infty + \eta^2, 
\end{align*}
then, $\left \|\hat{Q}_h(S^\#, a)\hat{Q}_h(s, A^\#) - Q_h(S^\#, a)Q_h(s, A^\#)   \right \|_F \leq (2\eta \|Q_h\|_\infty + \eta^2) k$. Because $|Q_h(s',a)Q_h(s, a')|\leq \|Q_h\|_\infty^2$ for all $s, s' \in S$ and $a, a' A$, clearly $\left \| Q_h(S^\#, a)Q_h(s, A^\#)   \right \|_F \leq \|Q_h\|_\infty^2k$. Using these inequalities gives that for all $(s,a) \in S\times A$, 
\begin{equation}\label{eqn:me}
 |\bar{Q}_{h}(s,a) - Q_{h}(s,a)| \leq \left( 6\sqrt{2}  \alpha k \eta  + 2(1+\sqrt{5})\alpha^2k^2 \eta^{\#} \right)  \in O(\alpha k \eta + \alpha^2k^2\eta^{\#})
\end{equation}
since $\eta \leq \|Q_h\|_\infty$. 
\end{proof}

\section{Inductive Arguments for Theorems \ref{thm:gap}, \ref{thm:qnolr}, and \ref{thm:tklr}}
\label{app:thmProofs}
We next present the missing proofs of our sample complexity bounds in Section \ref{sec:results}. Recall that for ease of notation, 
\[
N_{s, a, h} \coloneqq \begin{cases} N_h^{\#} \text{ if } (s, a) \in \Omega_h^{\#} = S^{\#}_h \times A^{\#}_h\\
N_h \text{ otherwise.}
\end{cases}
\]

\begin{proof}[Proof of Theorem \ref{thm:gap}]
Assume that $Q^*_h$ is rank $d$ and has suboptimality gap $\Delta_{\min}$ (Assumptions \ref{asm:lrqt} and \ref{asm:sog}), and $S^\#_h, A^\#_h$ are $(k, \alpha)$-anchor states and actions for $Q^*_h$ for all $h \in [H]$. Let $N_{H-t} = \frac{2(t+1)^2(c')^2 k^2\alpha^2\log(2H|S||A|/\delta)}{\Delta_{\min}^2}, N_{H-t}^{\#} = \alpha^2k^2 N_{H-t}$, where $c'$ satisfies the inequality in Lemma \ref{lem:ME_error_amplification}, for all $h \in [H]$. We prove the correctness of LR-MCPI with high probability with induction on $t$ that the learned policy $\hat{\pi}_{H-t}$ is an optimal policy with probability at least $1 - \delta(t+1)/H$.

The base case occurs at step $t=0$ in which case our estimates, $\hat{Q}_H(s,a) = \frac{1}{N_{(s, a, H)}}\sum_{i=1}^{N_{(s, a, H)}}r_{H}^i(s,a)$ over $\Omega_H$, are only averages of realizations $r_{H}^i \sim R_H(s,a)$. Since $R_H(s,a)$ has bounded support for all $(s,a) \in S \times A$, from Hoeffding's inequality (Theorem \ref{thm:hoef}) with our choice of $N_{(s, a, H)}$, 
\begin{align*}
|\hat{Q}_{H}(s, a) - Q^*_H(s,a)| &\leq \frac{\Delta_{\min}}{2c'k\alpha}   \qquad \qquad \forall (s, a) \in \Omega_H\\
|\hat{Q}_{H}(s, a) - Q^*_H(s,a)| &\leq \frac{\Delta_{\min}}{2c'k^2\alpha^2}   \qquad \qquad \forall (s, a) \in \Omega_H^{\#}
\end{align*}
with probability at least $1 - \delta/H$ because $|\Omega
_h| \leq |S||A|$. Step 2 of LR-MCPI gives 
\[
|\bar{Q}_{H}(s, a) - Q^*_H(s,a)| \leq \frac{\Delta_{\min}}{2}
\]
for all $(s,a) \in S \times A$ from Lemma \ref{lem:ME_error_amplification}. From Step 3 of LR-MCPI, the identified policy is $\hat{\pi}_{H}(s) = \argmax_{a \in A} \bar{Q}_{H}(s, a)$. Assume for sake of contradiction that there exists an $s \in S$ such that $Q^*_{H}(s, \hat{\pi}_{H}(s)) < Q^*_{H}(s, \pi^*_{H}(s)) $. Let $\hat{\pi}_{H}(s) = a, \pi^*_{H}(s) = a^*$. Hence,
\begin{align*}
     Q^*_{H}(s, a^*) - Q^*_{H}(s, a) &=  Q^*_{H}(s, a^*) - \bar{Q}_{H}(s, a)  + \bar{Q}_{H}(s, a) - Q^*_{H}(s, a) \\
    &\leq Q^*_{H}(s, a^*) - \bar{Q}_{H}(s, a^*) + \frac{\Delta_{\min}}{2} \\
    &\leq \Delta_{\min}
\end{align*}
where the first inequality comes from how $\hat{\pi}_{H}(s)$ is defined and the matrix estimation step. Hence, we reach a contradiction since $Q^*_{H}(s, a^*) - Q^*_{H}(s, a)$ is less than the suboptimality gap. Thus, $\hat{\pi}_{H}(s)$ is an optimal policy. Hence, the base case holds. 

Next, let $x \in \{0, \ldots, H-1\}$. Assume that the inductive hypothesis, the policy $\hat{\pi}_{H-x}$ found in Step 4 of LR-MCPI is an optimal policy with probability at least $1- \delta(x+1)/H$, holds.  

Following Step 1 of LR-MCPI, we have $\hat{Q}_{H-x-1}(s,a) = \hat{r}^{\text{cum}}_{H-x-1}(s,a)$, which is an unbiased estimate of $Q^*_{H-x-1}(s,a)$ and also bounded. Hence, from Hoeffding's inequality (Theorem \ref{thm:hoef}), with the choice of $N_{H-x-1} = \frac{2(x+2)^2(c')^2k^2\alpha^2\log(2H|S||A|/\delta)}{\Delta_{\min}^2}, N_{H-x-1}^{\#} = \alpha^2k^2 N_{H-x-1}$, it follows that
\begin{align*}
 |\hat{Q}_{H-x-1}(s, a) - Q^*_{H-x-1}(s, a)| &\leq \frac{\Delta_{\min}}{2c'k\alpha} \qquad \qquad \forall (s, a) \in \Omega_{H-x-1}   \\
 |\hat{Q}_{H-x-1}(s, a) - Q^*_{H-x-1}(s, a)| &\leq \frac{\Delta_{\min}}{2c'k^2\alpha^2} \qquad \qquad \forall (s, a) \in \Omega_{H-x-1}^{\#}
\end{align*}
with probability $1- \frac{\delta}{H|S||A|}$. Step 2 of LR-MCPI gives
\[
|\bar{Q}_{H-x-1} - Q^*_{H-x-1}|_\infty \leq \frac{\Delta_{\min}}{2}
\]
from Lemma \ref{lem:ME_error_amplification}. From a union bound, it follows that $\hat{\pi}_{H-x}$ is an optimal policy and the above event occur with probability at least $1- \delta(x + 2)/H$. From Step 3 of LR-MCPI, the identified policy is $\hat{\pi}_{H-x-1}(s) = \argmax_{a \in A} \bar{Q}_{H-x-1}(s, a)$. Assume for sake of contradiction that there exists an $s \in S$ such that $Q^*_{H-x-1}(s, \hat{\pi}_{H-x-1}(s)) < Q^*_{H-x-1}(s, \pi^*_{H-x-1}(s)) $. Let $\hat{\pi}_{H-x-1}(s) = a, \pi^*_{H-x-1}(s) = a^*$. Hence,
\begin{align*}
     Q^*_{H-x-1}(s, a^*) &- Q^*_{H-x-1}(s, a) \\
     &=  Q^*_{H-x-1}(s, a^*) - \bar{Q}_{H-x-1}(s, a)  + \bar{Q}_{H-x-1}(s, a) - Q^*_{H-x-1}(s, a) \\
    &\leq Q^*_{H-x-1}(s, a^*) - \bar{Q}_{H-x-1}(s, a^*) + \frac{\Delta_{\min}}{2} \\
    &\leq \Delta_{\min}
\end{align*}
where the first inequality comes from how $\hat{\pi}_{H-x-1}(s)$ is defined and the matrix estimation step. Hence, we reach a contradiction since $Q^*_{H-x-1}(s, a^*) - Q^*_{H-x-1}(s, a)$ is less than the suboptimality gap. Thus, $\hat{\pi}_{H-x-1}(s)$ is an optimal policy, and the inductive step holds for $x+1$. It follows from mathematical induction that the learned policy $\hat{\pi}$ is an optimal policy with probability at least $1 - \delta$.

Next, we bound the number of required samples. The number of samples used is
\[
\sum_{t=0}^{H-1}(k(|A|+|S|))N_{H-t}(t+1)  + k^2N_{H-t}^{\#}(t+1)
\]
where the $t+1$ comes from the length of the rollout. With our choice of $N_{H-t}$, it follows that 
\begin{align*}
\sum_{t=0}^{H-1}(k(|A|&+|S|))N_{H-t}(t+1) \\
&= \sum_{t=0}^{H-1}(k(|A|+|S|))\frac{2(t+1)^3(c')^2 k^2\alpha_{H-t}^2\log(2H|S||A|/\delta)}{\Delta_{\min}^2} + \frac{2(t+1)^3(c')^2 k^4\alpha_{H-t}^4\log(2H|S||A|/\delta)}{\Delta_{\min}^2}  \\
&\leq \left(\frac{2c'^2k^3\alpha^2(|S| + |A|)\log(2H|S||A|/\delta)}{\Delta_{\min}^2} + \frac{k^6 \alpha^4 c'^2\log(2H|S||A|/\delta}{\Delta_{\min}^2}\right)\sum_{t=0}^{H-1}(t+1)^3\\
&\in \Tilde{O}\left(\frac{k^3\alpha^2(|S|+|A|)H^4}{\Delta_{\min}^2} + \frac{k^6 \alpha^4 H^4}{\Delta_{\min}^2} \right).
\end{align*}
\end{proof}

\begin{proof}[Proof of Theorem \ref{thm:qnolr}]
This proof follows the same steps as the previous one.
Assume that for all $\eps$-optimal policies $\pi$, $Q^{\pi}_h$ is rank $d$ (Assumption \ref{asm:lrqe}), and  $S^\#_h, A^\#_h$ are $(k, \alpha)$-anchor states and actions for $Q^{\hat{\pi}}_h$, where $\hat{\pi}$ is the learned policy from Low Rank Monte Carlo Policy Iteration for all $h \in [H]$. Let $N_{H-t} = \frac{2(t+1)^2(c')^2 k^2 \alpha^2H^2\log(2H|S||A|/\delta)}{\eps^2}, N_{H-t}^{\#} = \alpha^2k^2 N_{H-t}$, where $c'$ satisfies the inequality in Lemma \ref{lem:ME_error_amplification}, for all $h \in [H]$. We prove the correctness of LR-MCPI with high probability with induction on $t$ that the learned policy $\hat{\pi}_{H-t}$ is $\eps(t+1)/H$-optimal policy with probability at least $1 - \delta(t+1)/H$.

The base case occurs at step $t=0$ in which case our estimates, $\hat{Q}_H(s,a) = \frac{1}{N_{s, a, H}}\sum_{i=1}^{N_{s, a, H}}r_{H}^i(s,a)$ over $\Omega_H$, are only averages of realizations $r_{H}^i \sim R_H(s,a)$. Since $R_H(s,a)$ has bounded support for all $(s,a) \in S \times A$, from Hoeffding's inequality (Theorem \ref{thm:hoef}) with our choice of $N_{s, a, H} $,
\begin{align*}
    |\hat{Q}_{H}(s, a) - Q^*_H(s,a)| &\leq \frac{\eps}{2c'k\alpha H} \qquad \qquad \forall (s, a) \in \Omega_h\\
    |\hat{Q}_{H}(s, a) - Q^*_H(s,a)| &\leq \frac{\eps}{2c'k^2\alpha^2H} \qquad \qquad \forall (s, a) \in \Omega_h^{\#}
\end{align*}
with probability at least $1 - \delta/H$ because $|\Omega
_h| \leq |S||A|$. Step 2 of LR-MCPI gives 
\[
|\bar{Q}_{H}(s, a) - Q^*_H(s,a)| \leq \frac{\eps}{2H}
\]
for all $(s,a) \in S \times A$ from Lemma \ref{lem:ME_error_amplification}. Assume for sake of contradiction that there exists an $s \in S$ such that $Q^*_{H}(s, \hat{\pi}_{H}(s)) < Q^*_{H}(s, \pi^*_{H}(s)) - \eps/H$. Let $\hat{\pi}_{H}(s) = a, \pi^*_{H}(s) = a^*$. Hence,
\begin{align*}
     Q^*_{H}(s, a^*) - Q^*_{H}(s, a) &=  Q^*_{H}(s, a^*) - \bar{Q}_{H}(s, a)  + \bar{Q}_{H}(s, a) - Q^*_{H}(s, a) \\
    &\leq Q^*_{H}(s, a^*) - \bar{Q}_{H}(s, a^*) + \frac{\eps}{2H} \\
    &\leq \frac{\eps}{H}
\end{align*}
where the first inequality comes from how $\hat{\pi}_{H}(s)$ is defined and the matrix estimation step. Hence, we reach a contradiction since $Q^*_{H}(s, a^*) - Q^*_{H}(s, a)$ is less $\eps/H$. Thus, $\bar{Q}_H$ and $\hat{\pi}_H$ are both $\eps/H$-optimal, and the base case holds.

Next, let $x \in \{0, \ldots, H-1\}$. Assume that the inductive hypothesis, the policy $\hat{\pi}_{H-x}$ and action-value function estimate $\bar{Q}_{H-x}$ found in Step 3 of LR-MCPI are $\eps(x+1)/H$-optimal with probability at least $1- \delta(x+1)/H$, holds. 

Following Step 1 from LR-MCPI, we have $\hat{Q}_{H-x-1}(s,a) = \hat{r}^{\text{cum}}_{H-x-1}(s,a)$, which is bounded and an unbiased estimate of $Q^{\hat{\pi}}(s,a)$ for $\hat{\pi} = \{\hat{\pi}_h\}_{H-x \leq h \leq H}$, which is an $\eps$-optimal policy. Hence, from Hoeffding's inequality (Theorem \ref{thm:hoef}), with the choice of $N_{H-x-1} = \frac{2(x+2)^2(c')^2H^2\alpha^2k^2\log(2H|S||A|/\delta)}{\eps^2}, N_{H-x-1}^{\#} = \alpha^2k^2 N_{H-x-1}$, it follows that
\begin{align*}
|\hat{Q}_{H-x-1}(s, a) - Q^{\hat{\pi}}_{H-x-1}(s, a)| &\leq \frac{\eps}{2c'\alpha k H} \qquad \qquad \forall (s, a) \in \Omega_{H-x-1}\\
|\hat{Q}_{H-x-1}(s, a) - Q^{\hat{\pi}}_{H-x-1}(s, a)| &\leq \frac{\eps}{2c'\alpha^2k^2 H} \qquad \qquad \forall (s, a) \in \Omega_{H-x-1}^{\#}
\end{align*}
with probability $1- \frac{\delta}{H|S||A|}$. Step 2 of LR-MCPI gives
\[
\|\bar{Q}_{H-x-1} - Q^{\hat{\pi}}_{H-x-1}\|_\infty \leq \frac{\eps}{2H}
\]
from Lemma \ref{lem:ME_error_amplification}. The union bound asserts that the above error guarantee and $\hat{\pi}_{H-x}$ and $\bar{Q}^{\hat{\pi}}_{H-x}$ are $(x+1)\eps/H$ holds with probability at least $1- \delta(x+2)/H$. From step 3 of LR-MCPI, the identified policy is $\hat{\pi}_{H-x-1}(s) = \argmax_{a \in A} \bar{Q}_{H-x-1}(s, a)$. For all $(s,a) \in S\times A$,
\begin{align*}
|\bar{Q}_{H-x-1}(s,a) - Q^*_{H-x-1}(s,a)| &\leq |\bar{Q}_{H-x-1}(s,a) - Q^{\hat{\pi}}_{H-x-1}(s,a)| \\
&+ |Q^{\hat{\pi}}_{H-x-1}(s,a) - Q^*_{H-x-1}(s,a)|\\ 
&\leq \frac{\eps}{2H} + \left | \E_{s'\sim P_{H-x-1}(\cdot| s,a)} \left [V^{\hat{\pi}}_{H-x}(s') - V^*_{H-x}(s')\right] \right|\\
&\leq \frac{\eps}{2H} + \left |\E_{s'\sim P_{H-x-1}(\cdot| s,a)}\left [(x+1)\eps/H \right] \right|\\
&= \frac{(2x+3)\eps}{2H} .
\end{align*}
Thus, $\bar{Q}_{H-x-1}$ is $\frac{\eps(x+2)}{H}$-optimal. It follows from the construction of $\hat{\pi}_{H-x-1}(s)$ that 
\[
\bar{Q}_{H-x-1}(s, \hat{\pi}_{H-x-1}(s)) \geq \bar{Q}_{H-x-1}(s, a'),
\]
 where $a' = \arg\max_a Q^*_{H-x-1}(s, a)$.
Hence, for all $s\in S$,
\begin{align*}
    |V^*_{H-x-1}(s) - V^{\hat{\pi}}_{H-x-1}(s) | &\leq  | Q^*_{H-x -1}(s, a') - \bar{Q}_{H-x-1}(s, \hat{\pi}_{H-x-1}(s)) | \\
    &+ | \bar{Q}_{H-x-1}(s, \hat{\pi}_{H-x-1}(s)) - Q^{\hat{\pi}}_{H-x-1}(s, \hat{\pi}_{H-x-1}(s))| \\
    &\leq  \frac{(2x+3)\eps}{2H} + \frac{\eps}{2H}\\
    &= \frac{(x+2)\eps}{H}.
\end{align*}
Thus, $\hat{\pi}_{H-x-1}(s)$ and $\bar{Q}_{H-x-1}$ are  $(x+2)\eps/H$-optimal, and the inductive step holds for $x+1$. It follows from mathematical induction that the learned policy $\hat{\pi}$ and action-value function are $\eps$-optimal with probability at least $1 - \delta$.

Next, we bound the number of required samples. The number of samples used is
\[
\sum_{t=0}^{H-1}(k(|A|+|S|))N_{H-t}(t+1)  + k^2N_{H-t}^{\#}(t+1)
\]
where the $t+1$ comes from the length of the rollout. With our choice of $N_{H-t}$, it follows that 
\begin{align*}
\sum_{t=0}^{H-1}(k(|A|&+|S|))N_{H-t}(t+1) \\
&= 
\sum_{t=0}^{H-1}(k(|A|+|S|))\frac{2(t+1)^3(c')^2 k^2\alpha^2H^2\log(2H|S||A|/\delta)}{\eps^2} +\frac{2(t+1)^3(c')^2 k^6\alpha^4H^2\log(2H|S||A|/\delta)}{\eps^2}   \\
&\in \Tilde{O}\left(\frac{k^3\alpha^2(|S| + |A|)H^6}{\eps^2} + \frac{k^6\alpha^4H^6}{\eps^2} \right).
\end{align*}
\end{proof}

\begin{proof}{Proof of Theorem \ref{thm:tklr}}
This proof follows the same steps as the previous two proofs.
Assume that for any $\eps$-optimal value function $V_{h+1}$, the matrix corresponding to $Q'_h = [r_h + [P_h V_{h+1}]]$ is rank $d$, and  $S^\#_h, A^\#_h$ are $(k, \alpha)$-anchor states and actions for $\hat{Q}'_h = [r_h + [P_h \hat{V}_{h+1}]]$, where $\hat{V}_{h+1}$ is the learned value function from Low Rank Empirical Value Iteration for all $h \in [H]$. Let $N_{H-t} = \frac{2(t+1)^2(c')^2 k^2\alpha^2H^2\log(2H|S||A|/\delta)}{\eps^2}, N_{H-t}^{\#} = \alpha^2k^2 N_{H-t}$, where $c'$ satisfies the inequality in Lemma \ref{lem:ME_error_amplification}, and $Q'_h = [r_h + P_h\hat{V}_{h+1}]$ for all $\eps$-optimal value fucntions $\hat{V}_{h+1}$ for all $h \in [H]$. We prove the correctness of LR-EVI with high probability with induction on $t$ that 
\[
\|\bar{Q}_{H-t}- Q^*_{H-t}\|_\infty\leq \frac{\eps(t+1)}{H}, \qquad \|\bar{Q}_{H-t}- Q^{\hat{\pi}}_{H-t}\|_\infty\leq \frac{\eps(t+1)}{H}
\]
where $\bar{Q}_{H-t}$ and $\hat{\pi}_{H-t}$ are the learned $Q$ function and policy with probability at least $1 - \delta(t+1)/H$. 

The base case occurs at step $t=0$ in which case our estimates, $\hat{Q}_H(s,a) = \frac{1}{N_{s,a, H}}\sum_{i=1}^{N_{s,a, H}}r_{H}^i(s,a)$ over $\Omega_H$, are only averages of realizations $r_{H}^i \sim R_H(s,a)$ since $\hat{V}_{H+1}= \Vec{0}$. Since $R_H(s,a)$ has bounded support for all $(s,a) \in S \times A$, from Hoeffding's inequality (Theorem \ref{thm:hoef}) with our choice of $N_{s,a, H}$, 
\begin{align*}
    |\hat{Q}_{H}(s, a) - Q^*_H(s,a)| &\leq \frac{\eps}{c'k\alpha H} \qquad \qquad \forall (s, a) \in \Omega_h\\
    |\hat{Q}_{H}(s, a) - Q^*_H(s,a)| &\leq \frac{\eps}{c'k^2\alpha^2H} \qquad \qquad \forall (s, a) \in \Omega_h^{\#}
\end{align*}
with probability at least $1 - \delta/H$ because $|\Omega
_h| \leq |S||A|$. Step 2 of LR-EVI gives 
\[
|\bar{Q}_{H}(s, a) - Q^*_H(s,a)| \leq \frac{\eps}{H}
\]
for all $(s,a) \in S \times A$ from Lemma \ref{lem:ME_error_amplification}.  Since $Q^*_H = Q_H^{\hat{\pi}}$, the base case holds. 

Next, let $x \in \{0, \ldots, H-1\}$. Assume that the inductive hypothesis, the action-value function estimates $\bar{Q}_{H-x}$ and learned policy $\hat{\pi}_{H-x}$ satisfy  
\[
\|\bar{Q}_{H-x} - Q^*_{H-x}\|_\infty \leq \frac{(x+1)\eps}{H}, \quad \|\bar{Q}_{H-x} - Q^{\hat{\pi}}_{H-x}\|_\infty \leq \frac{(x+1)\eps}{H}
\]
holds with probability at least $1 - \delta(x+1)/H$. Following Step 1 from LR-EVI, we have 
\[
\hat{Q}_{H-x-1}(s, a) = \hat{r}_{H-x-1}(s,a) + \E_{s' \sim \hat{P}_{H-x-1}(\cdot|s,a)}[\hat{V}_{{H-x}}(s')],
\]
an unbiased estimate of $Q'_{H-x-1}(s,a) = r_{H-x-1}(s,a) + \E_{s' \sim P_{H-x-1}(\cdot|s,a)}[\hat{V}_{{H-x}}(s')]$, which is bounded. Hence, from Hoeffding's inequality (Theorem \ref{thm:hoef}), with the choice of 
$N_{H-x-1} = \frac{(x+2)^2(c')^2k^2H^2\alpha^2\log(2H|S||A|/\delta)}{2\eps^2}$,
$ N_{H-x-1}^{\#} = \alpha^2k^2 N_{H-x-1}$, it follows that
\begin{align*}
|\hat{Q}_{H-x-1}(s, a) - Q'_{H-x-1}(s, a)| &\leq \frac{\eps }{2c'k \alpha H} \qquad \qquad \forall (s, a) \in \Omega_{H-x-1}\\
|\hat{Q}_{H-x-1}(s, a) - Q'_{H-x-1}(s, a)| &\leq \frac{\eps }{2c'k^2 \alpha^2 H } \qquad \quad \forall (s, a) \in \Omega_{H-x-1}^{\#}
\end{align*}
with probability $1- \frac{\delta}{H|S||A|}$. Step 2 of LR-EVI gives
\[
|\bar{Q}_{H-x-1} - Q'_{H-x-1}|_\infty \leq \frac{\eps}{H}
\]
from Lemma \ref{lem:ME_error_amplification}. The union bound asserts that the above error guarantee and $\bar{Q}_{H-x'}$ is close to $Q^*_{H-x'}$ and $Q^{\hat{\pi}}_{H-x'}$ for $x \in [x]$ holds with probability at least $1- \delta(x+2)/H$. Hence, for all $(s,a) \in S \times A$,
\begin{align*}
    |\bar{Q}_{H-x-1}(s,a)   - Q^*_{H-x-1}(s,a)| &\leq |\bar{Q}_{H-x-1}(s,a)   - Q'_{H-x-1}(s,a)| +|Q'_{H-x-1}(s,a)   - Q^*_{H-x-1}(s,a)| \\
    &\leq \frac{\eps}{H} + |\E_{s'\sim P_{H-x-1}(\cdot| s,a)}[\max_{a \in A} \bar{Q}_{H-x}(s', a') - V^*_{H-x}(s')]| \\
    &\leq \frac{\eps}{H}+ |\E_{s'\sim P_{H-x-1}(\cdot| s,a)}[(x+1)\eps/H]| \\
    &= \frac{(x+2)\eps}{H}
\end{align*}
Thus, $\bar{Q}_{H-x-1}$ is $(x+2)\eps/H$-optimal. Next, we note that 
\begin{align*}
    |\bar{Q}_{H-x-1}(s,a)   - Q^{\hat{\pi}}_{H-x-1}(s,a)| &\leq |\bar{Q}_{H-x-1}(s,a)   - Q'_{H-x-1}(s,a)| +|Q'_{H-x-1}(s,a)   - Q^{\hat{\pi}}_{H-x-1}(s,a)| \\
    &\leq \frac{\eps}{H} + |\E_{s'\sim P_{H-x-1}(\cdot| s,a)}[\max_{a \in A} \bar{Q}_{H-x}(s', a') - V^{\hat{\pi}}_{H-x}(s')]| \\
    &\leq \frac{\eps}{H}+ |\E_{s'\sim P_{H-x-1}(\cdot| s,a)}[(x+1)\eps/H]| \\
    &= \frac{(x+2)\eps}{H}
\end{align*}
where the third inequality holds because 
\begin{align*}
    |\max_{a\in A}\bar{Q}_{H-x}(s', a') - V^{\hat{\pi}}_{H-x}(s')| &\leq | \E_{a' \sim \hat{\pi}_{H-x}(s')}[\bar{Q}_{H-x}(s', a')] - \E_{a' \sim \hat{\pi}_{H-x}(s')}[ Q^{\hat{\pi}}_{H-x}(s', a')]\\
    &\leq \|\hat{Q}_{H-x} - Q^{\hat{\pi}}_{H-x}\|_\infty\\
    &\leq \frac{(x+1)\eps}{H}
\end{align*}
from the induction hypothesis, and the inductive step holds for $x+1$. It follows from mathematical induction (and the triangle inequality) that the learned policy $\hat{\pi}$ and action-value function are $2\eps$ and $\eps$-optimal with probability at least $1 - \delta$. Scaling $N_h$ by a factor of four results in learning an $\eps$-optimal policy with probability at least $1-\delta$ without changing the sample complexity's dependence on $|S|, |A|, H,$ or $\eps$.

Next, we bound the number of required samples. The number of samples used is
\[
\sum_{t=0}^{H-1}(k(|A|+|S|))N_{H-t} +  k^2N_{H-t}^{\#}.  
\]
Note that there is no $(t+1)$ term as samples are single transitions instead of rollouts. With our choice of $N_{H-t}$, it follows that 
\begin{align*}
\sum_{t=0}^{H-1}(k(|A|&+|S|))N_{H-t} \\
&= \sum_{t=0}^{H-1}k(|A|+|S|) \frac{8(t+1)^2(c')^2 k^2\alpha^2H^2\log(2H|S||A|/\delta)}{\eps^2} + \frac{8(t+1)^2(c')^2 k^6\alpha^4H^2\log(2H|S||A|/\delta)}{\eps^2} \\
&\in \Tilde{O}\left(\frac{k^3\alpha^2(|S| + |A|)H^5}{\eps^2} + \frac{k^6\alpha^4 H^5}{\eps^2} \right).
\end{align*}

\end{proof}

\section{Proofs for Approximately Low Rank Models} \label{app:approx}
We first present the proof of Proposition \ref{prop:lrEstApprox}, which shows  that if the reward function and transition kernel are low rank, then for any value function estimate $\hat{V}_{h+1}$, $r_h + [P_h\hat{V}_{h+1}]$ has rank upper bounded by $d$.
\begin{proof}[Proof of Proposition \ref{prop:lrEstApprox}]
Let  $\xi_{R}, \xi_{P},  r_{h,d},  [P_{h, d}\hat{V}_{h+1}]$ be defined as in Section \ref{sec:approx}. Then, for all $(s, a, h) \in S \times A\times [H]$,
\begin{align*}
    [r_{h, d} +  P_{h, d}\hat{V}_{h+1}](s,a) &- [r_h +  P_{h}\hat{V}_{h+1}](s,a)| \\
    &\leq |[r_h - r_{h, d}](s,a)| + |[(P_{h, d} - P_{h} ) \hat{V}_{h+1}](s,a)|\\
    &= \xi_{R} + |\sum_{s'\in S}\hat{V}_{h+1}(s')(P_{h, d}(s'|s,a) - P_{h}(s'|s,a))| \\
    &\leq  \xi_{ R} +  (H-h)|\sum_{s'\in S}(P_{h, d}(s'|s,a) - P_{h}(s'|s,a))|\\
    &= \xi_{R} +  (H-h)2d_{\mathrm{TV}}(P_h(\cdot|s, a), P_{h, d}(\cdot| s,a))_{\mathrm{TV}}\\
    &= \xi_{R} + (H-h) \xi_{P}
\end{align*}
since $\hat{V}_{h+1}(s) \in [0, H-h-1]$.
\end{proof}

We next prove that the learned policy's error is additive with respect to the approximation error.

\begin{proof}{Proof of Theorem \ref{thm:tklrApprox}}
This proof follows the same steps as the proof of Theorem \ref{thm:tklr} while accounting for the approximation error. Assume that we have a $(d, \xi_R, \xi_P)$-approximately low-rank MDP.  Let $S^\#_h, A^\#_h$ be $(k, \alpha)$-anchor states and actions,  $c'$ be a constant that satisfies the inequality in Lemma \ref{lem:ME_error_amplification}, $N_{H-t} = \frac{2(t+1)^2(c')^2 k^2\alpha^2H^2\log(2H|S||A|/\delta)}{\eps^2}, N_{H-t}^{\#} =\alpha^2k^2 N_{H-t}$, and $Q'_h = [r_h + P_h\hat{V}_{h+1}]$ for all $\eps$-optimal value functions $\hat{V}_{h+1}$ for all $h \in [H]$. We prove the correctness of LR-EVI with high probability with induction on $t$ that 
\[
\|\bar{Q}_{H-t} - Q^*_{H-t}\|_\infty, \|\bar{Q}_{H-t} - Q^{\hat{\pi}}_{H-t}\|_\infty \leq (t+1)\eps/H + \sum_{i = 0}^{t} (c'k^2 \alpha^2+1)\left(\xi_{R} + i \xi_{ P}\right)
\]
where $\bar{Q}_{H-t}$ and $\hat{\pi}_{H-t}$ are the learned $Q$ function and policy with probability at least $1-\delta(t+1)/H$ for each $t\in \{0,\ldots, H-1\}$.

The base case occurs at step $t=0$ in which case our estimates, $\hat{Q}_H(s,a) = \frac{1}{N_{s, a, H}}\sum_{i=1}^{N_{s, a, H}}r_{H,i}(s,a)$ over $\Omega_H$, are only averages of realizations $r_{H}^i \sim R_H(s,a)$ since $\hat{V}_{H+1}= \Vec{0}$. Since $R_H(s,a)$ has bounded support for all $(s,a) \in S \times A$, from Hoeffding's inequality (Theorem \ref{thm:hoef}) with our choice of $N_{s, a, H}$
\begin{align*}
    |\hat{Q}_{H}(s, a) - Q^*_H(s,a)| &\leq \frac{\eps}{c'k\alpha H} \qquad \qquad \forall (s, a) \in \Omega_h\\
    |\hat{Q}_{H}(s, a) - Q^*_H(s,a)| &\leq \frac{\eps}{c'k^2\alpha^2H} \qquad \qquad \forall (s, a) \in \Omega_h^{\#}
\end{align*}
with probability at least $1 - \delta/H$ because $|\Omega
_h| \leq |S||A|$. Under the event above, it follows that $|\hat{Q}_{H}(s, a) - Q^*_{H, d}(s,a)| \leq \frac{\eps}{c'_H\alpha^2k^2H} + \xi_{ R}$ or $|\hat{Q}_{H}(s, a) - Q^*_{H, d}(s,a)| \leq \frac{\eps}{c'_H\alpha kH} + \xi_{ R}$ for all $(s,a) \in \Omega_H$. Step 2 of LR-EVI gives 
\[
|\bar{Q}_{H}(s, a) - Q^*_{H, d}(s,a)| \leq \frac{\eps}{H} + Ck^2 \alpha^2 \xi_{ R}
\]
for all $(s,a) \in S \times A$ from Lemma \ref{lem:ME_error_amplification} for some positive constant $C$. By definition of the approximation error, 
\[
|\bar{Q}_{H}(s, a) - Q^*_{H}(s,a)| \leq \frac{\eps}{H}+ (C k^2\alpha^2 + 1) \xi_{ R}  \, \qquad \forall (s,a) \in S \times A.
\]
Since $Q^*_H = Q^{\hat{\pi}}_H$, the base case holds.

Next, let $x \in \{0, \ldots, H-1\}$. Assume that the inductive hypothesis, the action-value function estimates $\bar{Q}_{H-x}$ and learned policy $\hat{\pi}_{H-x}$ satisfy  
\[
\|\bar{Q}_{H-x} - Q^*_{H-x}\|_\infty \leq \frac{(x+1)\eps}{H}, \quad \|\bar{Q}_{H-x} - Q^{\hat{\pi}}_{H-x}\|_\infty \leq \frac{(x+1)\eps}{H} + \sum_{i = 0}^{x} (Ck^2\alpha^2 +1)\left(\xi_{ R} + i \xi_{P}\right)
\]
holds with probability at least $1 - \delta(x+1)/H$. At step $x+1$, following Step 1 from LR-EVI, we have 
\[
\hat{Q}_{H-x-1}(s, a) = \hat{r}_{H-x-1}(s,a) + \E_{s' \sim \hat{P}_{H-x-1}(\cdot|s,a)}[\hat{V}_{{H-x}}(s')],
\]
an unbiased estimate of $Q'_{H-x-1}(s,a) = r_{H-x-1}(s,a) + \E_{s' \sim P_{H-x-1}(\cdot|s,a)}[\hat{V}_{{H-x}}(s')]$. Furthermore, $\hat{Q}_{H-x-1}(s, a) \in [0, x+2]$ is a bounded random variable because of bounded rewards. Hence, from Hoeffding's inequality (Theorem \ref{thm:hoef}), with the choice of 
$
N_{H-x-1} = \frac{(x+2)^2(c')^2k^2 \alpha^2 H^2\log(2H|S||A|/\delta)}{2\eps^2}, N_{H-x-1}^{\#} = \alpha^2k^2 N_{H-x-1}, $
it follows that
\begin{align*}
  |\hat{Q}_{H-x-1}(s, a) - Q'_{H-x-1}(s, a)| &\leq \frac{\eps }{2c'k\alpha  H} \qquad \qquad \forall (s, a) \in \Omega_{H-x-1}\\
   |\hat{Q}_{H-x-1}(s, a) - Q'_{H-x-1}(s, a)| &\leq \frac{\eps }{2c'k^2\alpha ^2 H} \qquad \qquad \forall (s, a) \in \Omega_{H-x-1}^{\#}
\end{align*}

with probability $1- \frac{\delta}{H|S||A|}$. Under the event above, it follows that $|\hat{Q}_{H-x-1}(s, a) - Q'_{H-x-1}(s, a)| \leq \frac{\eps}{c'k^2\alpha^2H} + \xi_{R} + (H-x-1)\xi_{P}$ for all $(s,a) \in \Omega_{H-x-1}$ where $Q'_{h, d} = r_{h, d} + [ P_{h, d}\hat{V}_{h+1}]$. Step 2 of LR-EVI gives
\[
|\bar{Q}_{H-x-1} - Q'_{H-x-1, d}|_\infty \leq \frac{\eps}{H} + Ck^2\alpha^2 \left( \xi_{R} +  (H-x-1)\xi_{P} \right)
\] 
from Lemma \ref{lem:ME_error_amplification} for some positive constant $C$. The union bound asserts that the above error guarantee holds with probability at least $1- \delta(x+2)/H$. Hence, for all $(s,a) \in S \times A$,
\begin{align*}
    &\;\;\quad |\bar{Q}_{H-x-1}(s,a)   - Q^*_{H-x-1}(s,a)|\\
     &\leq |\bar{Q}_{H-x-1}(s,a)   - Q'_{H-x-1, d}(s,a)| + |Q'_{H-x-1,d}(s,a)  - Q'_{H-x-1}(s,a) | \\
    &+ |Q'_{H-x-1}(s,a)   - Q^*_{H-x-1}(s,a)| \\
    &\leq \frac{\eps}{H} + c'k^2\alpha^2 \left( \xi_{R} +  (H-x-1)\xi_{P} \right)\\
    &+ \xi_{R} +  (H-x-1)\xi_{P} \\
    &+|\E_{s'\sim P_{H-x-1}(\cdot| s,a)}[\max_{a \in A} \bar{Q}_{H-x}(s', a') - V^*_{H-x}(s')]| \\
    &\leq \frac{\eps}{H}+ (1 + c'k^2\alpha^2 ) \left( \xi_{R} + (H-x-1)\xi_{P} \right)\\
    &+     |\E_{s'\sim P_{H-x-1}(\cdot| s,a)}[(x+1)\eps/H + \sum_{i = 0}^{x} (Ck^2\alpha^2 +1)\left(\xi_{ R} +  i \xi_{P}\right)]| \\
    &= \frac{(x+2)\eps}{H} + \sum_{i = 0}^{ x+1} (Ck^2\alpha^2 +1)\left(\xi_{ R} +  i\xi_{P}\right)].
\end{align*}
With a similar argument, it follows that for all $(s,a) \in S\times A$, 
\[
\bar{Q}_{H-x-1}(s,a)   - Q^{\hat{\pi}}_{H-x-1}(s,a)| \leq \frac{(x+2)\eps}{H} + \sum_{i = 0}^{ x+1} (Ck^2\alpha^2 +1)\left(\xi_{ R} +  i\xi_{P}\right)].
\]
Thus, the inductive step holds for $x+1$, and from mathematical induction, the lemma holds.

Choosing $t = H-1$ proves the correctness of the algorithm. Next, we bound the number of required samples. The number of samples used is the same as in the proof of Theorem \ref{thm:tklr}, which implies a sample complexity of 
\[
\Tilde{O}\left(\frac{k^3\alpha^2(|S| + |A|)H^5}{\eps^2} + \frac{k^6\alpha^4 H^5}{\eps^2} \right).
\]

\end{proof}

\section{Proofs for Continuous MDPs} \label{app:cont}
We next present the proofs of the results in Section \ref{sec:cont}, starting with our procedure on how we obtain samples/rollouts from the discretized MDP. 

Using the generative model, we simulate trajectories from $M^\beta$ with the following procedure: to sample a trajectory from $M^\beta$ following policy $\pi$ starting at $(s,a, h)$, first sample a state $s'$ from $P_h(\cdot|s, a)$. Then, we take the closest discretized state $s'_\beta$ to $s'$ , i.e., $s'_\beta = \arg\min_{s\in S} \|s - s'\|_2$, to be the observed state in the trajectory. The generative model is then used to sample from $P_{h+1}(\cdot| s'_\beta, \pi(s'_\beta))$, and we repeat until the end of the horizon to obtain a trajectory from $P^\beta_h$ using the generative model on the original MDP. Lemma \ref{lem:traj} asserts the correctness of this procedure. 


\begin{lemma}\label{lem:traj}
Let $\tau^\pi_h(s_h,a_h)$ be a rollout obtained with the procedure detailed above starting at state-action pair $(s_h,a_h)$ at step $h$ with policy $\pi$, $\tau^\pi_{\beta,h}(s_h,a_h)$ be a rollout following policy $\pi$ from the discretized MDP starting at state-action pair $(s_h,a_h)$ at step $h$, and $\tau$ be a realization of a rollout following policy $\pi$ from the discretized MDP starting at state-action pair $(s_h,a_h)$ at step $h$. Then,
\[
\P(\tau_h(s_h,a_h)) = \tau) =  \P(\tau^\beta_h(s_h,a_h) = \tau).
\]
\end{lemma}
\begin{proof}[Proof of Lemma \ref{lem:traj}]
Let $\tau = (s_{h+1}, \pi(s_{h+1}), \ldots, s_{H}, \pi(s_{H}))$. From the Markov Property and the procedure defined above, it follows that 
\begin{align*}
    \P(\tau_h(s_h,a_h)) = \tau) &= \int_{\{s': |s_{h+1} - s'|_2 \leq \beta\}} P(s'|s_h,a_h) \mathrm{d}s'  \mathlarger{\Pi}_{i = h+1}^{H-1}  \int_{\{s': |s_{i+1} - s'|_2 \leq \beta\}} P(s'|s_i,\pi(s_i)) \mathrm{d}s'\\
    &= P^\beta_h(s_{h+1} | s_h, a_h)\mathlarger{\Pi}_{i = h+1}^{H-1} P^\beta_h(s_{i+1} | s_i, \pi(s_i))\\
    &= \P(\tau^\beta_h(s_h,a_h) = \tau).
\end{align*}
\end{proof}

We next present the proof of Lemma \ref{lem:bnet}, which allows us to use $Q^\beta_h$ to estimate $Q^*_h$.

\begin{proof}[Proof of Lemma \ref{lem:bnet}]
We prove this lemma via induction. For $h = H$, by construction of the $\beta$-nets, for any $s \in S, a \in A$ and $s' \in S^\beta, a' \in A^\beta$ such that $\|s - s'\|_2 \leq \beta, \|a - a'\|_2 \leq \beta$,
\[
|Q^*_H(s, a) - Q^\beta_H(s', a')| \leq 2L\beta 
\]
because $Q^*_H$ is $L$-Lipschitz. For all $s \in S$ and $s' \in S^\beta$ such that $\|s-s'\|_2 \leq \beta$, let $a_{max} = \argmax_{a\in A} Q^*_H(s, a)$ and $d_A(a)$ be the function that maps the action $a$ to the closest action in $A^\beta$, which is at most $\beta$ away. It follows that 
\begin{align*}
    |V^*_H(s) - V^\beta_H(s')| &= |Q^*_H(s, a_{max}) - \max_{a^*\in A^\beta} Q^\beta_H(s', a^*)| \\
    &\leq |Q^*_H(s, a_{max}) - Q^\beta_H(s', d(a_{max}))| \\
    &\leq 2L\beta
\end{align*}
where the first inequality comes from $Q^*_H(s, a_{max})$ being an upperbound of $Q^\beta_H(s', a')$ and the max operator, and the second inequality comes from $Q^*_H$ being Lipschitz. Next, assume that for any $s \in S$ and $s' \in S^\beta$ such that $\|s-s'\|_2 \leq \beta$, $|V^*_{H-t + 1}(s) - V^\beta_{H-t + 1}(s')| \leq 2(H-t+1)L\beta$. Let $d_S(s)$ be the function that maps the state $s$ to the closest state in $S^\beta$, which is at most $\beta$ away. For any $s \in S, a \in A$ and $s' \in S^\beta, a' \in A^\beta$ such that $\|s - s'\|_2 \leq \beta, \|a - a'\|_2 \leq \beta$,
\begin{align*}
&\;\quad |Q^*_{H-t}(s, a) - Q^\beta_{H-t}(s', a')| \\
&\leq |Q^*_{H-t}(s, a) - Q^*_{H-t}(s', a')| + |Q^*_{H-t}(s', a') - Q^\beta_{H-t}(s', a')|\\
&\leq 2L\beta + |\E_{s^* \sim P(\cdot| s', a')}[V^*_{H-t+1}(s^*)] - \E_{s^* \sim P^\beta(\cdot| s', a')}[V^\beta_{H-t+1}(s^*)] |\\
&= 2L\beta + |\int_{s^*\in S} P_{H+1}(s^*|s',a')V^*_{H-t+1}(s^*) \mathrm{d}s^* - \sum_{s^*\in S^\beta}P^\beta_{H+1}(s^*|s', a') V^\beta_{H-t+1}(s^*)\\
&=  2L\beta + |\int_{s^*\in S} P_{H+1}(s^*|s',a')V^*_{H-t+1}(s^*) \mathrm{d}s^* \\
&\qquad - \sum_{s^*\in S^\beta} \int_{\{s'' \in S: |s^* - s'' |_2 \leq \beta\}} V^\beta_{H-t+1}(s^*) P_h(s''| s',a') \mathrm{d}s^*|\\   
&= 2L\beta + |\int_{s^*\in S} P_{H+1}(s^*|s',a')V^*_{H-t+1}(s^*) \mathrm{d}s^* - \int_{s^*\in S} P_{H+1}(s^*|s',a')V^\beta_{H-t+1}(d_S(s^*)) \mathrm{d}s^*|\\
&\leq 2L\beta + |\E_{s^* \sim P(\cdot| s', a')} 2(H-t+1)L\beta|\\
&= 2(H-t)L\beta
\end{align*}
where the fourth line comes from the definition of $P^\beta_h$. For all $s \in S$ and $s' \in S^\beta$ such that $\|s-s'\|_2 \leq \beta$, let $a_{max} = \argmax_{a\in A} Q^*_{H-t}(s, a)$. It follows that 
\begin{align*}
    |V^*_{H-t}(s) - V^\beta_{H-t}(s')| &= |Q^*_{H-t}(s, a_{max}) - Q^\beta_{H-t}(s', a'_{max})| \\
    &\leq |Q^*_{H-t}(s, a_{max}) - Q^\beta_{H-t}(s', d_A(a_{max}))| \\
    &\leq 2(H-t)L\beta
\end{align*}
Thus, from induction, for any $s \in S, a \in A$ and $s' \in S^\beta, a' \in A^\beta$ such that $\|s - s'\|_2 \leq \beta, \|a - a'\|_2 \leq \beta$, for all $h \in [H]$, 
\[
|Q^*_h(s,a) - Q^\beta_h(s', a')| \leq 2L(H-h+1)\beta, \quad |V^*_h(s,a) - V^\beta_h(s', a')| \leq 2L(H-h+1)\beta.
\]
\end{proof}

To prove the desired sample complexity bounds, we first state a lemma on covering numbers that upperbounds the number of points required for our $\beta$-nets.

\begin{lemma}[Theorem 14.2 from \citep{wu}]\label{lem:covNum}
Let $\Theta \subset \R^n$. Then, 
\[
N(\Theta, \eps) \leq \left (\frac{3}{\eps} \right)^n\frac{Vol(\Theta)}{Vol(B)}
\]
where $N(\Theta, \eps)$ is the covering number of $\Theta$, and $B$ is the unit norm ball in $\R^n$.
\end{lemma}

We next present the sample complexity bound of LR-MCPI when $M^\beta$ satisfies Assumption \ref{asm:lrqe} and its proof. 
\begin{theorem}\label{thm:conNOLR}
Let $Q^\pi_{h, \beta} = [Q^\pi_h(s,a)]_{(s,a) \in S^\beta \times A^\beta}$, the action-value function of policy $\pi$ at step $h$ on only the discretized state-action pairs. After discretizing the continuous MDP, let Assumption \ref{asm:lrqe} hold on $M^\beta$. Furthermore, assume that $S^\#_h, A^\#_h$ are $(k, \alpha)$-anchor states and actions for $Q^{\pi, \beta}_h$ for all $h \in [H]$. Let $\bar{Q}_h$ be the action-value function estimates that Low Rank Monte Carlo Policy Iteration for Step 1  return for all $h\in [H]$ when run on $M^\beta$. For any $s \in S, a \in A$ and $s' \in S^\beta, a' \in A^\beta$ such that $\|s - s'\|_2 \leq \beta, \|a - a'\|_2 \leq \beta$ and $h\in [H]$, 
\[
|\bar{Q}_h(s', a') - Q^*_h(s,a)| \leq \eps
\]
with probability at least $1-\delta$ when $\beta = \frac{\eps}{4LH}$, $N_{H-t} = \frac{8(t+1)^2(c')^2H^2k^2\alpha^2\log(2H|S||A|/\delta)}{\eps^2}$, $N_{H-t}^{\#} = \alpha^2k^2 N_{H-t}$, and $c'$ satisfies the inequality in Lemma \ref{lem:ME_error_amplification} for all $t \in \{0, \ldots H-1\}$. Furthermore, at most $\Tilde{O}\left(\frac{k^3\alpha^2H^{n+6}}{\eps^{n+2}Vol(B)} \right)$ number of samples are required with the same probability where $B$ is the unit norm ball in $\R^n$.
\end{theorem}

\begin{proof}[Proof of Theorem \ref{thm:conNOLR}]

After discretizing the continuous MDP to get $M^\beta$ for $\beta = \frac{\eps}{4LH}$, we note that $|S^\beta|, |A^\beta| \in O(\frac{H^n}{\eps^n Vol(B)})$ from Lemma \ref{lem:covNum}. Since the required assumptions for Theorem \ref{thm:qnolr} hold on $M^\beta$, it follows that each $\bar{Q}_{h}$ is $\eps(H-h+1)/H$-optimal for all $h \in [H]$ on $M^\beta$ when running LR-MCPI with $N_{H-t} = \frac{8(t+1)^2(c')^2H^2k^2\alpha^2\log(2H|S||A|/\delta)}{\eps^2}, N_{H-t}^{\#} = \alpha^2k^2 N_{H-t}$ using at most $\Tilde{O}\left(\frac{k^3\alpha^2 H^{n+6}}{\eps^{n+2}Vol(B)} \right)$ samples with probability at least $1-\delta$.  Since $\beta = \frac{\eps}{4LH}$, from Lemma \ref{lem:bnet}, for any $s \in S, a \in A$ and $s' \in S^\beta, a' \in A^\beta$ such that $\|s - s'\|_2 \leq \beta, \|a - a'\|_2 \leq \beta$  and for all $t \in \{0, \ldots H-1\}$, 
\begin{align*}
|\bar{Q}_{H-t}(s', a') - Q^*_{H-t}(s,a)| &\leq |\bar{Q}_{H-t}(s', a') - Q^\beta_{H-t}(s',a')|+|Q^\beta_{H-t}(s', a') - Q^*_{H-t}(s,a)|\\
&\leq \frac{\eps}{2} + 2L(t+1)\beta \\
&\leq  \eps.
\end{align*}
Hence, an $\eps$-optimal $Q$ function on the continuous space is $\bar{Q}_h^c (s,a) = \bar{Q}_h(s', a')$, where $(s', a')$ is the discretized state-action pair closest to $(s,a)$.
\end{proof}

\begin{proof}[Proof of Theorem \ref{thm:conTKLR}]
After discretizing the continuous MDP to get $M^\beta$ for $\beta = \frac{\eps}{4LH}$, we note that $|S^\beta|, |A^\beta| \in O(\frac{H^n}{\eps^n Vol(B)})$ from Lemma \ref{lem:covNum}. Since Assumption \ref{asm:lrtk} holds on $M^\beta$, from Theorem \ref{thm:tklr}, it follows that each $\bar{Q}_{h}$ is $\eps/2$-optimal for all $h \in [H]$ on $M^\beta$ when running LR-EVI with $N_{H-t} = \frac{4(t+1)^2(c')^2k^2\alpha^2H^2\log(2H|S||A|/\delta)}{\eps^2}, N_{H-t}^{\#} = \frac{4(t+1)^2(c')^2k^4\alpha^4H^2\log(2H|S||A|/\delta)}{\eps^2}$ using at most $\Tilde{O}\left(\frac{k^3\alpha^2H^{n+5}}{\eps^{n+2}Vol(B)}\right )$ samples with probability at least $1-\delta$.  Since $\beta = \frac{\eps}{4LH}$, from Lemma \ref{lem:bnet}, for any $s \in S, a \in A$ and $s' \in S^\beta, a' \in A^\beta$ such that $\|s - s'\|_2 \leq \beta, \|a - a'\|_2 \leq \beta$  and for all $t \in \{0, \ldots H-1\}$, 
\begin{align*}
|\bar{Q}_{H-t}(s', a') - Q^*_{H-t}(s,a)| &\leq |\bar{Q}_{H-t}(s', a') - Q^\beta_{H-t}(s',a')|+|Q^\beta_{H-t}(s', a') - Q^*_{H-t}(s,a)|\\
&\leq \frac{\eps}{2} + 2L(t+1)\beta \\
&\leq  \eps.
\end{align*}
Hence, an $\eps$-optimal $Q$ function on the continuous space is $\bar{Q}_h^c (s,a) = \bar{Q}_h(s', a')$, where $(s', a')$ is the discretized state-action pair closest to $(s,a)$.
\end{proof}


\section{Proofs for Infinite-Horizon Discounted MDPs} \label{app:infHP}
In this section, we present the omitted proofs from Appendix \ref{sec:IH}. We first prove that for any estimate of the value function, $r+P\hat{V}_t$ has rank that is at most $d$. 
\begin{proof}[Proof of Proposition \ref{prop:lrEstI}]
Let MDP $M = (S,A,P,R, \gamma)$ satisfy Assumption \ref{asm:lrtkI}. For the Tucker rank $(|S|, |S|, d)$ case,it follows that 
\begin{align*}
    r(s, a) + \gamma \E_{s' \sim P(\cdot|s, a)} [\hat{V}(s')] &= \textstyle\sum_{i=1}^d W(s, i)V(a, i) + \gamma\textstyle\sum_{s' \in S}\textstyle\sum_{i=1}^d U(s', s, i) V(a, i) \hat{V}(s')\\
    &= \textstyle\sum_{i=1}^d V(a, i)\left(W(s, i) + \gamma\textstyle\sum_{s' \in S}U(s', s, i)\hat{V}(s')  \right)
\end{align*}
Thus, $r+ \gamma [P\hat{V}]$ has rank upper bounded by $d$, and the Tucker rank $(|S|, d, |A|)$ case follows the same steps. 
\end{proof}

To prove the correctness of LR-EVI for the infinite-horizon setting, we first show that the error of the $Q$-function decreases in each iteration, Lemma \ref{lem:tklrI}. 
\begin{proof}[Proof of Lemma \ref{lem:tklrI}]
Let $Q'_{t+1} = r + \gamma P\bar{V}_{t}$ and $t \in [T-1]$. From proposition \ref{prop:lrEstI}, $Q'_{t+1}$ has rank at most $d$ for all $t \in [T-1]$ Following step 1 from LR-EVI, $\hat{Q}_{t+1}(s,a) = \frac{1}{N_{t+1}}\sum_{i=1}^{N_{t+1}}R(s,a) + \gamma \bar{V}_t(s'_i)$
for all $(s,a) \in \Omega_{t+1}$. Hence, $\hat{Q}_{t+1}(s,a)$ is an unbiased estimate of $Q'_{t+1}(s,a)$ for all $(s,a) \in \Omega_t$. Furthermore, because of bounded rewards, $\hat{Q}_{t+1}(s,a) \in [0, \frac{1}{1-\gamma}]$ is a bounded random variable. With our choice of $N_{t+1} = \frac{2(c')^2k^2 \alpha^2\log(2T|S||A|/\delta)}{(1-\gamma)^4B_{t}^2}, N^\#_{t+1} = N_{t+1}\alpha^2, k^2$, it follows from Hoeffding's inequality that for all $(s,a) \in \Omega_{t+1}$, 
\begin{align*}
|\hat{Q}_{t+1}(s,a) - Q'_{t+1}(s,a)| \leq \frac{(1-\gamma)B_{t}}{2c'\alpha k} \quad \forall (s, a) \in \Omega_{t+1} \\
|\hat{Q}_{t+1}(s,a) - Q'_{t+1}(s,a)| \leq \frac{(1-\gamma)B_{t}}{2c'\alpha^2 k^2} \quad \forall (s, a) \in \Omega^\#_{t+1}
\end{align*}
with probability at least $1- \frac{\delta}{T|S||A|}$. Step 2 of LR-EVI gives that for all $(s,a) \in S \times A$
\[
|\bar{Q}_{t+1}(s,a) - Q'_{t+1}(s,a) | \leq  \frac{(1-\gamma)B_t}{2}
\]
from Lemma \ref{lem:ME_error_amplification}. Hence, for all $(s,a) \in S \times A$, 
\begin{align*}
    |\bar{Q}_{t+1}(s,a) - Q^*(s,a)| &\leq |\bar{Q}_{t+1}(s,a) - Q'_{t+1}(s,a)| + |Q'_{t+1}(s,a) - Q^*(s,a)| \\
    &\leq \frac{(1-\gamma)B_t}{2} + | \gamma \E_{s'\sim P(\cdot|s,a)}[\bar{V}_t(s') - V^*(s')]| \\
    &\leq \frac{(1-\gamma)B_t}{2} + \gamma B_t \\
    &= \frac{(1+\gamma)B_t}{2}.
\end{align*}
From step 4 of LR-EVI, the estimate of the value function is defined as $\bar{V}_{t+1}(s) = \max_{a\in A} \bar{Q}_{t+1}(s, a)$ for all $s\in S$. It follows that $|\bar{V}_{t+1}(s) - V^*(s)| \leq \frac{(1+\gamma)B_t}{2}$. \end{proof}

\begin{proof}[Proof of Theorem \ref{thm:tklrI}]]
Since the value function estimate is initialized as the zero vector, $|\bar{V}_0 - V^*|_\infty \leq \frac{1}{1-\gamma} = B_0$. We prove the correctness of this algorithm by repeatedly applying Lemma \ref{lem:tklrI} $T$ times. From the union bound, the $\bar{Q}_T$ that the algorithm returns satisfies
\[
|\bar{Q}_T(s,a) - Q^*(s,a) | \leq \left( \frac{1+\gamma}{2}\right)^T \left( \frac{1}{1-\gamma}  \right)
\]
with probability at least $1- \delta$. With $T =   \frac{\ln(\eps(1-\gamma))}{\ln(\frac{1+\gamma}{2})}$, it follows that $( \frac{1+\gamma}{2})^T ( \frac{1}{1-\gamma}) = \eps$, so 
\[
|\bar{Q}_T(s,a) - Q^*(s,a) | \leq \eps
\]
with probability at least $1 - \delta $. Note that since $B_t$ is strictly decreasing with respect to $t$, it follows that $N_t = \frac{2(c')^2\alpha^2k^2 \log(2T/\delta)}{(1-\gamma)^4B_{t-1}^2}, N_t^\# = \alpha^2 k^2$ are strictly increasing with respect to $t$. Furthermore, since $B_{T-1} > \eps$, $N_t \in \Tilde{O}\left( \frac{\alpha^2 k^2}{(1-\gamma)^4\eps^2}\right)$ for all $t \in [T]$. It follows that the sample complexity of the algorithm is 
\[
\Tilde{O}\left (   \frac{\alpha^2k^3(|S|+|A|)}{\eps^2(1-\gamma)^{-4}} +  \frac{\alpha^4k^6}{\eps^2(1-\gamma)^{-4}}       \right).
\]
\end{proof}

\section{Proofs for LR-EVI with Matrix Estimation using Nuclear Norm Regularization}
In this subsection, we present the omitted proofs from Section \ref{sec:CME}. We first prove a lemma that gives us the matrix estimation guarantee in our desired form.

\begin{lemma}\label{lem:meCR}
Assume that for any $\eps$-optimal value function $\hat{V}_{h+1}$, the matrix corresponding to $[r_h + [P_h\hat{V}_{h+1}]]$ is rank $d$, $\mu$-incoherent, and has condition number bounded by $\kappa$. Then, for 
\[
p_h = \frac{\mu^3 d^2 \kappa^2 H^4 C_{cvx}^2\log(n)}{\eps^2n},
\]
where $C_{cvx}$ is defined as in Theorem \ref{thm:conRel} with probability $1 - O(n^{-3})$, we have
\[
\|\hat{Q}_{h} - r_h + [P_h \hat{V}_{h+1}]\|_\infty \leq \frac{\eps}{H}.
\]
\end{lemma}
\begin{proof}[Proof of Lemma \ref{lem:meCR}]
Since $Q_h(s,a)$ is bounded by $H-h$, the estimates in Step 2 of LR-EVI-cvx are bounded random variables. Hence, they are unbiased with sub-Gaussian parameter $H-h$ \cite{wainwright2019high}. Let $Q' = r_h + [P_h \hat{V}_{h+1}]$. From Theorem \ref{thm:conRel}, with probability $1 - O(n^{-3})$, 
\[
\|\bar{Q}_h - Q'_h\|_\infty \leq \frac{C_{cvx}(H-h)}{\sigma_r(Q'_h)}\sqrt{\frac{\mu n \log n}{p_h}}\|Q'\|_\infty.
\]
Let $Q'$ have singular value decomposition $U\Sigma V^T$. Then, for $(s,a) \in S\times A$, 
\begin{align*}
|Q'_h(s,a)| &= |e_s^T U \Sigma V^T e_a|\\
&\leq \|U(s)\|_2\|\Sigma\|_{op}\|V(a)\|_2\\
&\leq \frac{\mu d}{n} \sigma_1(Q'_h)\\
&\leq \frac{\mu d \kappa }{n} \sigma_d(Q'_h) 
\end{align*}
where the second inequality comes from incoherence and the last inequality comes from bounded condition number. Plugging this inequality into the application of Theorem \ref{thm:conRel} gives 
\[
\|\bar{Q}_h - Q'_h\|_\infty \leq \mu d \kappa C_{cvx}(H-h)\sqrt{\frac{\mu  \log n}{p_hn}}.
\]
From our choice of $p_h$, we get the desired result. 
\end{proof}

Next, we prove a helper lemma that follows the same steps as the helper lemmas needed to prove Theorems \ref{thm:gap}, \ref{thm:qnolr}, and \ref{thm:tklr}. Similar lemmas can be proved in the suboptimality gap or all $\eps$-optimal $\pi$ have low-rank $Q^\pi$ setting. 

\begin{lemma}\label{lem:MECRThm}
Let $\eps, p_{H-t}$, and $\lambda$ be defined as in Theorem \ref{thm:MECR}. Then, the learned policy and action-value function estimate satisfy
\[
\|\bar{Q}_{H-t} - Q^*_{H-t}\|_\infty \leq \frac{\eps(t+1)}{H}, \qquad \|\bar{Q}_{H-t} - Q^{\hat{\pi}}_{H-t}\|_\infty \leq \frac{\eps(t+1)}{H}
\]
with probability at least $1 - O((t+1)n^{-3})$ for all $t \in \{0, \ldots, H_1\}$.

\end{lemma}
\begin{proof}[Proof of Lemma \ref{lem:MECRThm}]
We prove this with induction on $t$. At step $t=0$, it follows that from Lemma \ref{lem:meCR} with probability $1 - O(n^{-3})$,
\[
\|\bar{Q}_{H} - Q^*_H\|_\infty \leq \frac{\eps}{H}.
\]
Since $Q^*_H = Q^{\hat{\pi}}_H$, the base case holds. 

Let $x \in [H-1]$. Assume that the inductive hypothesis,
\[
\|\bar{Q}_{H-x} - Q^*_{H-x}\|_\infty \leq \frac{\eps(x+1)}{H}, \qquad \|\|\bar{Q}_{H-x} - Q^{\hat{\pi}}_{H-x}\|_\infty \leq \frac{\eps(x+1)}{H}
\]
with probability at least $1- O\left((x+1)n^{-3}\right)$, holds. Following the steps of LR-EVI with the convex program based matrix estimation method, it follows that with probability $1 - O(n^{-3})$,
\[
\|\bar{Q}_{H-s-1} - Q'_{H-s-1}\|_\infty \leq \frac{\eps}{H}
\]
where $Q_{H-x-1}' = r_{H-x-1} + P_{H-x-1}\hat{V}_{H-x}$. The union bound asserts that the above error guarantee holds with probability at least $1 - O((x+2)n^{-3})$. Hence, for all $(s,a) \in S \times A$, 
\begin{align*}
    |\bar{Q}_{H-x-1}(s,a) - Q^*_{H-x-1}(s,a)| &\leq |\bar{Q}_{H-x-1}(s,a)   - Q'_{H-x-1}(s,a)| +|Q'_{H-x-1}(s,a)   - Q^*_{H-x-1}(s,a)| \\
    &\leq \frac{\eps}{H} + |E_{s'\sim P_{H-x-1}(\cdot| s,a)}[\hat{V}_{H-x}(s') - V^*_{H-x}(s')]| \\
    &\leq \frac{\eps}{H}+ |E_{s'\sim P_{H-h}(\cdot| s,a)}[(x+1)\eps/H]| \\
    &= \frac{(x+2)\eps}{H}.
\end{align*}
Following the same steps, 
\begin{align*}
    |\bar{Q}_{H-x-1}(s,a) - Q^{\hat{\pi}}_{H-x-1}(s,a)| &\leq |\bar{Q}_{H-x-1}(s,a)   - Q'_{H-x-1}(s,a)| +|Q'_{H-x-1}(s,a)   - Q^{\hat{\pi}}_{H-x-1}(s,a)| \\
    &\leq \frac{\eps}{H} + |E_{s'\sim P_{H-x-1}(\cdot| s,a)}[\hat{V}_{H-x}(s') - V^{\hat{\pi}}_{H-x}(s')]| \\
    &\leq \frac{\eps}{H}+ |E_{s'\sim P_{H-h}(\cdot| s,a)}[(x+1)\eps/H]| \\
    &= \frac{(x+2)\eps}{H}.
\end{align*}
Hence, from mathematical induction, the lemma holds. 
\end{proof}

We now present the proof of the main result of this section. Similarly, the same steps can be used to prove similar results in our other low-rank settings. 

\begin{proof}[Proof of Theorem \ref{thm:MECR}]
We prove the correctness of the algorithm by applying Lemma \ref{lem:MECRThm} at time step $1$, which occurs with probability at least $1 - O(Hn^{-3})$. Next, the number of samples used is $\sum_{t=0}^{H-1}|\Omega_{H-t}|$. By definition of our sampling procedure, $k = |\Omega_h| \sim \Bin(n^2, p_h)$. Hence, from the one-sided Bernstein's inequality, Proposition \ref{prop:oBern}, for $h \in [H]$ and $C'' = \frac{\sqrt{8}}{C_{cvx}\sqrt{3}}$, it follows that 
\begin{align*}
    \P(|\Omega_h| - \E[|\Omega_h|] \geq C''p_hn^2) &\leq \exp \left(-\frac{p_h^2(C'')^2n^2}{2(p_h + \frac{p_hC''}{3}} \right)\\
    &\leq \exp \left(-\frac{3p_hC''n^2}{8} \right)\\
    &\leq \exp \left(-\mu^3d^2\kappa^2H^4 n\log(n) /\eps^2 \right).
\end{align*}
Since $\E[|\Omega_h|] = n^2p_h = C_{cvx}\mu^3d^2\kappa^2H^4n\log(n)/\eps^2$, from the union bound, it follows that $|\Omega_h| \in O(H^4n\log(n)/\eps^2)$ for all $h \in [H]$ with probability at least $1 - \exp \left(-\mu^3d^2\kappa^2H^4 n\log(n) /\eps^2 \right)$. Hence,  the sample complexity is upper bounded by 
\[
    \sum_{t=0}^{H-1}|\Omega_{H-t}| \in \Tilde{O}\left(\frac{\mu^3 H^5 n }{\eps^2} \right)
\]
with probability at least $1 - O(Hn^{-3})- \exp \left(-\mu^3d^2\kappa^2H^4 n\log(n) /\eps^2 \right)$.

\end{proof}

\section{Additional Theorems for Reference}\label{sec:additional_theorems}
We present the following lemmas, propositions, and theorems for the readers' convenience. 
\begin{theorem}[Hoeffding's Inequality \citep{wainwright2019high}]\label{thm:hoef}
Let $X_1, \ldots, X_n$ be independent, and $X_i$ have mean $\mu_i$ and sub-Gaussian parameter $\sigma_i$. Then, for all $t\geq 0$, we have 
\[
\P\left[ \sum_{i=1}^n(X_i - \mu_i) \geq t \right] \leq \exp \left( - \frac{t^2}{2\sum_{i=1}^n\sigma_i^2}       \right).
\]
\end{theorem}

\begin{proposition}[Proposition 2.14 (One-sided Bernstein's Inequality) \citep{wainwright2019high}]\label{prop:oBern}
Given $n$ independent random variables such that $X_i \leq b$ almost surely, we have 
\[
\P \left ( \sum_{i=1}^n (X_i - \E[X_i]) \geq cn       \right) \leq \exp\left(    -\frac{nc^2}{2(\frac{1}{n}\sum_{i=1}^n \E[X_i^2] + \frac{bc}{3})}  \right).
\]
\end{proposition}

\begin{theorem}[Matrix Bernstein \citep{2011}]\label{thm:mb}
Let $X^{(1)}, \ldots, X^{(n)} \in \R^{d_1\times d_2}$ be independent zero-mean matrices satisfying 
\begin{align*}
    \|X^{(i)}\|_{op} &\leq b, \quad \text{a.s.}\\
    \max\{\|\sum_{i=1}^n\E[X^{(i)^\top} X^{(i)}] \|_{op}, \|\sum_{i=1}^n \E[X^{(i)}X^{(i)^\top} ] \|_{op} \} &\leq n \sigma^2.
\end{align*}
Then 
\[
\P\left( \left \|\sum_{i=1}^n X^{(i)} \right \|_{op}     \geq t \right) \leq (d_1+d_2)\exp \left( -\frac{t^2}{2(n\sigma^2 + \frac{bt}{3})} \right).
\]
\end{theorem}

\begin{theorem}[Singular Value Courant-Fischer Minimax Theorem (Theorem 7.3.8 \cite{horn_johnson_1985})]\label{thm:cfSV}
Let $A \in \R^{m\times n}$, and $q = \min(m, n)$, let $\sigma_1(A), \sigma_2(A), \ldots, \sigma_q(A)$ be the ordered singular values of $A$, and let $k \in [q]$. Then, 
\[
\sigma_k(A) = \min_{S: dim(S)= m - k +1} \max_{x: 0 \neq X \in S} \frac{\|Ax\|_2}{\|x\|_2}
\]
and
\[
\sigma_k(A) = \max_{S: dim(S)=  k} \min_{x: 0 \neq X \in S} \frac{\|Ax\|_2}{\|x\|_2}.
\]
\end{theorem}



\end{document}